\documentclass[11pt,reqno]{amsart}
\usepackage{graphicx,enumerate,amssymb,bbold}
\usepackage[notrig]{physics}
\usepackage[font=footnotesize,labelfont=small]{subcaption}
\usepackage[left=3.0cm,right=3.0cm,top=2.5cm,bottom=2.5cm,includeheadfoot]{geometry} 

\usepackage[colorlinks=true, pdfstartview=FitV, linkcolor=blue, 
            citecolor=blue, urlcolor=blue]{hyperref}

\numberwithin{equation}{section}
\numberwithin{figure}{section}
\theoremstyle{plain}
\newtheorem{theorem}{Theorem}[section]
\newtheorem{lemma}[theorem]{Lemma}
\newtheorem{proposition}[theorem]{Proposition}

\newtheorem{conjecture}{Conjecture}[section]

\theoremstyle{definition}
\newtheorem{definition}{Definition}[section]

\newtheorem{remark}{Remark}[section]
\newtheorem{example}{Example}[section]

\newcommand{\bitem}{\begin{itemize}}
\newcommand{\eitem}{\end{itemize}}
\newcommand{\mc}[1]{\mathcal{#1}}

\newcommand{\N}{\mathbb{N}}
\newcommand{\R}{\mathbb{R}}

\newcommand{\B}{\mathbb{B}}
\newcommand{\Sp}{\mathbb{S}}

\newcommand{\EE}{\mathbb{E}}

\newcommand{\bpm}{\begin{pmatrix}}
\newcommand{\epm}{\end{pmatrix}}
\newcommand{\bsm}{\left(\begin{smallmatrix}}
\newcommand{\esm}{\end{smallmatrix}\right)}
\newcommand{\T}{\top}

\newcommand{\ol}[1]{\overline{#1}}
\newcommand{\la}{\langle}
\newcommand{\ra}{\rangle}

\newcommand{\mrm}[1]{\mathrm{#1}}

\newcommand{\veps}{\varepsilon}

\newcommand{\gdw}{\Leftrightarrow}

\newcommand{\eins}{\mathbb{1}}

\DeclareMathOperator{\Diag}{Diag}

\newcommand{\otop}[1]{\mathring{#1}}

\DeclareMathOperator{\supp}{supp}

\DeclareMathOperator{\Exp}{Exp}

\title[Image Labeling by Assignment]{Image Labeling by Assignment}
\author[F.~~{\AA}str\"{o}m, S.~Petra, B.~Schmitzer, C.~Schn\"{o}rr]{F.~{\AA}str\"{o}m, S.~Petra, B.~Schmitzer, C.~Schn\"{o}rr}
\address[F.~{\AA}str\"{o}m]{Heidelberg Collaboratory for Image Processing, Heidelberg University, Germany} 
\email{freddie.astroem@iwr.uni-heidelberg.de}
\urladdr{\url{http://hci.iwr.uni-heidelberg.de/Staff/fastroem/}}
\address[S.~Petra]{Mathematical Image Analysis Group, Heidelberg University, Germany} 
\email{petra@math.uni-heidelberg.de}
\urladdr{\url{http://ipa.iwr.uni-heidelberg.de/dokuwiki/doku.php?id=people:spetra:start}}
\address[B.~Schmitzer]{CEREMADE, University Paris-Dauphine, France} 
\email{schmitzer@ceremade.dauphine.fr}
\urladdr{\url{https://www.ceremade.dauphine.fr/~schmitzer/}}
\address[C.~Schn\"{o}rr]{Image and Pattern Analysis Group, Heidelberg University, Germany} 
\email[corresponding author]{schnoerr@math.uni-heidelberg.de}
\urladdr{\url{http://ipa.iwr.uni-heidelberg.de}}
\date{\today} 

\thanks{Support by the German Research Foundation (DFG) is gratefully acknowledged, grant GRK 1653.}


\begin{document}

\maketitle

\begin{abstract}
We introduce a novel geometric approach to the image labeling problem. Abstracting from specific labeling applications, a general objective function is defined on a manifold of stochastic matrices, whose elements assign prior data that are given in any metric space, to observed image measurements. The corresponding Riemannian gradient flow entails a set of replicator equations, one for each data point, that are spatially coupled by geometric averaging on the manifold. Starting from uniform assignments at the barycenter as natural initialization, the flow terminates at some global maximum, each of which corresponds to an image labeling that uniquely assigns the prior data. Our geometric variational approach constitutes a smooth non-convex inner approximation of the general image labeling problem, implemented with sparse interior-point numerics in terms of parallel multiplicative updates that converge efficiently.
\end{abstract}

%\newpage

\tableofcontents

%%%
%\newpage
%%%

\section{Introduction}
\label{sec:Introduction}

\subsection{Motivation}
\label{sec:Motivation}

\textit{Image labeling} is a basic problem of variational low-level image analysis. It amounts to determining a \textit{partition} of the image domain by uniquely assigning to each pixel a single element from a finite set of labels. Most applications require such decisions to be made depending on other decisions. This gives rise to a global objective function whose minima correspond to favorable label assignments and partitions. Because the problem of computing globally optimal partitions generally is NP-hard, \textit{relaxations} of the variational problem only define computationally feasible optimization approaches.

\textit{Continuous models} and relaxations of the image labeling problem were studied e.g.~in \cite{Lellmann-Schnoerr-10a,Chambolle2012}, including the specific binary case, where two labels are only assigned \cite{ContinuousGlobalBinary06} and the convex relaxation is tight, such that the global optimum can be determined by convex programming. \textit{Discrete models} prevail in the field of computer vision. They lead to polyhedral relaxations of the image partitioning problem that are tighter than those obtained from continuous models after discretization. We refer to \cite{Kappes2015} for a comprehensive survey and evaluation. Similar to the continuous case, the binary partition problem can be efficiently and globally optimal solved using a subclass of binary discrete models \cite{EnergyGraphCuts-PAMI04}.

Relaxations of the variational image labeling problem fall into two categories: \textit{convex and non-convex relaxations}. The dominant \textit{convex approach} is based on the local-polytope relaxation, a particular linear programming (LP-) relaxation \cite{WernerPAMI07}. This has spurred a lot of research on developing specific algorithms for efficiently solving large problem instances, as they often occur in applications. We mention \cite{Kolmogorov-TRMP06} as a prominent example and otherwise refer again to \cite{Kappes2015}. Yet, models with higher connectivity in terms of objective functions with local potentials that are defined on larger cliques, are still difficult to solve efficiently. A major reason that has been largely motivating our present work is the \textit{non-smoothness} of optimization problems resulting from convex relaxation -- the price to pay for convexity.

Major classes of \textit{non-convex relaxations} are based on the mean-field approach \cite{Orland1985}, \cite[Section 5]{WainrightJordan08} or on approximations of the intractable entropy of the probability distribution whose negative logarithm equals the functional to be minimized \cite{Yedidia-GenBP-05}. Examples for early applications of relaxations of the former approach include \cite{Herault1993, PairwiseClusteringAnnealing-97}. The basic instance of the latter class of approaches is known as the Bethe-approximation. In connection with image labeling, all these approaches amount to \textit{non-convex inner} relaxations of the combinatorially complex set of feasible solutions (the so-called marginal polytope), in contrast to the \textit{convex outer} relaxations in terms of the local polytope discussed above. As a consequence, the non-convex approaches provide a mathematically valid basis for \textit{probabilistic inference} like computing marginal distributions, which in principle enables a more sophisticated data analysis than mere energy minimization or maximum a-posteriori inference, to which energy minimization corresponds from a probabilistic viewpoint.

On the other hand, like non-convex optimization problems in general, these  relaxations are plagued by the problem of avoiding poor local minima. Although attempts were made to tame this problem by local convexification \cite{Heskes-ApprInferConvex06}, the class of \textit{convex} relaxation approaches has become dominant in the field, because the ability to solve the relaxed problem for a global optimum is a much better basis for research on algorithms and also results in more reliable software for users and applications.

\vspace{0.25cm}
Both classes of convex and non-convex approaches to the image labeling problem motivate the present work as an attempt to address the following two issues.
\begin{itemize}
\item \textbf{Smoothness vs.~Non-Smoothness.} Regarding convex approaches and the development of efficient algorithms, a major obstacle stems from the inherent non-smoothness of the corresponding optimization problems. This issue becomes particularly visible in connection with decompositions of the optimization task into simpler problems by dropping complicating contraints, at the cost of a non-smooth dual master problem where these constraints have to be enforced. Advanced bundle methods \cite{Kappes2012a} then seem to be among the most efficient methods. Yet, how to make rapid progress in systematic way does not seem obvious. 

On the other hand, since the early days of linear programming, e.g.~\cite{Bayer1989a,Bayer1989}, it has been known that endowing the feasible set with a proper \textit{smooth} geometry enables efficient numerics. Yet, such \textit{interior point} methods \cite{Nesterov2002} are considered as not applicable for large-scale problems of variational image analysis, due to dense numerical linear algebra steps that are both too expensive and too memory intensive.

In view of these aspects, \textbf{our approach} may be seen as a \textit{smooth geometric approach} to image labeling based on \textit{first-order, sparse} numerical operations. \\
\item
\textbf{Local vs.~global optimality.} Global optimality distinguishes convex approaches from other ones and is the major argument for the former ones. Yet, having computed a global optimum of the relaxed problem, it has to be projected to the feasible set of combinatorial solutions (labelings) in a post-processing step. While the inherent suboptimality of this step can be bounded \cite{Lellmann2011b}, and despite progress has been made to recover the true combinatorial optimum as least partially \cite{Swoboda2016}, it is clear that the benefit of global optimality of convex optimization has to be relativized when it constitutes a relaxation of an intractable optimization problem.
Turning to non-convex problems, on the other hand, raises the two well-known issues: local optimality of solutions instead of global optimality, and susceptibility to initialization.

In view of these aspects, \textbf{our approach} enjoys the following properties. While being non-convex, there is a \textit{single natural} initialization only which makes obsolete the need to search for a good initialization. Furthermore, the approach returns a \textit{global} optimum (out of many), which corresponds to an image labeling (combinatorial solution) without the need of further post-processing. 

Clearly, the latter property is typical for concave minimization formulations of combinatorial optimization problems \cite{Horst1996} where solutions of the latter problem are enforced by weighting the concave penalty sufficiently large. Yet, in such cases, and in particular so when working in high dimensions as in image analysis, the problem persists to determine good initializations and to carefully design the numerics (search direction, step-size selection, etc.), in order to ensure convergence and a reasonable convergence rate.
\end{itemize}

\begin{figure}
\centering
\includegraphics[width=0.5\textwidth]{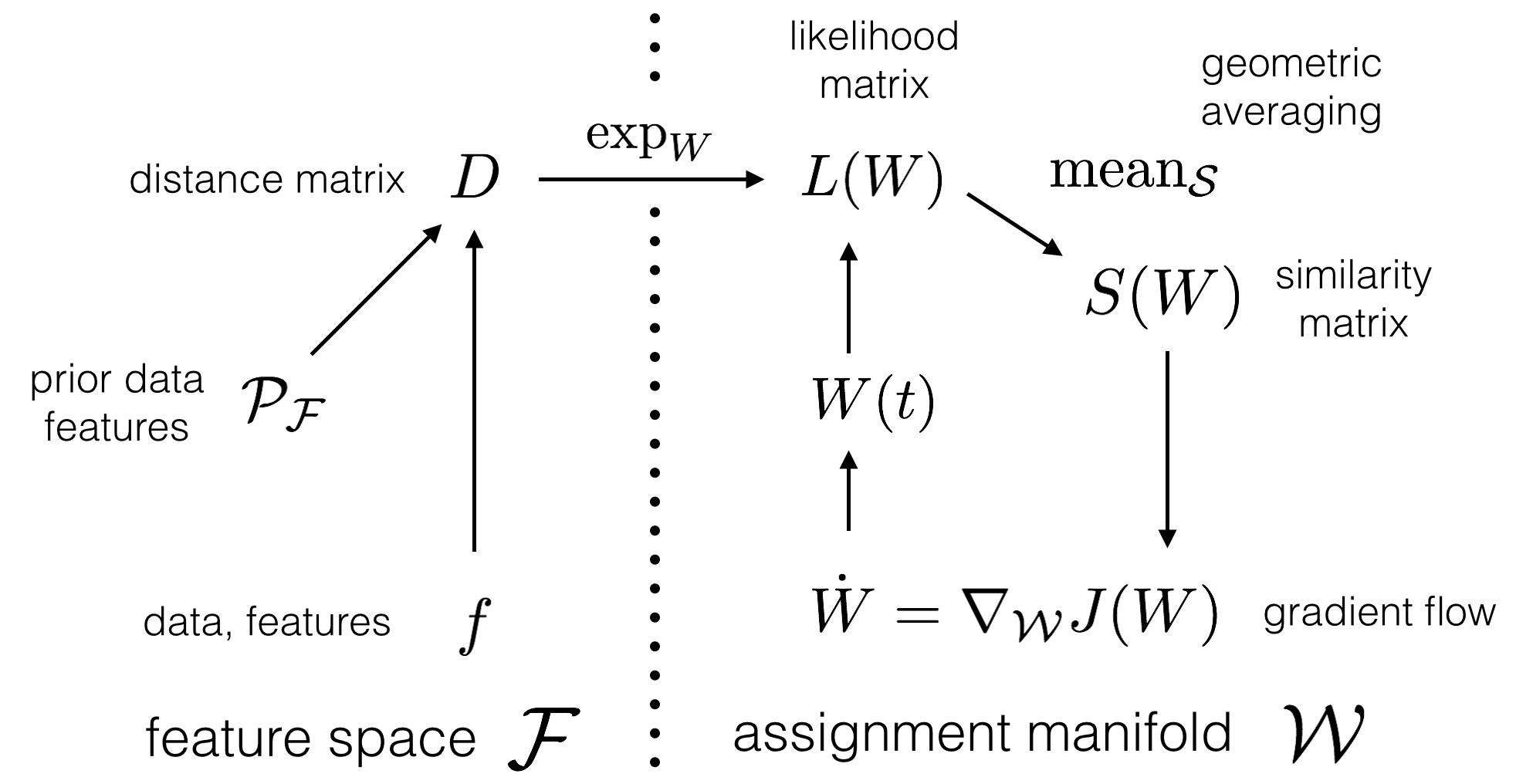}
\caption{
Overview of the variational approach. Given data and prior features in a metric space $\mc{F}$, inference corresponds to a Riemannian gradient flow with respect to an objective function $J(W)$ on the assignment manifold $\mc{W}$. The curve of matrices $W(t)$ assigns at each $t$ prior data $\mc{P}_{\mc{F}}$ to observed data $f$ and terminates at a global maximum $W^{\ast}$ that constitutes a labeling, i.e.~a unique assignment of a single prior datum to each data point. Spatial coherence of the labeling field is enforced by geometric averaging over spatial neighborhoods. The entire dynamic process on the assignment manifold achieves a MAP-labeling in a smooth, geometrical setting, realized with sparse interior-point numerics in terms of parallel multiplicative updates.
}
\label{fig:approach-overall}
\end{figure}

\subsection{Approach: Overview}
Figure \ref{fig:approach-overall} illustrates our set-up and the approach. We distinguish the feature space $\mc{F}$, that models all application-specific aspects, and the assignment manifold $\mc{W}$ used for modelling the image labeling problem and for computing a solution. This distinction avoids to mix up physical dimensions, specific data formats etc.~with the representation of the inference problem. It ensures broad applicability to any application domain that can be equipped with a metric which properly reflects data similarity. It also enables to normalize the representation used for inference, so as to remove any bias towards a solution \textit{not} induced by the data at hand.

We consider \textit{image labeling} as the task to assign to the image data an arbitrary prior data set $\mc{P}_{\mc{F}}$, provided the distance of its elements to any given data element can be measured by a distance function $d_{\mc{F}}$, which the user has to supply. Basic examples for the elements of $\mc{P}_{\mc{F}}$ include prototypical feature vectors, patches, etc.  Collecting all pairwise distance data into a distance matrix $D$, which could be computed on the fly for extremely large problem sizes, provides the input data to the inference problem.

The mapping $\exp_{W}$ lifts the distance matrix to the assignment manifold $\mc{W}$. The resulting likelihood matrix $L$ constitutes a normalized version of the distance matrix $D$ that reflects the initial feature space geometry as given by the distance function $d_{\mc{F}}$. Each point on $\mc{W}$, like the matrices $L, S$ and $W$, are \textit{stochastic matrices} with strictly positive entries, that is with row vectors that are discrete probability distributions having full support. Each such row vector indexed by $i$ represents the \textit{assignment} of prior elements of $\mc{P}_{\mc{F}}$ to the given datum a location $i$, in other words, the \textit{labeling} of datum $i$. We equip the set of all such matrices with the geometry induced by the Fisher-Rao metric and call it \textit{assignment manifold}.

The inference task (image labeling) is accomplished by \textit{geometric averaging} in terms of Riemannian means of assignment vectors over spatial neighborhoods. This step transforms the likelihood matrix $L$ into the similarity matrix $S$. It also induces a dependency of labeling decisions on each other, akin to the prior (regularization) terms of the established variational approaches to image labeling, as discussed in the preceding section. These dependencies are resolved by maximizing the correlation (inner product) between the assignment in terms of the matrix $W$ and the similarity matrix $S$, where the latter matrix is induced by $W$ as well. The Riemannian gradient flow of the corresponding objective function $J(W)$, that is highly nonlinear but smooth, evolves $W(t)$ on the manifold $\mc{W}$ until a fixed point is reached which terminates the loop on the right-hand side of Figure \ref{fig:approach-overall}. The resulting fixed point corresponds to an \textit{image labeling} which \textit{uniquely} assigns to each datum a prior element of $\mc{P}_{\mc{F}}$.

Adopting a probabilistic Bayesian viewpoint, this fixed-point iteration may be viewed as maximum a-posterior inference carried out in a geometric setting with multiplicative, sparse and highly parallel numerical operations.

\subsection{Further Related Work} \label{sec:Further-Related-Work}
Besides current research on image labeling, there are further classes of approaches that resemble our approach. We briefly sketch each of them in turn and highlight similarities and differences. 
\begin{description}
\item[Neighborhood Filters]
A large class of approaches to \textit{denoising} of given image data $f$ are defined in terms of neighborhood filters, that iteratively perform operations of the form
\begin{equation} \label{eq:neighborhood-filter}
u_{i}^{(k+1)} = \sum_{j} \frac{K(x_{i}, x_{j}, u_{i}^{(k)}, u_{j}^{(k)})}{\sum_{l} K(x_{i}, x_{l}, u_{i}^{(k)}, u_{l}^{(k)})} u_{j}^{(k)},\qquad u(0)=f,\qquad \forall i,
\end{equation}
where $K$ is a nonnegative kernel function that is symmetric with respect to the two indexed locations (e.g.~$i,j$ in the numerator) and may depend on both the spatial distance $\|x^{i}-x^{j}\|$ and the values $|u_{i}-u_{j}|$ of pairs of pixels. Maybe the most promiment example is the non-local means filter \cite{Buades-et-al-06} where $K$ depends on the distance of \textit{patches} centered at $i$ and $j$, respectively. We refer to  \cite{Milanfar2013} for a recent survey.

Noting that \eqref{eq:neighborhood-filter} is a linear operation with a row-normalized nonnegative (i.e.~stochastic) matrix, a similar situation would be
\begin{equation}
u_{i} = \sum_{j} L_{ij}(W) u_{j},
\end{equation}
with the likelihood matrix from Fig.~\ref{fig:approach-overall}, if we would replace the prior data $\mc{P}_{\mc{F}}$ with the given image data $f$ itself and adopt a distance function $d_{\mc{F}}$, in order to mimick the kernel function $K$ of \eqref{eq:neighborhood-filter}. 

In our approach, however, the likelihood matrix along with its nonlinear geometric transformation, the similarity matrix $S(W)$, evolves along with the evolution of assignment matrix $W$, so as to determine a labeling with \textit{unique} assignments to each pixel $i$, rather than convex combinations as required for denoising. Furthermore, the prior data set $\mc{P}_{\mc{F}}$ that is assigned in our case, may be very different from the given image data and, accordingly, the assignment matrix may have any rectangular shape rather than being a quadratic $m \times m$ matrix.

Conceptually, we are concerned with \textit{decision making} (labeling, partitioning, unique assignments) rather than with mapping one image to another one. Whenever the prior data $\mc{P}_{\mc{F}}$ comprise a finite set of \textit{prototypical} image values or patches, such that a mapping of the form
\begin{equation} \label{eq:u=Wf=related}
u_{i} = \sum_{j} W_{ij} f_{j}^{\ast},\qquad f_{j}^{\ast} \in \mc{P}_{\mc{F}},\qquad \forall i,
\end{equation}
is well-defined, then this does result in a transformed image $u$ after having reached a fixed point of the evolution of $W$. This result then should not be considered as a denoised image, however. Rather, it merely illustrates the interpretation of the given data $f$ in terms of the prior data $\mc{P}_{\mc{F}}$ and a corresponding optimal assignment.
\item[Nonlinear Diffusion]
Neighborhood filters are closely related to iterative algorithms for numerically solving discretized diffusion equations. Just think of the basic 5-point stencil of the discrete Laplacian, the iterative averaging of nearest neighbors differences, and the large class of adaptive generalizations in terms of nonlinear diffusion filters 
\cite{Weickert1998}. More recent work directly addressing this connection includes \cite{Buades2006, DiffusionNonlocalFilters-09, Milanfar2013a}. 
The author of \cite{Milanfar2013a}, for instance, advocates the approximation of the matrix of \eqref{eq:neighborhood-filter} by a \textit{symmetric} (hence, doubly-stochastic) \textit{positive-definite} matrix, in order to enable interpretations of the denoising operation in terms of the spectral decomposition of the assignment matrix, and to make the connection to diffusion mappings on graphs.

The connection to our work is implicitly given by the discussion of the previous point, the relation of our approach to neighborhood filters. Roughly speaking, the application of our approach in the \textit{specific} case of assigning image data to image data, may be seen as some kind of nonlinear diffusion that results in an image whose degrees of freedom are given by the cardinality of the prior set $\mc{P}_{\mc{F}}$. We plan to explore the exact nature of this connection in more detail in our future work.
\item[Replicator Dynamics] Replicator dynamics and the corresponding equations are well known \cite{EvolutionaryGameDynamics-03}. They play a major role in models of various disciplines, including theoretical biology and applications of game-theory to economy. In the field of image analysis, such models have been promoted by Pelillo and co-workers, mainly to efficiently determine by continuous optimization techniques good local optima of intractable problems, like matchings through maximum-clique search in an association graph \cite{ReplicatorMaxCliqueGraphIsomorphism-99}. Although the corresponding objective functions are merely quadratic, the analysis of the corresponding equations is rather involved \cite{Bomze-DynamicQP-SIREV-02}. Accordingly, clever heuristics have been suggested to tame related problems of non-convex optimization \cite{AnnealedMaxClique-02}.

Regarding our approach, we aim to get rid of these issues -- see the discussion of ``Global optimality'' in Section \ref{sec:Motivation} -- through three ingredients: (i) a unique natural initialization, (ii) spatial averaging that removes spurious local affects of noisy data, and (iii) adopting the Riemannian geometry which determines the structure of the replicator equations, for both geometric spatial averaging and numerical optimization.
\item[Relaxation Labeling] 
The task of labeling  primitives in images has been formulated as a problem of contextual decision making already 40 years ago \cite{Rosenfeld1976,Hummel1983}. Originally, update rules were merely formulated in order to find mutually consistent individual label assignments. Subsequent research related these labeling rules to optimization tasks. We refer to \cite{RelaxationLabeling-97} for a concise account of the literature and for putting the approach on mathematically solid ground. Specifically, the so-called Baum-Eager theorem was applied in order to show that updates increase the mutual consistency of label assignments. Applications include pairwise clustering \cite{DominantSetsPairwiseClustering-07} that boils down to determining a local optimum by continuous optimization of a non-convex quadratic form, similar to the optimization tasks considered in \cite{ReplicatorMaxCliqueGraphIsomorphism-99} and \cite{Bomze-DynamicQP-SIREV-02}. We attribute the fact that these approaches have not been widely applied to the problems of non-convex optimization discussed above.

The measure of mutual consistency of our approach is non-quadratic and the Baum-Eager theorem about polynomial growth transforms does not apply. Increasing consistency follows from the Riemannian gradient flow that governs the evolution of label assignments. Regarding the non-convexity from the viewpoint of optimization, we believe that the set-up of our approach displayed by Fig.~\ref{fig:approach-overall} significantly alleviates these problems, in particular through the geometric averaging of assignments that emanates from a natural initialization.
\end{description}
We address again some of these points, that are relevant our future work, in Section \ref{sec:Conclusion}.

\subsection{Organization}
Section \ref{sec:Assignment-Manifold} summarizes the geometry of the probability simplex in order to define the assignment manifold, which is the basis of our variational approach. The approach is presented in Section \ref{sec:Approach} by repeating the discussion of Figure \ref{fig:approach-overall}, together with the mathematical details. Finally, several numerical experiments are reported in Section \ref{sec:Basic-Applications}. They are academical, yet non-trivial, and supposed to illustrate properties of the approach as claimed in the preceding sections. Specific applications of image labeling are not within the scope of this paper. We conclude and indicate  further directions of research in Section \ref{sec:Conclusion}.

\vspace{0.25cm}
\noindent
Major symbols and the basic notation used in this paper are listed in Appendix \ref{sec:Basic-Notation}. In order not to disrupt the flow of reading and reasoning, proofs and technical details, all of which are elementary but essentially complement the presentation and make this paper self-contained, are listed as Appendix \ref{sec:App-Proofs}.

%%%%%%%
%\newpage
%
\section{The Assignment Manifold}
\label{sec:Assignment-Manifold}

In this section, we define the feasible set for representing and computating image labelings in terms of assignment matrices $W \in \mc{W}$, the assignment manifold $\mc{W}$. The basic building block is the open probability simplex $\mc{S}$ equipped with the Fisher-Rao metric. We collect below and in Appendix \ref{sec:S-proofs} corresponding definitions and properties. 

For background reading and much more details on information and Riemannian geometry, we refer to \cite{Amari2000} and \cite{Jost2005}.

\subsection{Geometry of the Probability Simplex} 
\label{sec:Geometry-Simplex}

The relative interior $\mc{S}=\otop{\Delta}_{n-1}$ of the probability simplex given by \eqref{eq:rint-delta}
becomes a differentiable Riemannian manifold 
when endowed with the Fisher-Rao metric. In the present particular case, it reads (cf.~the notation \eqref{eq:def-Rmetric})
\begin{equation} \label{eq:metric-simplex}
\la u, v \ra_{p} := \big\la \frac{u}{\sqrt{p}}, \frac{v}{\sqrt{p}} \big\ra,\qquad
\forall u,v \in T_{p}\mc{S},
\end{equation}
with tangent spaces given by
\begin{equation} \label{eq:def-TSimplex}
T_{p}\mc{S} = \{ v \in \R^{n} \colon \la \eins, v \ra = 0 \},\qquad
p \in \mc{S}.
\end{equation}
We regard the scaled sphere $\mc{N}=2\Sp^{n-1}$ as manifold with Riemannian metric induced by the Euclidean inner product of $\R^{n}$.
The following diffeomorphism $\psi$ between $\mc{S}_{n}$ and the open subset $\psi(\mc{S}_{n}) \subset \mc{N}$, was suggested e.g.~by \cite[Section 2.1]{Kass1989} and \cite[Section 2.5]{Amari2000}.
\begin{definition}[Sphere-Map]
We call the diffeomorphism 
\begin{equation} \label{eq:def-psi-simplex-sphere}
\psi \colon \mc{S} \to \mc{N},\qquad
p \mapsto s = \psi(p) := 2 \sqrt{p},
\end{equation}
\textit{sphere-map} (see Fig.~\ref{fig:simplexSphereMap}).
\end{definition}
The sphere-map enables to compute the geometry of $\mc{S}$ from the geometry of the 2-sphere.
\begin{lemma} \label{lem:sphere-map}
The sphere-map $\psi$ \eqref{eq:def-psi-simplex-sphere} is an isometry, i.e.~the Riemannian metric is preserved. Consequently, lenghts of tangent vectors and curves are preserved as well.
\end{lemma}
\begin{proof} See Appendix \ref{sec:S-proofs}. \end{proof}
\noindent

In particular, geodesics as critical points of length functionals are mapped by $\psi$ to geodesics. As a consequence, we have
\begin{lemma}[Riemannian Distance on $\mc{S}$] \label{eq:geoDist-simplex}
The Riemannian distance on $\mc{S}$ is given by
\begin{equation} \label{eq:geodist}
d_{\mc{S}}(p,q) = 2 \arccos\bigg(\sum_{i \in [n]} \sqrt{p_{i} q_{i}} \bigg) \in [0,\pi).
\end{equation}
\end{lemma}
\begin{figure}
%\centering
\begin{subfigure}[b]{0.25\textwidth}
\centering
\includegraphics[width=\textwidth]{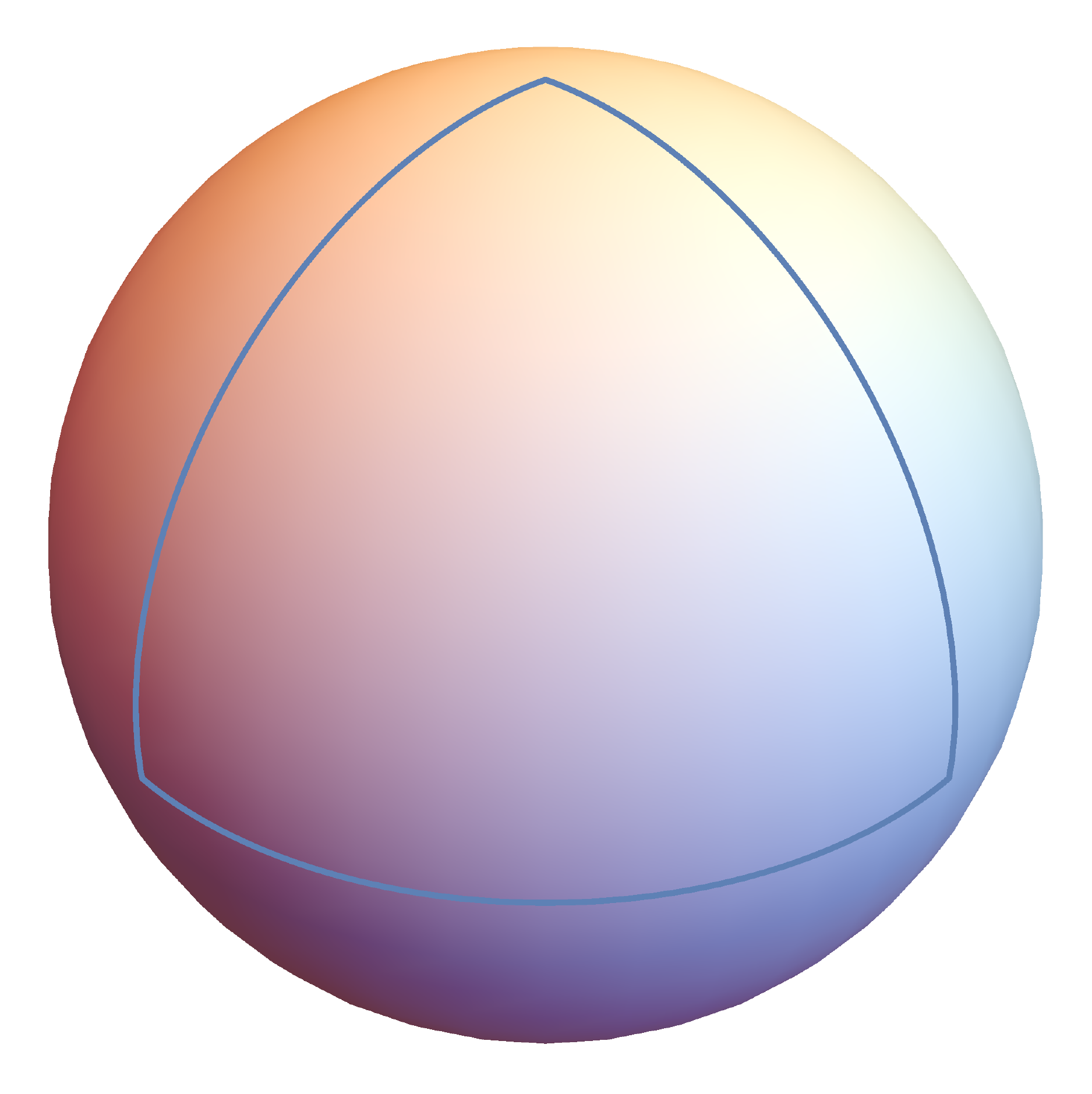}
\caption{
The triangle encloses the image $\psi(\mc{S}_{2}) \subset 2 \Sp^{2}$ of the simplex $\mc{S}_{2}$ under the sphere-map \eqref{eq:def-psi-simplex-sphere}.
}
\label{fig:simplexSphereMap}
\end{subfigure}%
\hfill%
\begin{subfigure}[b]{0.7\textwidth}
\centering
\includegraphics[width=0.45\textwidth]{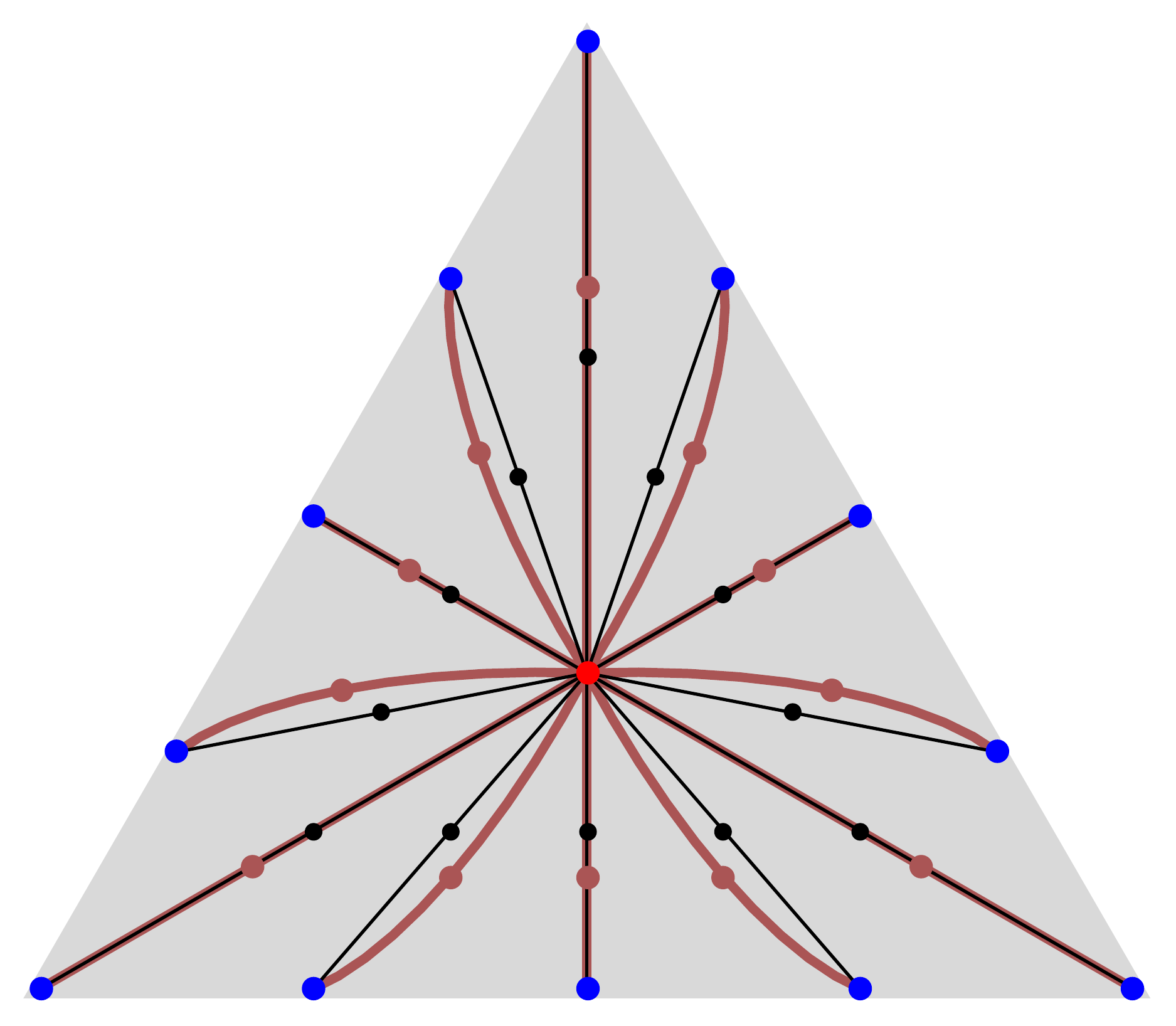}
\hspace{0.025\textwidth}
\includegraphics[width=0.45\textwidth]{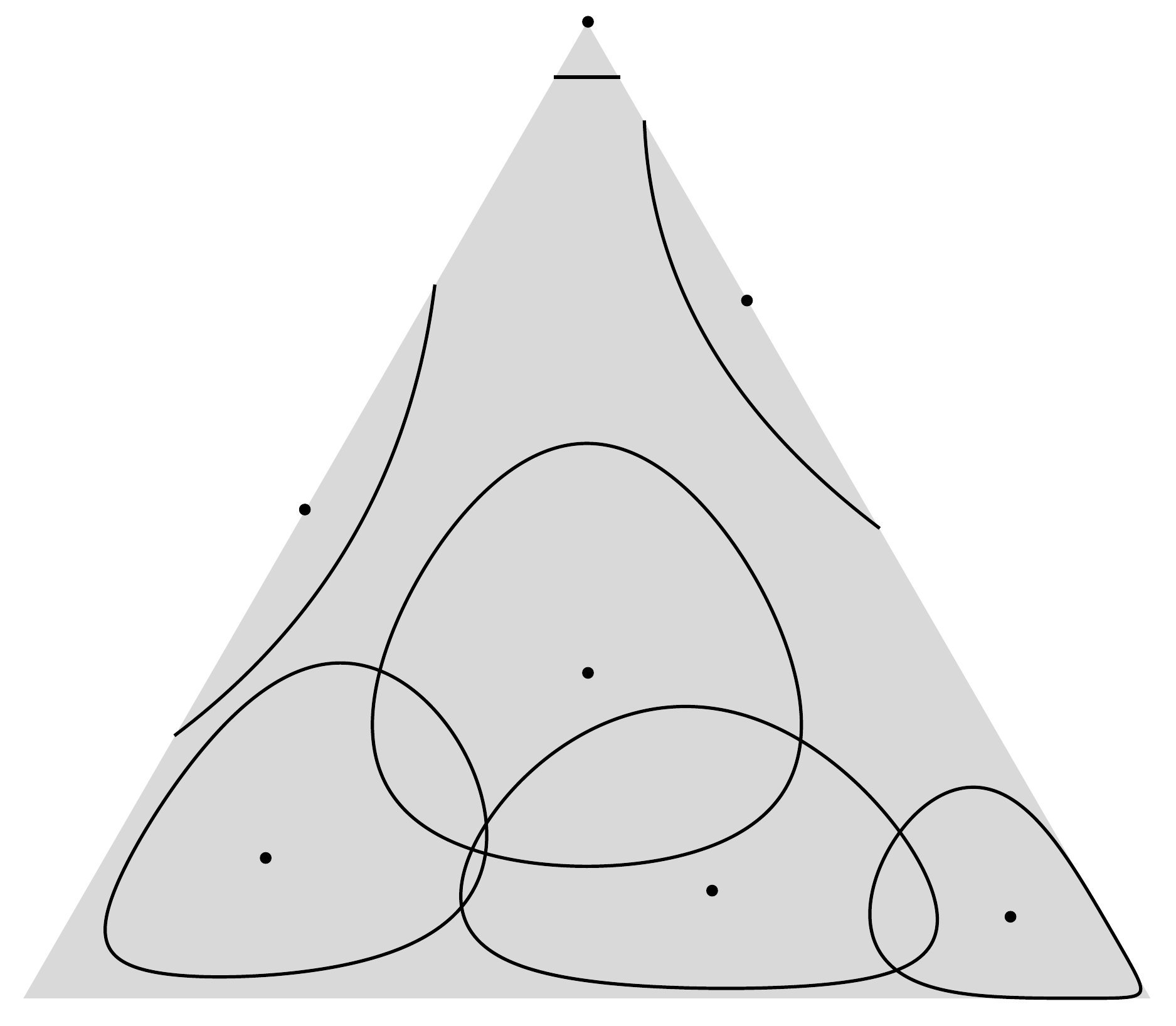}
\caption{
The geometry of the probability simplex induced by the Fisher-Rao Metric. The \textit{left panel} shows Euclidean (black) and non-Euclidean geodesics (brown) connecting the barycenter (red) and the blue points, along with the corresponding Euclidean and Riemannian means: In comparison to Euclidean means, geometric averaging pushes towards the boundary. The \textit{right panel} shows contour lines of points that have the same Riemannian distance from the respective center point (black dots). The different sizes of these regions indicates that geometric averaging causes a larger effect around the barycenters of both the simplex and its faces, where such points represent fuzzy labelings, and a smaller effect within regions close to the vertices (unit vectors).
}
\label{fig:simplexGeodesics}
\end{subfigure}
\caption{
Geometry of the probability simplex $\mc{S}_{2}$.
}
\label{fig:simplex-2-geometry}
\end{figure}
%%%%%
\begin{figure}
\centering
\includegraphics[width=0.5\textwidth]{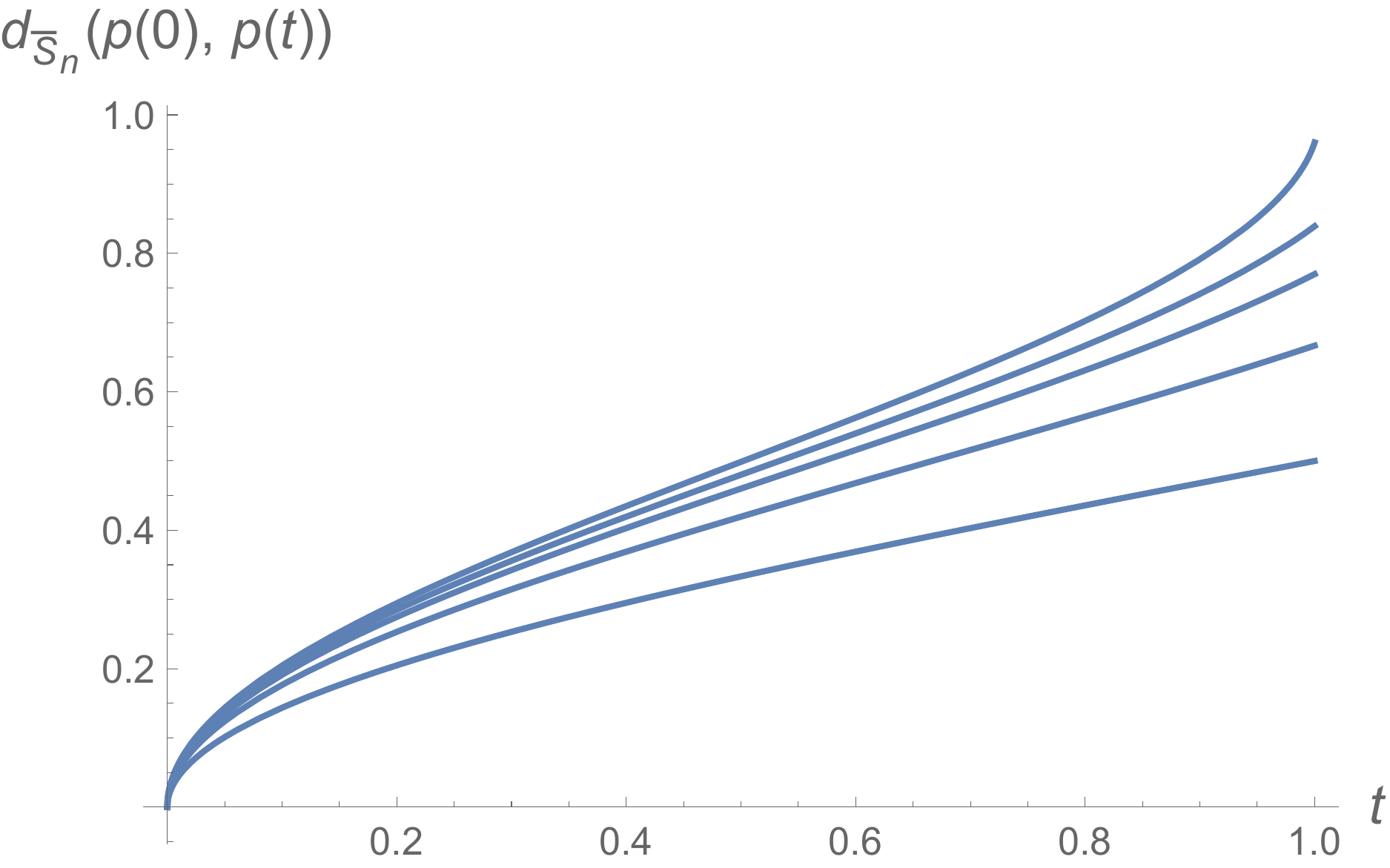}
\caption{
Each curve represents the Riemannian distances $d_{\ol{\mc{S}}_{n}}\big(p(0),p(t)\big)$ (normalized to [0,1]; Eq.~\eqref{eq:geodist}) of points on the curve $\{p(t)\}_{t \in [0,1]}$ that linearly (i.e., Euclidean) interpolates between the fixed vertex $p(0)=e^{1}$ of the simplex $\ol{\mc{S}}_{n}=\Delta_{n-1}$ and the barycenter $p(1)=\ol{p}$, for dimensions $n = 2^{k},\,k \in \{1,2,3,4,8\}$. As the dimension $n$ grows, the barycenter is located as far away from $e^{1}$ as all other boundary points $e^{i}, t e^{i}+(1-t) e^{j},\, t \in [0,1],\, i,j \neq 1$, etc., which have \textit{disjoint} supports. This entails a normalizing effect on the Riemannian mean of points that are far away, unlike with  Euclidean averaging where this influence increases with the Euclidean distance.
}
\label{fig:simplexDistances-over-d}
\end{figure}
The objective function for computing Riemannian means (geometric averaging;  see Definition \ref{def:Riemannian-mean} and Eq.~\eqref{eq:objective-Rmean} below) is based on the distance \eqref{eq:geodist}. Figure \ref{fig:simplexGeodesics} visualizes corresponding geodesics and level sets on $\mc{S}_{3}$, that differ for discrete distributions $p \in \mc{S}_{3}$ close to the barycenter and for low-entropy distributions close to the vertices. See also the caption of Fig.~\ref{fig:simplexGeodesics}.

It is well known from the literature (e.g.~\cite{Ball1997,Ledoux2001}) that geometries may considerably change in higher dimensions. Figure \ref{fig:simplexDistances-over-d} displays the Riemannian distances of points on curves that connect the barycenter and vertices on $\ol{\mc{S}}_{n}$ (to which the distance \eqref{eq:geodist} extends), depending on the dimension $n$. The normalizing effect on geometric averaging, further discussed in the caption, increases with $n$ and is relevant to image labeling, where large values of $n$ may occur in applications.

Let $\mc{M}$ be a smooth Riemannian manifold (see the paragraph around Eq.~\eqref{eq:differential-notation} introducing our notation). The Riemannian gradient $\nabla_{\mc{M}} f(p) \in T_{p}\mc{M}$ of a smooth function $f \colon \mc{M} \to \R$ at $p \in \mc{M}$ is the tangent vector defined by \cite[p.~89]{Jost2005}
\begin{equation} \label{eq:def-nabla_S}
\la \nabla_{\mc{M}} f(p), v \ra_{p} = Df(p)[v] 
= \la \nabla f(p), v \ra,\qquad
\forall v \in T_{p}\mc{M}.
\end{equation}
We consider next the specific case $\mc{M}=\mc{S}=\mc{S}_{n}$.
\begin{proposition}[Riemannian Gradient on $\mc{S}_{n}$] \label{prop:simplex-gradient}
For any smooth function $f \colon \mc{S} \to \R$, the Riemannian gradient of $f$ at $p \in \mc{S}$ is given by
\begin{equation} \label{eq:simplex-gradient}
\nabla_{\mc{S}} f(p) = p\big(\nabla f(p)-\la p, \nabla f(p) \ra \eins\big).
\end{equation}
\end{proposition}
\begin{proof} See Appendix \ref{sec:S-proofs}. \end{proof}
The exponential map associated with the open probability simplex $\mc{S}$ is detailed next.
\begin{proposition}[Exponential Map (Manifold $\mc{S}$)] \label{prop:Exp}
The exponential mapping 
\begin{subequations} \label{eq:Exp-explicit}
\begin{gather}
\Exp_{p} \colon V_{p} \to \mc{S},\qquad
v \mapsto \Exp_{p}(v) = \gamma_{v}(1), \qquad p \in \mc{S},
\intertext{is given by} \label{eq:gamma-S}
\gamma_{v}(t) = \frac{1}{2} \Big(p + \frac{v_{p}^{2}}{\|v_{p}\|^{2}}\Big)
+ \frac{1}{2}\Big(p - \frac{v_{p}^{2}}{\|v_{p}\|^{2}}\Big) \cos\big(\|v_{p}\| t\big) + \frac{v_{p}}{\|v_{p}\|} \sqrt{p} \sin\big(\|v_{p}\| t\big),
\end{gather}
with $t=1$, $v_{p} = v/\sqrt{p},\, p = \gamma(0)$, $\dot\gamma_{v}(0)=v$ and
\begin{gather} \label{eq:S-def-Vp}
V_{p} = \big\{v \in T_{p}\mc{S} \colon \gamma_{v}(t) \in \mc{S},\; t \in [0,1]\big\}.
\end{gather}
\end{subequations}
\end{proposition}
\begin{proof} See Appendix \ref{sec:S-proofs}. \end{proof}
\begin{remark} \label{rem:computing-Exp}
Checking the inclusion $v \in V_{p}$ due to \eqref{eq:S-def-Vp}, for a given tangent vector $v \in T_{p}\mc{S}$, is inconvenient for applications. Therefore, the mapping $\exp$ is defined below by Eq.~\eqref{eq:def-exp-vector} which approximates the exponential mapping $\Exp$, with the feasible set $V_{p}$ replaced by the entire space $T_{p}\mc{S}$ (Lemma \ref{prop:Exp-by-exp}). 

Accordingly, geometric averaging as defined next (Section \ref{sec:Riemannian-Mean}) based on $\Exp$, can be approximated as well using the mapping $\exp$. This is discussed in Section \ref{sec:computing-Rmeans}.
\end{remark}

%%%
\subsection{Riemannian Means}
\label{sec:Riemannian-Mean}

The \emph{Riemannian center of mass} is commonly called \textit{Karcher mean} or \textit{Fr\'{e}chet mean} in the more recent literature, in particular outside the field of mathematics. We prefer -- cf.~\cite{Karcher2014} -- the former notion and use the shorter term \textit{Riemannian mean}.
\begin{definition}[Riemannian Mean, Geometric Averaging] \label{def:Riemannian-mean}
The \textit{Riemannian mean} $\ol{p}$ of a set of points $\{p^{i}\}_{i \in [N]} \subset \mc{S}$ with corresponding weights $w \in \Delta_{N-1}$ minimizes the objective function
\begin{equation} \label{eq:objective-Rmean}
p \mapsto \frac{1}{2} \sum_{i \in [N]} w_{i} d_{\mc{S}}^{2}(p,p^{i})
\end{equation}
and satisfies the optimality condition \cite[Lemma 4.8.4]{Jost2005}
\begin{equation} \label{eq:grad-Rmean}
\sum_{i \in [N]} w_{i} \Exp_{\ol{p}}^{-1}(p^{i}) = 0,
\end{equation}
with the inverse of the exponential mapping $\Exp^{-1}_{p} \colon \mc{S} \to T_{p}\mc{S}$. We denote the Riemannian mean by
\begin{equation} \label{eq:notation-Rmean}
\mrm{mean}_{\mc{S},w}(\mc{P}),\qquad w \in \Delta_{N-1},\quad
\mc{P}=\{p^{1},\dotsc,p^{N}\},
\end{equation}
and drop the subscript $w$ in the case of uniform weights $w = \frac{1}{N} \eins_{N}$.
\end{definition}
\begin{lemma} \label{lem:Rmean-unique}
The Riemannian mean \eqref{eq:notation-Rmean} defined as minimizer of \eqref{eq:objective-Rmean} is unique for any data $\mc{P} = \{p^{i}\}_{i \in [n]} \subset \mc{S}$ and weights $w \in \Delta_{n-1}$.
\end{lemma}
\begin{proof}
Using the isometry $\psi$ given by \eqref{eq:def-psi-simplex-sphere}, we may consider the scenario transferred to the domain on the 2-sphere depicted by Fig.~\ref{fig:simplexSphereMap}. Due to \cite[Thm.~1.2]{Karcher1977}, the objective \eqref{eq:objective-Rmean} is convex along geodesics and has a unique minimizer within any geodesic Ball $\B_{r}$ with diameter upper bounded by $2 r \leq \frac{\pi}{2 \sqrt{\kappa}}$, where $\kappa$ upper bounds the sectional curvatures in $\B_{r}$. For the 2-sphere $\mc{N}$, we have $\kappa = 1/4$ constant, and hence the inequality is satisfied for the domain $\psi(\mc{S}) \subset \mc{N}$ which has geodesic diameter $\pi$.
\end{proof}
We call the computation of Riemannian means \textit{geometric averaging}. The implementation of this iterative operation and its efficient approximation by a closed-form expression are adressed in Section \ref{sec:Implementation}.

%%%
\subsection{Assignment Matrices and Manifold}
\label{sec:Assignment-Matrices}

A natural question is how to extend the geometry of $\mc{S}$ to stochastic matrices $W \in \R^{m \times n}$ with $W_{i} \in \mc{S},\, i \in [m]$, so as to preserve the information-theoretic properties induced by this metric (that we do not discuss here -- cf.~\cite{Cencov1982,Amari2000}).

This problem was recently studied by \cite{Montufar2014}. The authors suggested three natural definitions of manifolds. It turned out that all of them are slight variations of taking the product of $\mc{S}$, differing only by the scaling of the resulting product metric. As a consequence, we make the following
\begin{definition}[Assignment Manifold] \label{def:MW}
The manifold of assignment matrices, called \textit{assignment manifold}, is the set
\begin{equation} \label{eq:def-MW}
\mc{W} = \{W \in \R^{m \times n} \colon W_{i} \in \mc{S},\, i \in [m]\}.
\end{equation}
According to this product structure and based on \eqref{eq:metric-simplex}, the Riemannian metric is given by
\begin{equation} \label{eq:MW-metric}
\la U, V \ra_{W} := \sum_{i \in [m]} \la U_{i}, V_{i} \ra_{W_{i}},\qquad
U, V \in T_{W}\mc{W}.
\end{equation}
\end{definition}
Note that $V \in T_{W}\mc{W}$ means $V_{i} \in T_{W_{i}} \mc{S},\, i \in [m]$.
\begin{remark}
We call stochastic matrices contained in $\mc{W}$ \textit{assignment matrices}, due to their role in the variational approach (Section \ref{sec:Approach}).
\end{remark}

%%%%%%%%%%
\newpage
%%%
\section{Variational Approach}
\label{sec:Approach}

We introduce in this section the basic components of the variational approach and the corresponding optimization task, as illustrated by Figure \ref{fig:approach-overall}. 

%%%
\subsection{Basic Components}
\label{sec:Variational-Approach}

\subsubsection{Features, Distance Function, Assignment Task}
\label{sec:features-distance}
Let 
\begin{equation} \label{eq:data}
f \colon \mc{V} \to \mc{F},\qquad
i \mapsto f_{i},\qquad i \in \mc{V}=[m],
\end{equation}
denote any given data, either raw image data or features extracted from the data in a preprocessing step. In any case, we call $f$ \textit{feature}. At this point, we do not make any assumption about the \textit{feature space} $\mc{F}$ except that a \textit{distance function}
\begin{equation}
d_{\mc{F}} \colon \mc{F} \times \mc{F} \to \R,
\end{equation}
is specified. We assume that a finite subset of $\mc{F}$
\begin{equation} \label{eq:prior-data}
\mc{P}_{\mc{F}} := \{f^{\ast}_{j}\}_{j \in [n]},
\end{equation}
additionally is given, called \textit{prior set}. 
We are interested in the assignment of the prior set to the data in terms of an \textit{assignment matrix}
\begin{equation}
W \in \mc{W} \subset \R^{m \times n},
\end{equation}
with the manifold $\mc{W}$ defined by \eqref{eq:def-MW}. Thus, by definition, every row vector $0 < W_{i} \in \mc{S}$ is a discrete distribution with full support $\supp(W_{i})=[n]$. The element 
\begin{equation}
W_{ij} = \Pr(f^{\ast}_{j}|f_{i}),\qquad i \in [m],\quad j \in [n],
\end{equation}
quantifies the assignment of prior item $f^{\ast}_{j}$ to the observed data point $f_{i}$. We may think of this number as the \textit{posterior probability} that $f^{\ast}_{j}$ generated the observation $f_{i}$.

The \textit{assignment task} asks for determining an optimal assignment $W^{\ast}$, considered as ``explanation'' of the data based on the prior data $\mc{P}_{\mc{F}}$. We discuss next the ingredients of the objective function that will be used to solve assignment tasks.

\subsubsection{Distance Matrix} \label{sec:Distance-Matrix}
Given $\mc{F}, d_{\mc{F}}$ and $\mc{P}_{\mc{F}}$, we compute the \textit{distance matrix}
\begin{equation} \label{eq:uDist}
D \in \R^{m \times n},\quad
D_{i} \in \R^{n},\quad D_{ij} = \frac{1}{\rho} d_{\mc{F}} (f_{i},f^{\ast}_{j}),\quad \rho>0, \quad i \in [m],\quad j \in [n],
\end{equation}
where $\rho$ is the first (from two) \textit{user parameters} to be set. This parameter serves two purposes. It accounts for the unknown scale of the data $f$ that depends on the application and hence cannot be known beforehand. Furthermore, its value determines what subset of the prior features $f^{\ast}_{j},\, j \in [n]$ effectively affects the process of determining the assignment matrix $W$. This will be explained in detail in Section \ref{sec:Likelihood-Matrix} in connection with the subsequent processing stage that uses $D$ as input. We call $\rho$ \textit{selectivity parameter}.

Furthermore, we set
\begin{equation}
W = W(0),\qquad W_{i}(0) := \frac{1}{n} \eins_{n},\quad i \in [m].
\end{equation}
That is, $W$ is initialized with the uninformative \textit{uniform assignment} that is not biased towards a solution in any way.

\subsubsection{Likelihood Matrix} \label{sec:Likelihood-Matrix}
The next processing step is based on the following
\begin{definition}[Lifting Map (Manifolds $\mc{S}, \mc{W}$)]
The lifting mapping is defined by
\begin{subequations}
\begin{align}
\exp &\colon T\mc{S} \to \mc{S}, &
(p,u) &\mapsto \exp_{p}(u) = \frac{p e^{u}}{\la p, e^{u} \ra}, 
\label{eq:def-exp-vector} \\ \label{eq:def-exp-matrix}
\exp &\colon T\mc{W} \to \mc{W}, &
(W,U) &\mapsto \exp_{W}(U) = \bpm \exp_{W_{1}}(U_{1}) \\ \dots \\ \exp_{W_{m}}(U_{m}) \epm,
\end{align}
\end{subequations}
where $U_{i}, W_{i}, i \in [m]$ index the row vectors of the matrices $U, W$, and where the argument decides which of the two mappings $\exp$ applies.
\end{definition}
\begin{remark}
After replacing the arbitrary point $p \in \mc{S}$ by the barycenter $\frac{1}{n} \eins_{n}$, readers will recognize the \textit{softmax function} in \eqref{eq:def-exp-vector}, i.e.~$\la \frac{1}{n} \eins_{n}, e^{u} \ra^{-1} \big(\frac{1}{n} \eins_{n} e^{u}\big) = \frac{e^{u}}{\la \eins, e^{u} \ra}$.  This function is widely used in various application fields of applied statistics (e.g.~\cite{Sutton1999}), ranging from parametrizations of distributions, e.g.~for logistic classification \cite{Bishop-PRML06}, to other problems of modelling \cite{Luce1959} not related to our approach.

The lifting mapping generalizes the softmax function through the dependency on the base point $p$. In addition, it approximates geodesics and accordingly the exponential mapping $\Exp$, as stated next. We therefore use the symbol $\exp$ as mnemomic. Unlike $\Exp_{p}$, the mapping $\exp_{p}$ is defined on the entire tangent space, cf.~Remark \ref{rem:computing-Exp}.
\end{remark}
\begin{figure}
\centering
\includegraphics[width=0.3\textwidth]{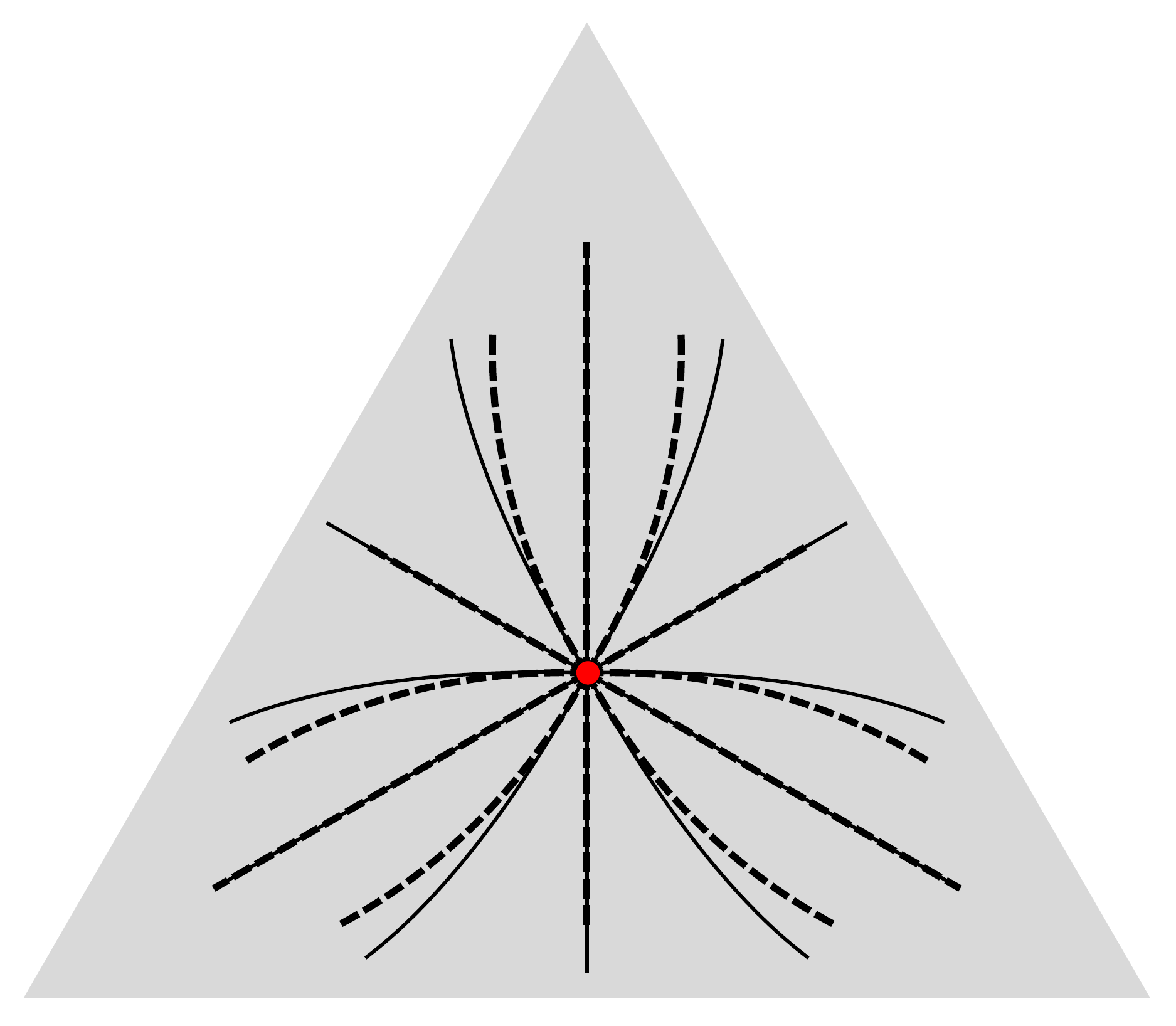}
\hspace{0.02\textwidth}
\includegraphics[width=0.3\textwidth]{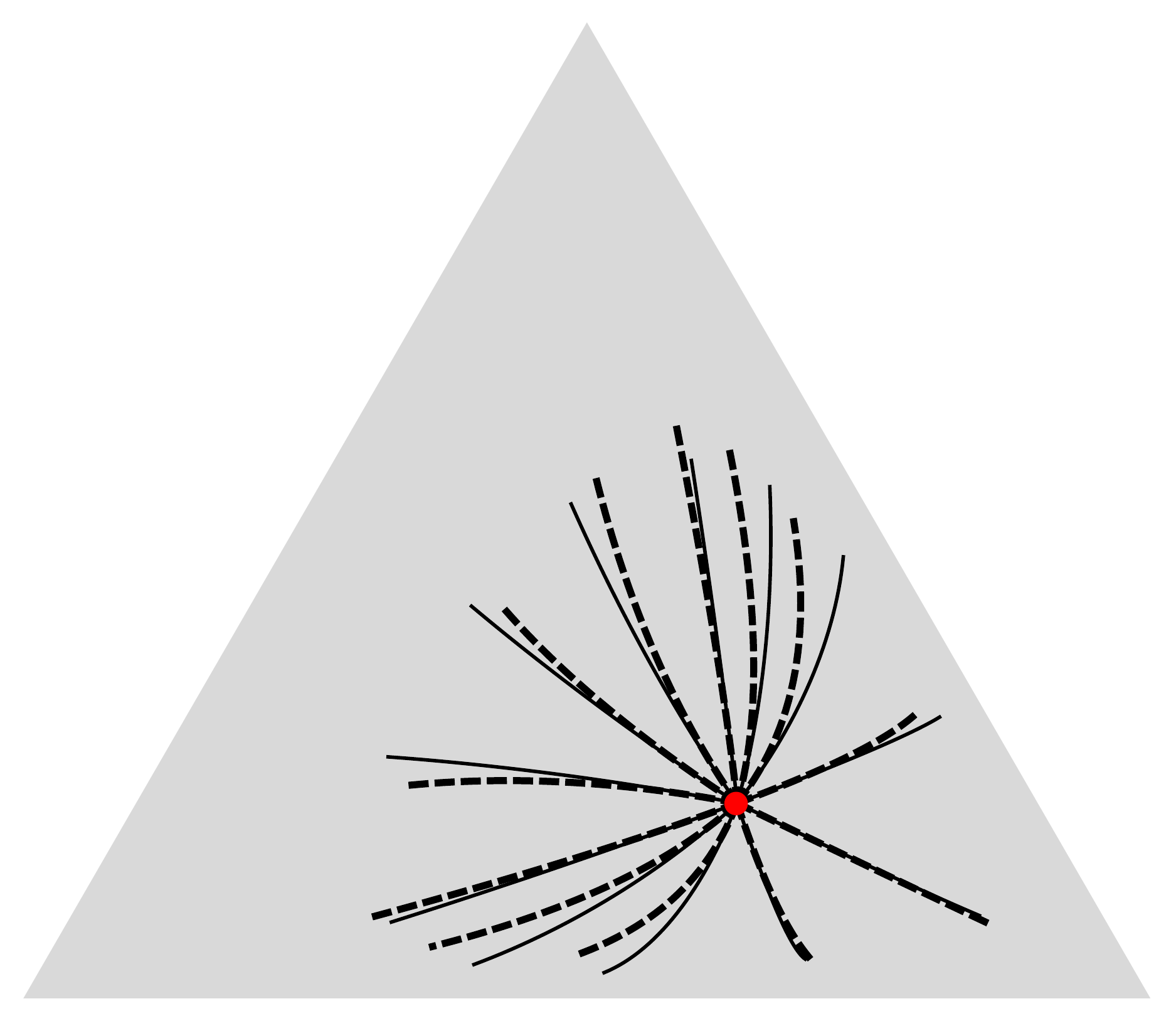}
\caption{
Illustration of Prop.~\ref{prop:Exp-by-exp}. Various geodesics $\gamma_{v^{i}}(t),\,i \in [k],\, t \in [t,t_{\max}]$ (solid lines) emanating from $p$ (red point) with the same speed $\|v^{i}\|_{p}=\|v^{j}\|_{p},\, \forall i,j$, are displayed together with the curves $\exp_{p}(u^{i} t),\,i \in [k],\, t \in [t,t_{\max}]$, where the vectors $u^{i}, v^{i},\, i \in [k]$ satisfy \eqref{eq:v-from-u}.
}
\label{fig:Exp-by-exp}
\end{figure}
\begin{proposition} \label{prop:Exp-by-exp}
Let 
\begin{equation} \label{eq:v-from-u}
v = \big(\Diag(p)-p p^{\T}\big) u,\qquad v \in T_{p}\mc{S}.
\end{equation}
Then $\exp_{p}(u t)$ given by \eqref{eq:def-exp-vector} solves
\begin{equation} \label{eq:exp-pt}
\dot p(t) = p(t) u - \la p(t), u \ra p(t),\qquad p(0)=p,
\end{equation}
and provides a first-order approximation of the geodesic $\gamma_{v}(t)$ from \eqref{eq:gamma-S}
\begin{equation} \label{eq:exp-approximates-gamma}
\exp_{p}(u t) \approx p + v t,\qquad
\|\gamma_{v}(t)-\exp_{p}(u t)\| = \mc{O}(t^{2}).
\end{equation}
\end{proposition}
\begin{proof} See Appendix \ref{sec:approach-proofs} \end{proof}
Figure \ref{fig:Exp-by-exp} illustrates the approximation of geodesics $\gamma_{v}$ and the exponential mapping $\Exp_{p}$, respectively, by the lifting mapping $\exp_{p}$.
\begin{remark} \label{rem:mean-of-u}
Note that adding any constant vector $c \eins,\, c \in \R$ to a vector $u$ does not change $\exp_{p}(u)$: $\frac{p e^{u+c \eins}}{\la p, e^{u+c \eins} \ra} = \frac{p (e^{c}\eins) e^{u}}{\la p, (e^{c}\eins) e^{u} \ra} = \frac{p e^{u}}{\la p, e^{u} \ra} = \exp_{p}(u)$. Accordingly, the same vector $v$ is generated by \eqref{eq:v-from-u}. While the definition \eqref{eq:def-exp-vector} removes this ambiguity, there is no need to remove the mean of the vector $u$ in numerical computations.
\end{remark}

Given $D$ and $W$ as described in Section \ref{sec:Distance-Matrix}, we lift the matrix
%\footnote{In view of Remark \ref{rem:mean-of-u}, we may regard the matrix $D$ both as vector field over $\mc{V}=[m]$ and as tangent vector field on $\mc{W}$ evaluated at $W$.} 
$D$ to the manifold $\mc{W}$ by
\begin{equation} \label{eq:def-L}
L = L(W) := \exp_{W}(-U) \in \mc{W},\qquad U_{i} = D_{i}-\frac{1}{n} \la \eins, D_{i} \ra \eins,\quad i \in [m],
\end{equation}
with $\exp$ defined by \eqref{eq:def-exp-matrix}. We call $L$ \textit{likelihood matrix} because the row vectors are discrete probability distributions which separately represent the similarity of each observation $f_{i}$ to the prior data $\mc{P}_{\mc{F}}$, as measured by the distance $d_{\mc{F}}$ in \eqref{eq:uDist}.

Note that the operation \eqref{eq:def-L} depends on the assignment matrix $W \in \mc{W}$. 

\subsubsection{Similarity Matrix} \label{sec:similarity-matrix}
Based on the likelihood matrix $L$, we define the \textit{similarity matrix} 
\begin{equation} \label{eq:def-S}
S = S(W) \in \mc{W},\qquad
S_{i} = \mrm{mean}_{\mc{S}}\{L_{j}\}_{j \in \tilde{\mc{N}}_{\mc{E}}(i)},\qquad i \in [m],
\end{equation}
where each row is the Riemannian mean \eqref{eq:notation-Rmean} (using uniform weights) of the likelihood vectors, indexed by
the neighborhoods as specified by the underying graph $\mc{G}=(\mc{V},\mc{E})$,
\begin{equation} \label{eq:def-Neps}
\tilde{\mc{N}}_{\mc{E}}(i) = \{i\} \cup \mc{N}_{\mc{E}}(i),\qquad
\mc{N}_{\mc{E}}(i) = \{j \in \mc{V} \colon ij \in \mc{E}\}.
\end{equation} 
Thus, $S$ represents the similarity of the data within a local spatial neighborhood to the prior data $\mc{P}_{\mc{F}}$.

Note that $S$ depends on $W$ because $L$ does so by \eqref{eq:def-L}. The \textit{size} of the neighbourhoods $|\tilde{\mc{N}}_{\mc{E}}(i)|$ is the \textit{second user parameter}, besides the selectivity parameter $\rho$ for scaling the distance matrix \eqref{eq:uDist}. Typically, each $\tilde{\mc{N}}_{\mc{E}}(i)$ indexes the same local ``window'' around pixel location $i$. We then call the window size $|\tilde{\mc{N}}_{\mc{E}}(i)|$ \textit{scale parameter}.
\begin{remark}
In basic applications, the distance matrix $D$ will not change once the features and the feature distance $d_{\mc{F}}$ are determined. On the other hand, the likelihood matrix $L(W)$ and the similarity matrix $S(W)$ have to be recomputed as the assignment $W$ evolves, as part of any numerical algorithm used to compute an optimal assignment $W^{\ast}$.

We point out, however, that more general scenarios are conceivable -- without essentially changing the overall approach -- where $D = D(W)$ depends on the assignment as well and hence has to be updated too, as part of the optimization process. Section \ref{sec:rectangles} provides an example.
\end{remark}

%%%
\subsection{Objective Function, Optimal Assignment}
\label{sec:Objective-Function}

We specify next the objective function as criterion for assignments and the gradient flow on the assignment manifold, to compute an optimal assignment $W^{\ast}$. Finally, based on $W^{\ast}$, the so-called assignment mapping is defined.

\subsubsection{Objective Function}
Getting back to the interpretation from Section \ref{sec:features-distance} of the assignment matrix $W \in \mc{W}$ as \textit{posterior probabilities},
\begin{equation} \label{eq:posterior-Wij}
W_{ij} = \Pr(f^{\ast}_{j}|f_{i}),
\end{equation}
of assigning prior feature $f^{\ast}_{j}$ to the observed feature $f_{i}$, a natural \textit{objective function} to be maximized is
\begin{equation} \label{eq:objective-orig}
\max_{W \in \mc{W}} J(W),\qquad J(W) := \la S(W),W \ra.
\end{equation}
The functional $J$ together with the feasible set $\mc{W}$ formalizes the following objectives:
\begin{enumerate}
\item
Assignments $W$ should \textit{maximally correlate} with the feature-induced similarities $S = S(W)$, as measured by the inner product which defines the objective function $J(W)$.
\item
Assignments of prior data to observations should be done in a \textit{spatially coherent} way. This is accomplished by \textit{geometric averaging} of likelihood vectors over local spatial neighborhoods, which turns the likelihood matrix $L(W)$ into the similarity matrix $S(W)$, \textit{depending} on $W$.
\item
Maximizers $W^{\ast}$ should define \textit{image labelings} in terms of rows $\ol{W}_{i}^{\ast} = e^{k_{i}} \in \{0,1\}^{n},\; i, k_{i} \in [m]$, that are indicator vectors. While the latter matrices are not contained in the assignment manifold $\mc{W}$ as feasible set, we compute in practice assignments $W^{\ast} \approx \ol{W}^{\ast}$ arbitrarily close to such points. It will turn out below that the \textit{geometry enforces} this approximation. 

As a consequence, in view of \eqref{eq:posterior-Wij}, such points $W^{\ast}$ \textit{maximize posterior probabilities}, akin to the interpretation of MAP-inference with discrete graphical models by minimizing corresponding energy functionals. As discussed in Section \ref{sec:Introduction}, however, the mathematical structure of the optimization task of our approach, and the way of fusing data and prior information, are quite different.
\end{enumerate}
The following statement formalizes the discussion of the form of desired maximizers $W^{\ast}$.
\begin{lemma} \label{lem:global-maxima}
We have
\begin{equation}
\sup_{W \in \mc{W}} J(W) = m,
\end{equation}
and the supremum is attained at the extreme points 
\begin{equation} \label{eq:def-extreme-points-W}
\ol{\mc{W}}^{\ast} := \big\{\ol{W}^{\ast} \in \{0,1\}^{m \times n} \colon \ol{W}^{\ast}_{i} = e^{k_{i}},\,i \in [m],\; k_{1},\dotsc,k_{m} \in [n]\big\} \subset \ol{\mc{W}},
\end{equation}
corresponding to matrices with unit vectors as row vectors.
\end{lemma}
\begin{proof} See Appendix \ref{sec:approach-proofs} \end{proof}

\subsubsection{Assignment Mapping}
Regarding the feature space $\mc{F}$, no assumptions were made so far, except for specifying a distance function $d_{\mc{F}}$. We have to be more specific about $\mc{F}$ only if we wish to \textit{synthesize} the approximation to the given data $f$, in terms of an assignment $W^{\ast}$ that optimizes \eqref{eq:objective-orig} and the prior data $\mc{P}_{\mc{F}}$. We denote the corresponding approximation by 
\begin{equation} \label{eq:def-u-map}
u \colon \mc{W} \to \mc{F}^{|\mc{V}|},\qquad
W \mapsto u(W),\qquad
u^{\ast} := u(W^{\ast}),
\end{equation}
and call it \textit{assignment mapping}. 

A trivial example of such a mapping concerns cases where prototypical feature vectors $f^{\ast j},\, j\in [n]$ are assigned to data vectors $f^{i},\, i \in [m]$: the mapping $u(W^{\ast})$ then simply replaces each data vector by the convex combination of prior vectors assigned to it,
\begin{equation} \label{eq:trivial-assignment-map}
u^{\ast i} = \sum_{j \in [n]} W_{ij}^{\ast} f^{\ast j},\qquad i \in [m].
\end{equation}
And if $W^{\ast}$ approximates a global maximum $\ol{W}^{\ast}$ as characterized by Lemma \ref{lem:global-maxima}, then each $f_{i}$ is (almost) uniquely replaced by some $u^{\ast k_{i}} = f^{\ast k_{i}}$.

A less trivial example is the case of prior information in terms of patches. We specify the mapping $u$ for this case and further concrete scenarios in Section \ref{sec:Basic-Applications}. 

%%%
\subsubsection{Optimization Approach}
\label{sec:Optimization}

The optimization task \eqref{eq:objective-orig} does not admit a closed-form solution. We therefore compute the assignment by the \textit{Riemannian gradient ascent flow} on the manifold $\mc{W}$,
\begin{subequations} \label{eq:W-gradient-flow}
\begin{align} \label{eq:W-gradient-flow-a}
\dot W_{ij} = \big(\nabla_{\mc{W}} J(W)\big)_{ij}
&= W_{ij} \Big(\big(\nabla_{i}J(W)\big)_{j} - \big\la W_{i},\nabla_{i}J(W) \big\ra \big), \quad W_{i}(0) = \frac{1}{n} \eins, \quad j \in [n],
\intertext{with}
\nabla_{i}J(W) &:= \pdv{W_{i}} J(W) = \Big(\pdv{W_{i1}} J(W),\dotsc,\pdv{W_{in}} J(W)\Big), \qquad i \in [m],
\end{align}
\end{subequations}
which results from applying \eqref{eq:simplex-gradient} to the objective \eqref{eq:objective-orig}. The flows \eqref{eq:W-gradient-flow}, for $i \in [m]$, are \textit{not} independent as the product structure of $\mc{W}$ (cf.~Section \ref{sec:Assignment-Matrices}) might suggest. Rather, they are coupled through the gradient $\nabla J(W)$ which reflects the interaction of the distributions $W_{i},\,i \in [m]$, due to the geometric averaging which results in the similarity matrix \eqref{eq:def-S}. 

Observe that, by \eqref{eq:W-gradient-flow-a} and $\la \eins, W_{i} \ra=1$,
\begin{equation}
\la \eins, \dot W_{i} \ra = \la \eins, W_{i} \nabla_{i} J(W) \ra - \la W_{i}, \nabla_{i} J(W) \ra \la \eins, W_{i} \ra = 0,\qquad i \in [m],
\end{equation} 
that is $\nabla_{\mc{W}} J(W) \in T_{W}\mc{W}$, and thus the flow \eqref{eq:W-gradient-flow-a} evolves on $\mc{W}$. Let $W(t) \in \mc{W},\, t \geq 0$ solve \eqref{eq:W-gradient-flow-a}. Then, with the Riemannian metric \eqref{eq:MW-metric},
\begin{equation} \label{eq:dot-J}
\dv{t} J\big(W(t)\big) = \big\la \nabla_{\mc{W}}J\big(W(t)\big), \dot W(t) \big\ra_{W(t)} \overset{\eqref{eq:W-gradient-flow-a}}{=}
\big\|\nabla_{\mc{W}}J\big(W(t)\big)\big\|_{W(t)}^{2} \geq 0,
\end{equation}
that is, the objective function value \textit{increases} until a stationary point is reached where the Riemannian gradient vanishes. Clearly, we expect $W(t)$ to approximate a global maximum due to Lemma 
\ref{lem:global-maxima}, which all satisfy the condition for stationary points $\ol{W}$,
\begin{equation} \label{eq:stationary-points}
0 = \dot {\ol{W}}_{i} = \ol{W}_{i}\big(\nabla_{i} J(\ol{W}) - \la \ol{W}_{i}, \nabla_{i} J(\ol{W}) \ra \eins\big),\qquad i \in [m],
\end{equation}
because replacing $\ol{W}_{i}$ in \eqref{eq:stationary-points} by $\ol{W}_{i}^{\ast} = e^{k_{i}}$ for some $k_{i} \in [n]$ makes the bracket vanish for the $k_{i}$-th equation, whereas all other equations indexed by $j \neq k_{i},\, j \in [n]$ are satisfied due to $\ol{W}_{ij}^{\ast}=0$.

Regarding \textit{interior} stationary points $\ol{W} \in \mc{W}$ with $\ol{W} \geq 0$ due to the definition of $\mc{W}$, all brackets on the r.h.s.~of \eqref{eq:stationary-points} must vanish, which can only happen if the Euclidean gradient satisfies
\begin{equation} \label{eq:stationary-cond-interior}
\nabla_{i} J(\ol{W}) = \la \ol{W}_{i}, \nabla_{i} J(\ol{W}) \ra \eins,\qquad
i \in [m]
\end{equation}
including the case $\nabla J(\ol{W})=0$. Inspecting the gradient of the objective function \eqref{eq:objective-orig}, we get
\begin{subequations} \label{eq:nabla-J-explicit}
\begin{align}
\pdv{W_{ij}} J(W)
&= \pdv{W_{ij}} \la S(W), W \ra
= \sum_{k,l} \pdv{W_{ij}} \big(S_{kl}(W) W_{kl}\big) 
\label{eq:nabla-J-1} \\ \label{eq:nabla-J-2}
&= \sum_{k,l} \Big(\pdv{W_{ij}} S_{kl}(W)\Big) W_{kl}
+ S_{ij}(W) =: \la T^{ij}(W), W \ra + S_{ij}(W),
\end{align}
\end{subequations}
where both matrices $S(W)$ and $T^{ij}(W) = \pdv{W_{ij}} S(W)$ depend in a smooth but involved way on the data \eqref{eq:data} and \eqref{eq:prior-data} through the distance matrix \eqref{eq:uDist}, the likelihood matrix \eqref{eq:def-L} and the geometric averaging \eqref{eq:def-S} which forms the similarity matrix $S(W)$. Regarding the second term on the r.h.s.~of \eqref{eq:nabla-J-2}, a computation relegated to Appendix \ref{sec:approach-proofs} yields
\begin{equation} \label{eq:Tij-matrices}
\la T^{ij}(W), W \ra 
= \sum_{k,l} -\Big(\big(H^{k}(W)\big)^{-1} h^{k,ij}(W)\Big)_{l} W_{kl}.
\end{equation}
The way to compute the somewhat unwieldy explicit form of the r.h.s.~is explained by \eqref{eq:def-H-h} and the corresponding appendix. In terms of these quantities, condition \eqref{eq:stationary-cond-interior} for stationary interior points translates to
\begin{equation} \label{eq:Tij-Sij-condition}
\la T^{ij}(\ol{W}), \ol{W} \ra + S_{ij}(\ol{W}) = 
\sum_{j} \big(\la T^{ij}(\ol{W}), \ol{W} \ra + S_{ij}(\ol{W})\big) \ol{W}_{ij},\qquad
\forall i \in [m], \quad \forall j \in [n]
\end{equation}
including the special case $S_{ij}(W) = -\la T^{ij}(W), W \ra,\; \forall i \in [m],j \in [n]$, corresponding to $\nabla J(\ol{W})=0$. Note that condition \eqref{eq:Tij-Sij-condition} requires that for every $i \in [m]$, the l.h.s.~takes the \textit{same} value for every $j \in [n]$, such that averaging with respect to $W_{i}$ on the r.h.s.~causes no change. 

We do not have evidence for the non-existence of specific data configurations, for which the flow \eqref{eq:W-gradient-flow} may reach such very specific stationary interior points. Any such point, however, will not be a maximum and be isolated, by virtue of the local strict convexity of the objective function \eqref{eq:objective-Rmean} for Riemannian means (cf.~Lemma \ref{lem:Rmean-unique} below), which determines the similarity matrix \eqref{eq:def-S}. Consequently, any perturbation (e.g.~by numerical computation) will let the flow escape from such a point, in order to maximize the objective due to \eqref{eq:dot-J}.

We summarize this reasoning by the
\begin{conjecture} \label{conj:convergence}
For any data \eqref{eq:data}, \eqref{eq:prior-data}, up to a subset of $\mc{W}$ of measure zero, the flow $W(t)$ generated by \eqref{eq:W-gradient-flow} approximates a global maximum as defined by \eqref{eq:def-extreme-points-W} in the sense that, for any $0 < \veps \ll 1$, there is a $t=t(\veps)$ such that 
\begin{equation} \label{eq:maxima-approximation}
\big\|W\big(t(\veps)\big) - \ol{W}^{\ast}\big\| \leq \veps,\qquad
\text{for some}\quad \ol{W}^{\ast} \in \ol{\mc{W}}^{\ast}.
\end{equation}
\end{conjecture}
\begin{remark} $\text{ }$
\begin{enumerate}
\item Since $\ol{\mc{W}}^{\ast} \not\in \mc{W}$, the flow $W(t)$ cannot converge to a global maximum, and numerical problems arise when \eqref{eq:maxima-approximation} holds for $\veps$ very close to zero. Our strategy to avoid such problems is described in Section \ref{sec:W-normalization}.
\item Although global maxima are not attained, we agree to call a point $W^{\ast}=W(t)$ \textit{maximum} and \textit{optimal assignment}, that satisfies \eqref{eq:maxima-approximation} for some fixed small $\veps$. The criterion which terminates our algorithm is specified in Section \ref{sec:termination-criterion}.
\item Our numerical approximation of the flow \eqref{eq:W-gradient-flow} is detailed in Section \ref{sec:optimization-algorithm}.
\end{enumerate}
\end{remark}

%\vspace{0.5cm}
%%
\subsection{Implementation}
\label{sec:Implementation}

We discuss in this section specific aspects of the implementation of the variational approach.

\subsubsection{Assignment Normalization} \label{sec:W-normalization}
Because each vector $W_{i}$ approaches some vertex $\ol{W}^{\ast} \in \ol{\mc{W}}^{\ast}$ by construction, and  because the numerical computations are designed to evolve on $\mc{W}$, we avoid numerical issues by checking for each $i \in [m]$ every entry $W_{ij},\, j \in [n]$, after each iteration of the algorithm \eqref{eq:def-algorithm} below. Whenever an entry drops below $\veps=10^{-10}$, we rectify $W_{i}$ by
\begin{equation}
W_{i} \quad\leftarrow\quad \frac{1}{\la \eins, \tilde W_{i} \ra} \tilde W_{i},\qquad \tilde W_{i} = W_{i} - \min_{j \in [n]} W_{ij} + \veps,\qquad \veps = 10^{-10}.
\end{equation}
In other words, the number $\veps$ plays the role of $0$ in our impementation. Our numerical experiments (Section \ref{sec:Basic-Applications}) showed that this operation removed any numerical issues without affecting convergence in terms of the criterion specified in Section \ref{sec:termination-criterion}.

\subsubsection{Computing Riemannian Means} 
\label{sec:computing-Rmeans}
Computation of the similarity matrix $S(W)$ due to Eq.~\eqref{eq:def-S} involves the computation of Riemannian means. In view of Definition \eqref{def:Riemannian-mean}, we compute the Riemannian mean $\mrm{mean}_{\mc{S}}(\mc{P})$ of given points $\mc{P}=\{p^{i}\}_{i \in [N]} \subset \mc{S}$, using uniform weights, as fixed point $p^{(\infty)}$ by iterating the following steps.
\begin{subequations} \label{eq:Rmean-iteration}
\begin{align}
(1)\quad&
\text{Set}\; p^{(0)}=\frac{1}{n} \eins.
\intertext{Given $p^{(k)},\; k \geq 0$, compute (cf.~the explicit expressions \eqref{eq:InvExp-pdfLi-1} and \eqref{eq:Exp-explicit})}
(2)\quad&
v^{i}= \Exp_{p^{(k)}}^{-1}(p^{i}),\quad i \in [N], \\
(3)\quad&
v = \frac{1}{N} \sum_{i \in [N]} v^{i}, \\
(4)\quad&
p^{(k+1)} = \Exp_{p^{(k)}}(v),
\end{align}
\end{subequations}
and continue with step (2) until convergence. In view of the optimality condition \eqref{eq:grad-Rmean}, our implementation returns $p^{(k+1)}$ as result if after carrying out step (3) the condition $\|v\|_{\infty} \leq 10^{-3}$ holds.

We point out that numerical problems arise at step (2) if \textit{identical} vectors are averaged, as the expression \eqref{eq:InvExp-pdfLi-1} shows. Such situations may occur e.g.~when computer-generated images are processed. Setting $\veps=1-\la \sqrt{p},\sqrt{q}\ra$ for two vectors $p, q \in \mc{S}$, we replace the expression \eqref{eq:InvExp-pdfLi-1} by
\begin{equation}
\Exp_{p}^{-1}(q) \approx
\frac{9 \veps^{2}+40 \veps + 480}{240 \sqrt{1-\veps/2}} 
(\sqrt{p q}-(1-\veps) p) \qquad\text{if}\quad \veps < 10^{-3}.
\end{equation}
Although the iteration \eqref{eq:Rmean-iteration} converges quickly, carrying out such iterations as a subroutine, at each pixel and iterative step of the outer iteration \eqref{eq:def-algorithm}, increases runtime (of non-parallel implementations) noticeably. In view of the approximation of the exponential map $\Exp_{p}(v) = \gamma_{v}(1)$ by \eqref{eq:exp-approximates-gamma}, it seems natural to approximate the Riemannian mean as well by modifying steps (2) and (4) above accordingly.
\begin{lemma} \label{eq:Rmean-approximation}
Replacing in the iteration \eqref{eq:Rmean-iteration} above the exponential mapping $\Exp_{p}$ by the lifting map $\exp_{p}$ \eqref{eq:def-exp-vector} yields the closed-form expression
\begin{equation}
\frac{\mrm{mean}_{g}(\mc{P})}{\la \eins, \mrm{mean}_{g}(\mc{P}) \ra},\qquad
\mrm{mean}_{g}(\mc{P}) = \Big(\prod_{i \in [N]} p^{i}\Big)^{\frac{1}{N}}
\end{equation}
as approximation of the Riemannian mean $\mrm{mean}_{\mc{S}}(\mc{P})$, with the geometric mean $\mrm{mean}_{g}(\mc{P})$ applied componentwise to the vectors in $\mc{P}$.
\end{lemma}
\begin{proof} See Appendix \ref{sec:approach-proofs} \end{proof}

\subsubsection{Optimization Algorithm} 
\label{sec:optimization-algorithm}

A thorough analysis of various discrete schemes for numerically integrating the gradient flow \eqref{eq:W-gradient-flow}, including stability estimates, is beyond the scope of this paper and will be separately addressed in follow-up work (see Section \ref{sec:Conclusion} for a short discussion). 

Here, we merely adopted the following basic strategy from \cite{Losert1983}, that has been widely applied in the literature and performed remarkably well in our experiments. Approximating the flow \eqref{eq:W-gradient-flow} for each vector $W_{i},\, i \in [m]$, by the time-discrete scheme 
\begin{equation}
\frac{W_{i}^{(k+1)}-W_{i}^{(k)}}{t_{i}^{(k+1)}-t_{i}^{(k)}} = W_{i}^{(k)} \big(\nabla_{i}J(W^{(k)})-\la W_{i}^{(k)}, \nabla_{i} J(W^{(k)}) \ra \eins\big),\quad W_{i}^{(k)} := W_{i}(t_{i}^{(k)}), 
\end{equation}
and choosing the adaptive step-sizes $t_{i}^{(k+1)}-t_{i}^{(k)} = \frac{1}{\la W_{i}^{(k)}, \nabla_{i} J(W^{(k)})\ra}$, yields the multiplicative updates
\begin{equation}
W_{i}^{(k+1)} = \frac{W_{i}^{(k)} \big(\nabla_{i}J(W^{(k)})\big)}{\la W_{i}^{(k)}, \nabla_{i} J(W^{(k)}) \ra},\qquad i \in [m].
\end{equation}
We further simplify this update in view of the explicit expression \eqref{eq:nabla-J-explicit} of the gradient $\nabla_{i} J(W)$ of the objective function, that comprises two terms. The first one contributes the derivative of $S(W)$ with respect to $W_{i}$, which is significantly smaller than the second term $S_{i}(W)$ of \eqref{eq:nabla-J-explicit}, because $S_{i}(W)$ results from \textit{averaging} \eqref{eq:def-S} the likelihood vectors $L_{j}(W_{j})$ over spatial neighborhoods and hence changes slowly. As a consequence, we simply drop this first term which, as a byproduct, avoids the numerical evaluation of the expensive expressions \eqref{eq:Tij-matrices} specifying the first term. 

Thus, for computing the numerical results reported in this paper, we used the fixed-point iteration
\begin{equation} \label{eq:def-algorithm}
W_{i}^{(k+1)} = \frac{W_{i}^{(k)} \big(S_{i}(W^{(k)})\big)}{\la W_{i}^{(k)}, S_{i}(W^{(k)}) \ra},\qquad W_{i}^{(0)} = \frac{1}{n} \eins, \qquad i \in [m]
\end{equation}
together with the approximation due to Lemma \ref{eq:Rmean-approximation} for computing Riemannian means, which define by \eqref{eq:def-S} the similarity matrices $S(W^{(k)})$. Note that this requires to recompute the likelihood matrices \eqref{eq:def-L} as well, at each iteration $k$ (see Fig.~\ref{fig:approach-overall}).

\subsubsection{Termination Criterion} \label{sec:termination-criterion}
Algorithm \eqref{eq:def-algorithm} was terminated if the average entropy
\begin{equation} \label{eq:average-entropy}
-\frac{1}{m} \sum_{i \in [m]} \sum_{j \in [n]} W_{ij}^{(k)} \log W_{ij}^{(k)}
\end{equation}
dropped below a threshold. For example, a threshold value $10^{-3}$ means in practice that, up to a tiny fraction of indices $i \subset [m]$ that should not matter for a subsequent further analysis, all vectors $W_{i}$ are very close to unit vectors, thus indicating an almost unique assignment of prior items $f_{j}^{\ast},\, j \in [n]$ to the data $f_{i},\, i \in [m]$. Note that this termination criterion conforms to Conjecture \ref{conj:convergence} and was met in all experiments.

%\newpage
%%%
\section{Illustrative Applications and Discussion}
\label{sec:Basic-Applications}

We focus in this section on few academical, yet non-trivial numerical examples, to illustrate and discuss basic properties of the approach. Elaborating any specific application is outside the scope of this paper.

\subsection{Parameters, Empirical Convergence Rate}
Figure \ref{fig:rgb-diffusion} shows a color image and a noisy version of it. The latter image was used as input data of a labeling problem. Both images comprise $31$ color vectors forming the prior data set $\mc{P}_{\mc{F}} = \{f^{1\ast},\dotsc,f^{31\ast}\}$. The labeling task is to assign these vectors in a spatially coherent way to the input data so as to recover the ground truth image. 

Every color vector was encoded by the vertices of the simplex $\Delta_{30}$, that is by the unit vectors $\{e^{1},\dotsc,e^{31}\} \subset \{0,1\}^{31}$. Choosing the distance $d_{\mc{F}}(f^{i},f^{j}) := \|f^{i}-f^{j}\|_{1}$, this results in unit distances between all pairs of data points and hence enables to assess most clearly the impact of geometric spatial averaging and the influence of the two parameters $\rho$ and $|\mc{N}_{\veps}|$, introduced in Sections \ref{sec:Distance-Matrix} and \ref{sec:similarity-matrix}, respectively. We refer to the caption for a brief discussion of the selectivity parameter $\rho$ and the spatial scale in terms of $|\mc{N}_{\veps}|$. 

The reader familiar with total variation based denoising, where a \textit{single} parameter is only used to control the influence of regularization, may ask why \textit{two} parameters are used in the present approach and if they are necessary. We refer again to Figure \ref{fig:rgb-diffusion} and the caption where the separation of the physical and spatial scale based on different parameter choices is demonstrated and discussed. The total variation measure couples these scales as the co-area formula explicitly shows. As a consequence, a single parameter is only needed. On the other, larger values of this parameter lead to the well-known loss-of-contrast effect, which in the present approach can be avoided by properly choosing the parameters $\rho, |\mc{N}_{\veps}|$ corresponding to these two scales.

Figure \ref{fig:convergence} shows how convergence of the iterative algorithm \eqref{eq:def-algorithm} is affected by these two parameters. It also demonstrates that few tens of massively parallel outer iterations suffice to reach the termination criterion of Section \ref{sec:termination-criterion}.
\begin{figure}[h]
\centering
\includegraphics[width=0.4\textwidth]{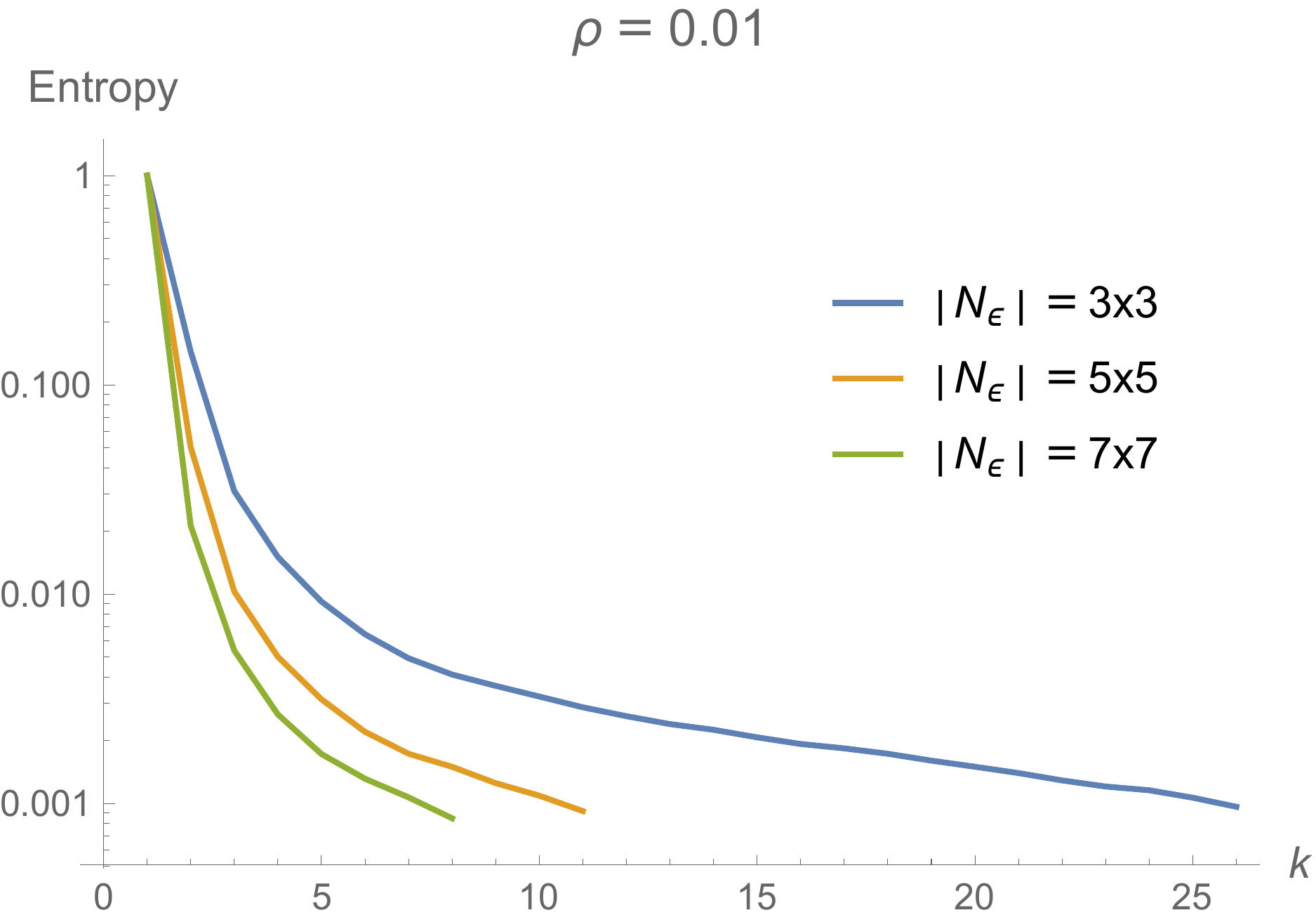}
\hspace{0.025\textwidth}
\includegraphics[width=0.4\textwidth]{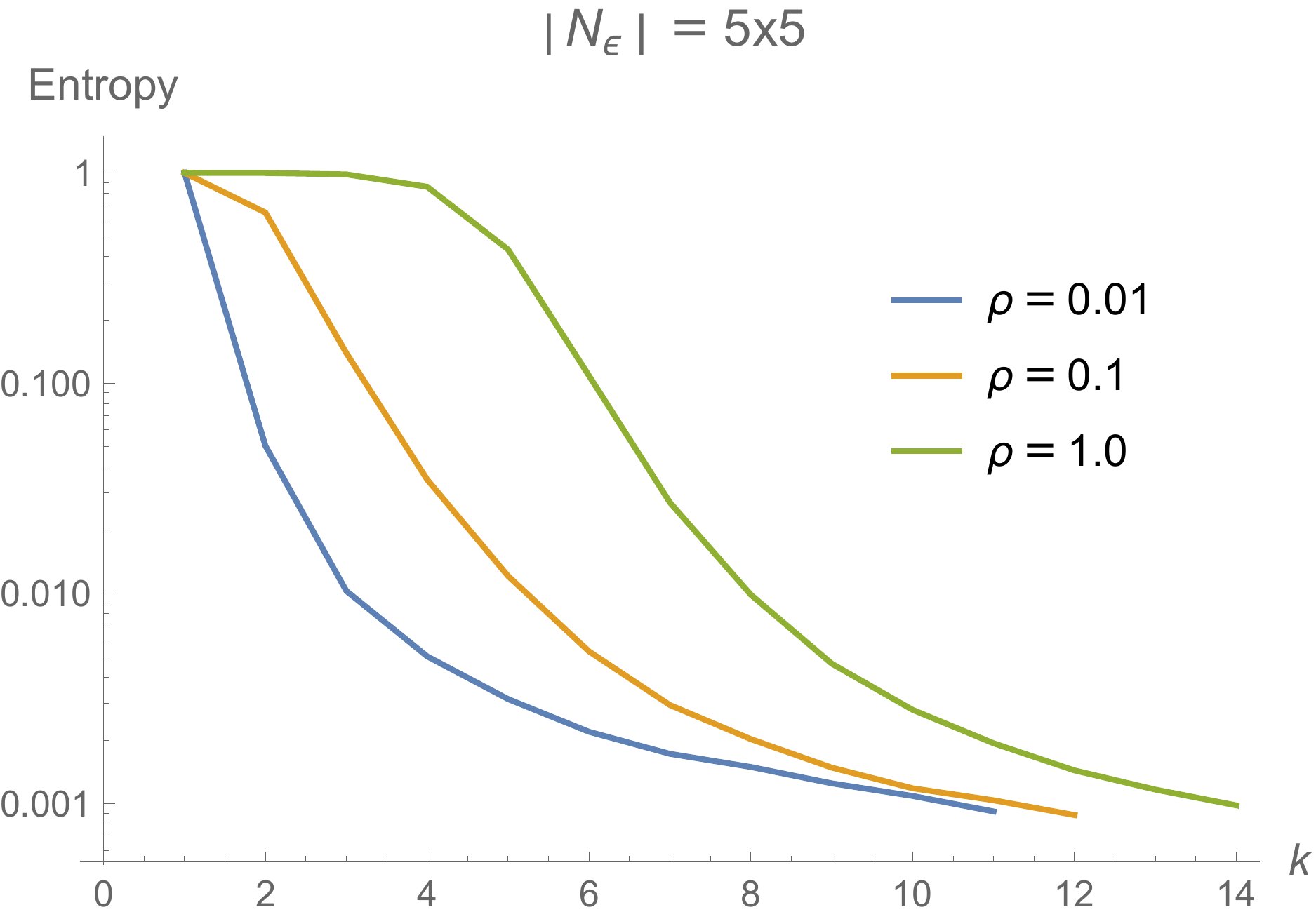}
\caption{\textbf{Parameter values and convergence rate.}
Average entropy \eqref{eq:average-entropy} of the assignment vectors $W_{i}^{(k)}$ as a function of the iteration counter $k$ and the two parameters $\rho$ and $|\mc{N}_{\veps}|$, for the labeling task illustrated by Figure \ref{fig:rgb-diffusion}. The left panel shows that despite high selectivity in terms of a small value of $\rho$, small spatial scales necessitate to resolve more conflicting assignments through propagating information by geometric spatial averaging. As a consequence, more iterations are needed to achieve convergence and a labeling. The right panel, on the other hand, shows that at a fixed spatial scale $|\mc{N}_{\veps}|$ higher selectivity leads to faster convergence, because outliers are simply removed from the averaging process, whereas low selectivity leads to an assignment (labeling) taking all data into account.
}
\label{fig:convergence}
\end{figure}

\begin{figure}
\centering
\begin{subfigure}[c]{0.175\textwidth}
\begin{subfigure}[b]{\textwidth}
\includegraphics[width=\textwidth]{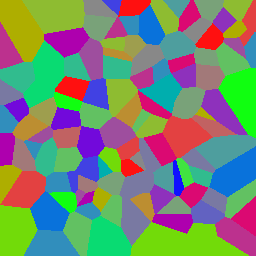} 
\caption{
Ground truth image.
}
\label{fig:denoising-dat}
\end{subfigure}% 
\\[0.1\textwidth]%%%%%%
\begin{subfigure}[b]{\textwidth}
\includegraphics[width=\textwidth]{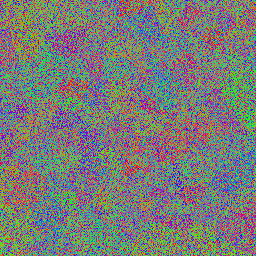}
\caption{
Noisy input image.
}
\label{fig:denoising-NoisyDat}
\end{subfigure}
\end{subfigure}%
%%%%%%%%%%%%%%
\hspace{0.05\textwidth}%
%%%%%%%%%%%%%%
\begin{subfigure}[c]{0.6\textwidth}
\centering
\begin{subfigure}[b]{0.3\textwidth}
\includegraphics[width=\textwidth]{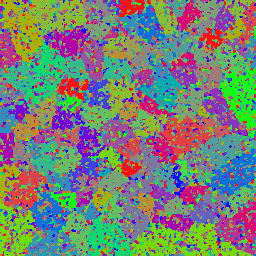}
\caption{
$\rho=0.01$, $|\mc{N}_{\mc{E}}|=3 \times 3$
}
\label{fig:denoisingResult=11}
\end{subfigure}%
\hfill%
\begin{subfigure}[b]{0.3\textwidth}
\includegraphics[width=\textwidth]{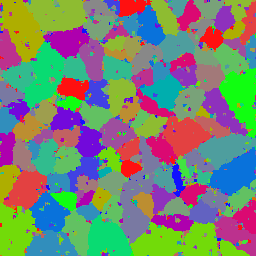}
\caption{
$\rho=0.01$, $|\mc{N}_{\mc{E}}|=5 \times 5$
}
\label{fig:denoisingResult=12}
\end{subfigure}%
\hfill%
\begin{subfigure}[b]{0.3\textwidth}
\includegraphics[width=\textwidth]{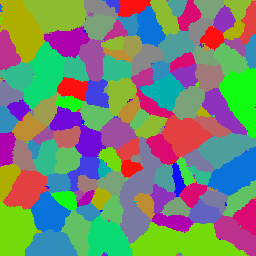}
\caption{
$\rho=0.01$, $|\mc{N}_{\mc{E}}|=7 \times 7$
}
\label{fig:denoisingResult=13}
\end{subfigure}
\\%%%%%%%%%%%
\begin{subfigure}[b]{0.3\textwidth}
\includegraphics[width=\textwidth]{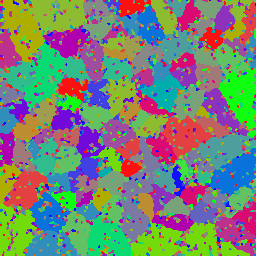}
\caption{
$\rho=0.1$, $|\mc{N}_{\mc{E}}|=3 \times 3$
}
\label{fig:denoisingResult=21}
\end{subfigure}%
\hfill%
\begin{subfigure}[b]{0.3\textwidth}
\includegraphics[width=\textwidth]{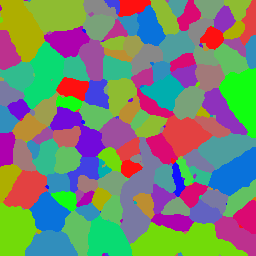}
\caption{
$\rho=0.1$, $|\mc{N}_{\mc{E}}|=5 \times 5$
}
\label{fig:denoisingResult=22}
\end{subfigure}%
\hfill%
\begin{subfigure}[b]{0.3\textwidth}
\includegraphics[width=\textwidth]{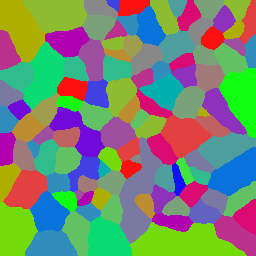}
\caption{
$\rho=0.1$, $|\mc{N}_{\mc{E}}|=7 \times 7$
}
\label{fig:denoisingResult=23}
\end{subfigure}
\\%%%%%%%%%%%
\begin{subfigure}[b]{0.3\textwidth}
\includegraphics[width=\textwidth]{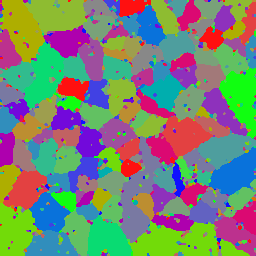}
\caption{
$\rho=1.0$, $|\mc{N}_{\mc{E}}|=3 \times 3$
}
\label{fig:denoisingResult=31}
\end{subfigure}%
\hfill%
\begin{subfigure}[b]{0.3\textwidth}
\includegraphics[width=\textwidth]{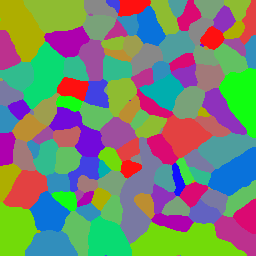}
\caption{
$\rho=1.0$, $|\mc{N}_{\mc{E}}|=5 \times 5$
}
\label{fig:denoisingResult=32}
\end{subfigure}%
\hfill%
\begin{subfigure}[b]{0.3\textwidth}
\includegraphics[width=\textwidth]{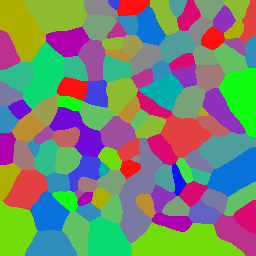}
\caption{
$\rho=1.0$, $|\mc{N}_{\mc{E}}|=7 \times 7$
}
\label{fig:denoisingResult=33}
\end{subfigure}
\end{subfigure}
\caption{
\textbf{Parameter influence on labeling.} Panels (a) and (b) show a ground-truth image and noisy input data. Both images and the prior data set $\mc{P}_{\mc{F}}$ are composed of 31 color vectors. Each color vector is encoded as a vertex of the simplex $\Delta_{30}$. This results in unit distances between all colors and thus enables an unbiased assessment of the impact of geometric averaging and the two parameter values $\rho, |\mc{N}_{\veps}|$. Panels (c)-(k) show the assignments $u(W^{\ast})$ for various parameter values where $W^{\ast}$ maximizes the objective function \eqref{eq:objective-orig}. The spatial scale $|\mc{N}_{\mc{E}}|$ increases from left to right. The results illustrate the compromise between sensitivity to noise and to the geometry of signal transitions. The selectivity parameter $\rho$ increases from top to bottom. If $\rho$ is chosen too small, then there is a tendency to noise-induced oversegmentation, in particular at small spatial scales $|\mc{N}_{\mc{E}}|$. Note, however, that depending on the application, the ability to separate the physical and the spatial scale in order to recognize outliers with small spatial support, while performing diffusion at a larger spatial scale as in panels (c),(d),(f),(i), may be beneficial. We point out that this separation of the physical and spatial scales (image range vs.~image domain) is not possible with total variation based regularization where these scales are coupled through the co-area formula. All results were computed using the assignment mapping \eqref{eq:trivial-assignment-map} \textit{without} rounding. This shows that the termination criterion of Section \ref{sec:termination-criterion}, illustrated by Figure \ref{fig:convergence} leads to (almost) unique assignments.
}
\label{fig:rgb-diffusion}
\end{figure}

\subsection{Vector-Valued Data}
\label{sec:vector-supervised}

Let $f^{i} \in \R^{d}$ denote vector-valued image data or extracted feature vectors at locations $i \in [m]$, and let
\begin{equation} \label{eq:PF-color-supervised}
\mc{P}_{\mc{F}} = \{f^{\ast 1},\dotsc,f^{\ast n}\}
\end{equation}
denote the prior information given by prototypical feature vectors. In the example that follows below, $f^{i}$ will be a RGB-color vector. It should be clear, however, that \textit{any} feature vector of arbitrary dimension $d$ could be used instead, depending on the application at hand. We used the distance function
\begin{equation} \label{eq:df-vector-l1}
d_{\mc{F}}(f^{i},f^{\ast j}) = \frac{1}{d} \|f^{i}-f^{\ast j}\|_{1},
\end{equation}
with the normalizing factor $1/d$ to make the choice of the parameter $\rho$ insensitive with respect to the dimension $d$ of the feature space.
Given an optimal assignment matrix $W^{\ast}$ as solution to \eqref{eq:objective-orig}, the prior information assigned to the data is given by the assignment mapping
\begin{equation} \label{eq:u(W)-vectorValued}
u^{i} = u^{i}(W^{\ast}) = \EE_{W_{i}^{\ast}}[\mc{P}_{\mc{F}}],\qquad i \in [m],
\end{equation}
which merely replaces each data vector $f^{i}$ by the prior vector $f^{\ast j}$ assigned to it through $W_{i}^{\ast}$.

\vspace{0.25cm}
Figure \ref{fig:mandrill} shows the assignment of 20 prototypical color vectors to a color image for various values of the spatial scale parameter $|\mc{N}_{\veps}|$, while keeping the selectivity parameter $\rho$ fixed. As a consequence, the induced assignments and image partitions exhibit a natural coarsening effect in the spatial domain.
%
%%%%%%%%%%%%%%%%%%%%%%%%%%
\begin{figure}
\centering
\begin{subfigure}[b]{0.7\textwidth}
\centering
\includegraphics[width=0.24\textwidth]{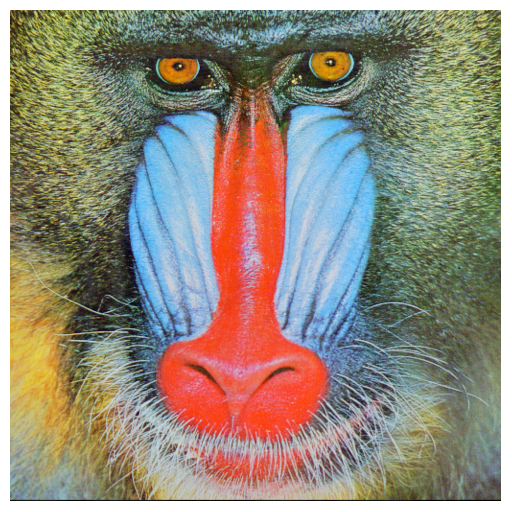}
%\hspace{0.005\textwidth}
\includegraphics[width=0.24\textwidth]{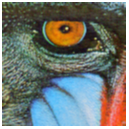}
\hfill
\includegraphics[width=0.24\textwidth]{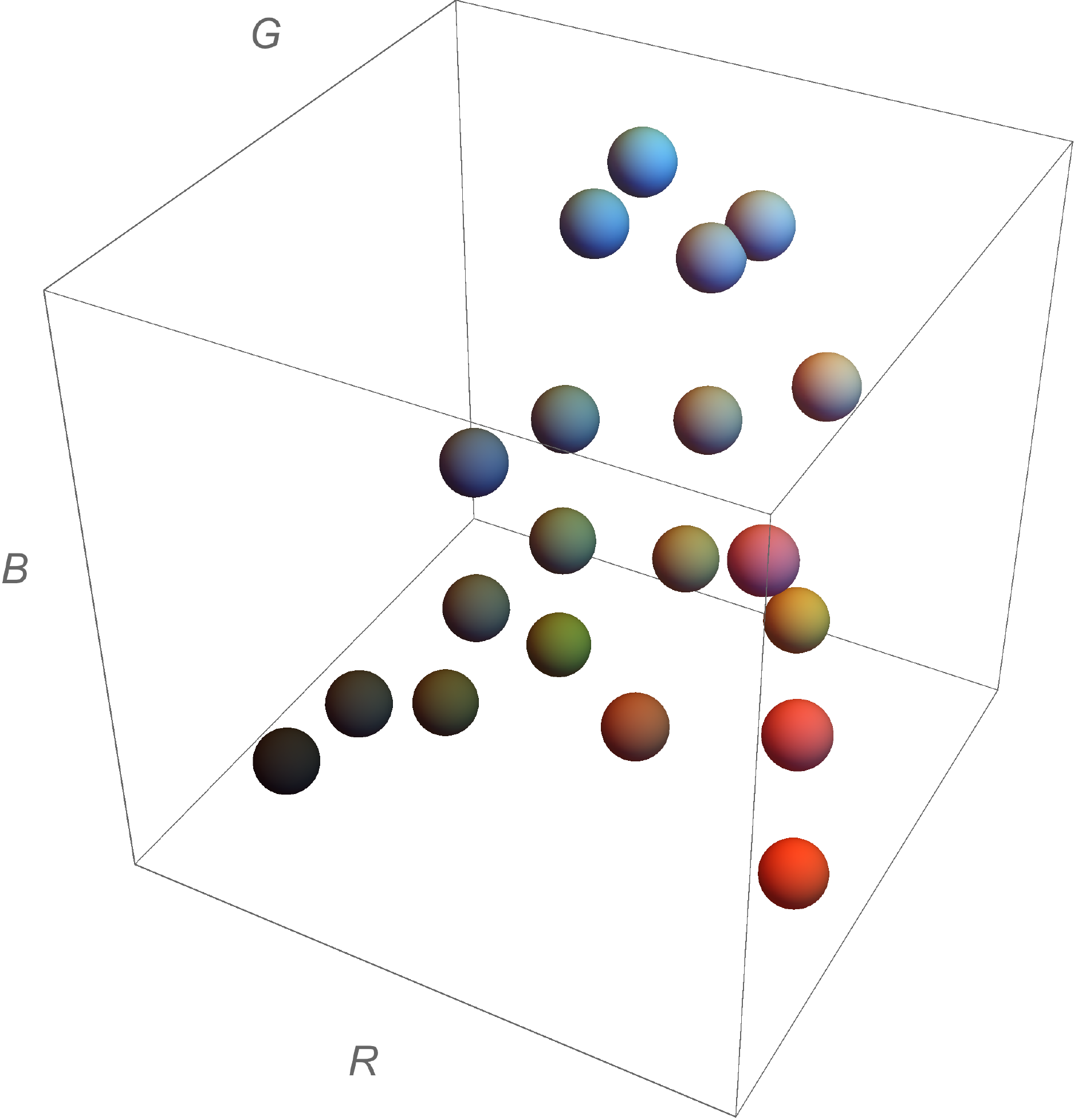}
\caption{
Input image (left) and a section of it. 20 color vectors (right) forming the set prior data set $\mc{P}_{\mc{F}}$ according to Eq.~\eqref{eq:PF-color-supervised}.
}
\label{fig:mandrill-data}
\end{subfigure}
%%%%%%%%%
\begin{subfigure}[b]{0.7\textwidth}
\centering
\includegraphics[width=0.24\textwidth]{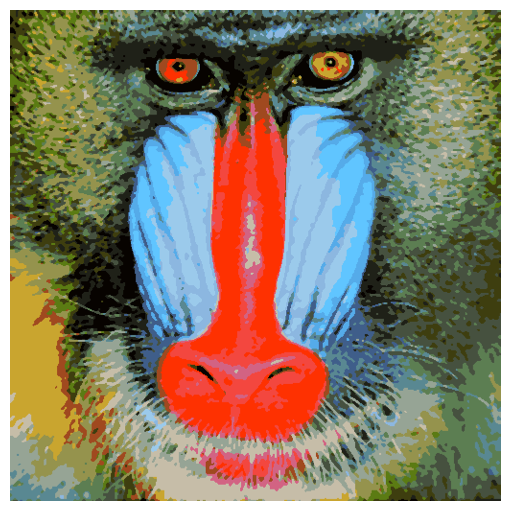}
\hfill
\includegraphics[width=0.24\textwidth]{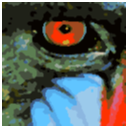}
\hfill
\includegraphics[width=0.24\textwidth]{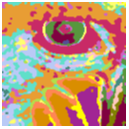}
\hfill
\includegraphics[width=0.24\textwidth]{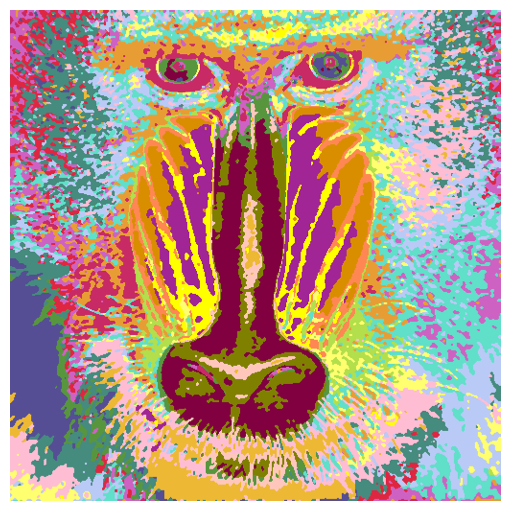}
\caption{
Assignment $u(W^{\ast})$, $|\mc{N}_{\veps}|=3 \times 3,\, \rho=0.01$.
}
\label{fig:mandrill-n1}
\end{subfigure}
%%%%%%%%%%%%%%%%%%%%
%%%%%%%%%
\begin{subfigure}[b]{0.7\textwidth}
\centering
\includegraphics[width=0.24\textwidth]{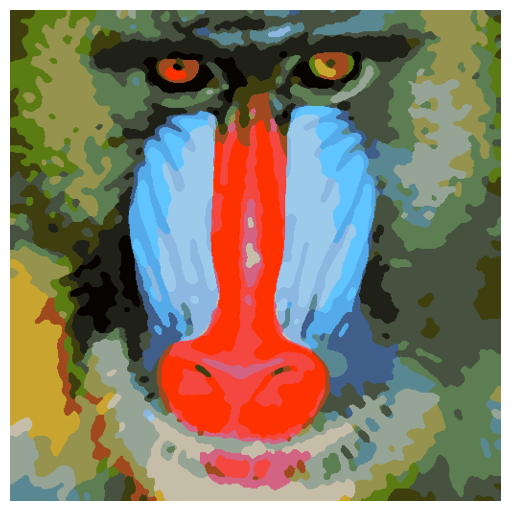}
\hfill
\includegraphics[width=0.24\textwidth]{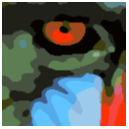}
\hfill
\includegraphics[width=0.24\textwidth]{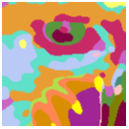}
\hfill
\includegraphics[width=0.24\textwidth]{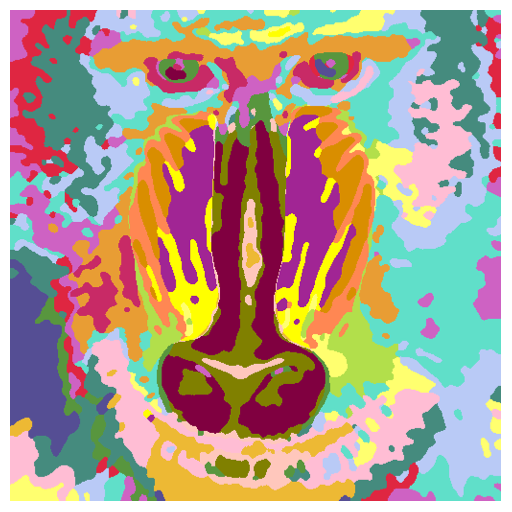}
\caption{
Assignment $u(W^{\ast})$, $|\mc{N}_{\veps}|=7 \times 7,\, \rho=0.01$.
}
\label{fig:mandrill-n3}
\end{subfigure}
%%%%%%%%%%%%%%%%%
\begin{subfigure}[b]{0.7\textwidth}
\centering
\includegraphics[width=0.24\textwidth]{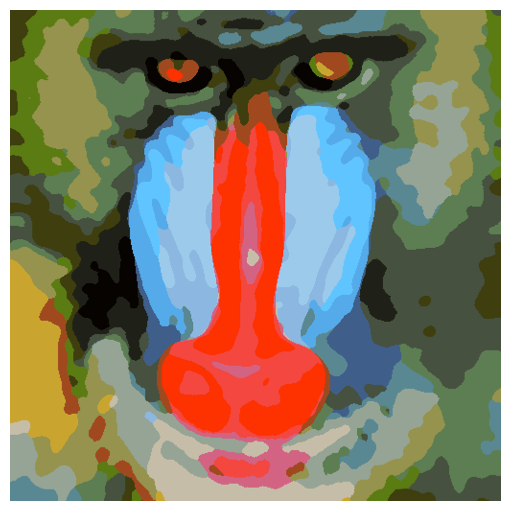}
\hfill
\includegraphics[width=0.24\textwidth]{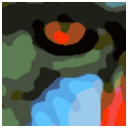}
\hfill
\includegraphics[width=0.24\textwidth]{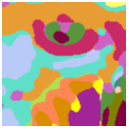}
\hfill
\includegraphics[width=0.24\textwidth]{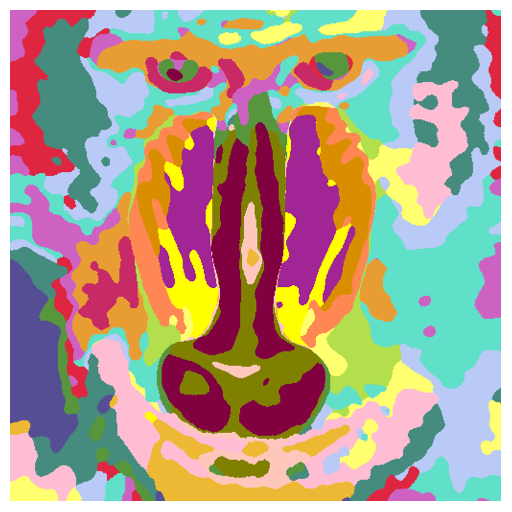}
\caption{
Assignment $u(W^{\ast})$, $|\mc{N}_{\veps}|=11 \times 11,\, \rho=0.01$.
}
\label{fig:mandrill-n5}
\end{subfigure}
%%%%%%%%%%%%%%%%%%%%%%%%%%%
\caption{
\textbf{Image labeling at different spatial scales.} The two rightmost columns show the same information using a random color code for the assignment of the 20 prior vectors to pixel locations, to highlight the induced image partitions. Increasing the spatial scale $|\mc{N}_{\veps}|$ for a \textit{fixed} value of the selectivity parameter $\rho$ induces a natural coarsening of the assignments and the corresponding image partitions along the spatial scale.
}
\label{fig:mandrill}
\end{figure}
%%%%%%%%%%%%%%%%%%%%%%%%%%
\begin{figure}
\centering
\begin{subfigure}[b]{0.1\textwidth}
\includegraphics[width=\textwidth]{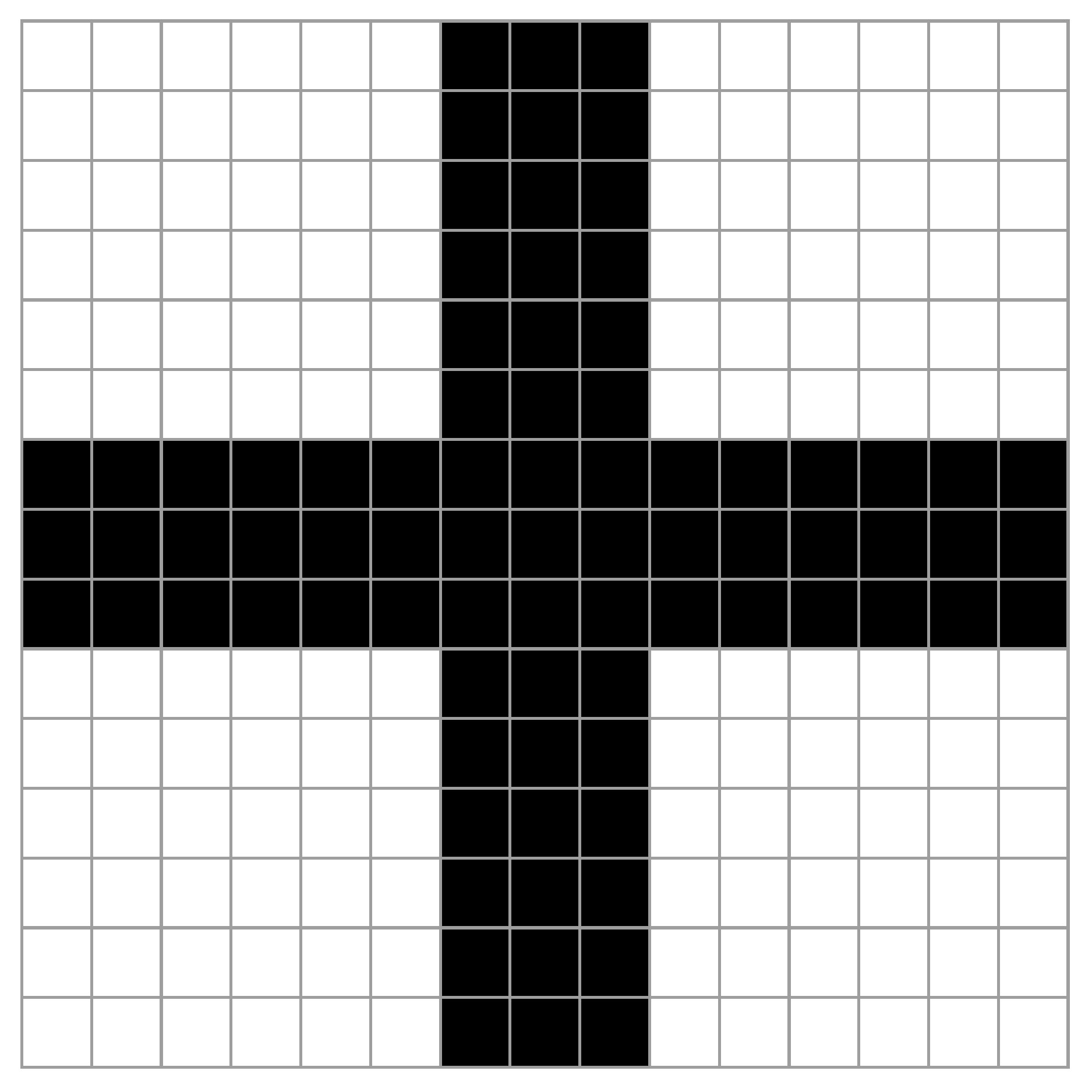}
\caption{ 
Patch generating the dictionary by translation.
}
\label{fig:roof-uPrototype}
\end{subfigure}
\hspace{0.01\textwidth}
\begin{subfigure}[b]{0.2\textwidth}
\includegraphics[width=\textwidth]{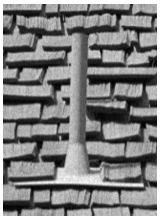}
\caption{ 
Input image $f$. \\ $\text{ }$
}
\label{fig:roof-dat}
\end{subfigure}
\hspace{0.01\textwidth}
\begin{subfigure}[b]{0.2\textwidth}
\includegraphics[width=\textwidth]{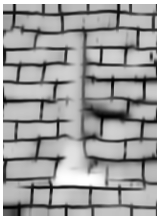}
\caption{ 
Patch assignment $u(W^{\ast})$.
}
\label{fig:roof-u}
\end{subfigure}
\hspace{0.01\textwidth}
\begin{subfigure}[b]{0.2\textwidth}
\includegraphics[width=\textwidth]{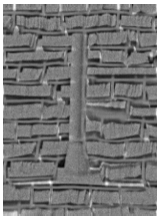}
\caption{ 
Residual image $v(W^{\ast})=f-u(W^{\ast})$.
}
\label{fig:roof-v}
\end{subfigure}
\caption{
A patch (a) supposed to represent prior knowledge about the structure of an image $f$ (b). The dictionary $\mc{P}_{\mc{F}}$ of Eq.~\eqref{eq:PF-patches-general} was generated by all translations of (a) and assigned to the image (b), using a distance $d_{\mc{F}}$ that adapts the two grayvalues of each template to the data -- see Eqns.~\eqref{eq:df-vector-l1-2adapted} and \eqref{eq:roof-fi-values}. The resulting assignment $u(W^{\ast})$ is depicted by (c). Panel (d) shows the residual image $v(W^{\ast}) := f-u(W^{\ast})$ by substracting (c) from (b) (rescaled for better visibility). The result (c) illustrates how the approximation of $f$ is restricted by the prior knowledge, leading to normalized signal transitions regarding both the spatial geometry and the signal values. By maximizing the objective \eqref{eq:objective-orig}, a patch-consistent and dense cover of the image is computed. It induces a strong nonlinear image filtering effect by fusing through assignment for each single pixel value more than 200 predictions of possible values based on the patch dictionary $\mc{P}_{\mc{F}}$. The approach enables to model additive image decompositions $f = u+v$ (i.e.~image = geometry + texture \& noise) for  specific image classes, which are implicitly represented by the dictionary $\mc{P}_{\mc{F}}$.
}
\label{fig:roof}
\end{figure}

\subsection{Patches}
\label{sec:patches-supervised}

Let $f^{i}$ denote a patch of raw image data (or, more generally, a patch of features vectors)
\begin{equation} \label{eq:def-Nf}
f^{ij} \in \R^{d},\qquad j \in \mc{N}_{p}(i),\qquad i \in [m],
\end{equation}
centered at location $i \in [m]$ and indexed by $\mc{N}_{p}(i) \subset \mc{V}$ (subscript $p$ indicates neighborhoods for \underline{p}atches). With each entry $j \in \mc{N}_{p}(i)$, we associate the Gaussian weight 
\begin{equation} \label{eq:Gaussian-weight}
w^{p}_{ij} := G_{\sigma}(\|x^{i}-x^{j}\|),\qquad i,j \in \mc{N}_{p}(i),
\end{equation}
where the vectors $x^{i}, x^{j} \in \R^{d}$ correspond to the locations in the image domain indexed by $i, j \in \mc{V}$. Specifically, $w^{p}$ is chosen to be the discrete impulse response of a Gaussian lowpass filter supported on $\mc{N}_{p}(i)$, so that the scale $\sigma$ directly depends on the patch size and does not need to be chosen by hand. Such downweighting of values, that are less close to the center location of a patch, is an established elementary technique for reducing boundary and ringing effects of patch (``window'')-based image processing.

The prior information is given in terms of $n$ prototypical patches
\begin{equation} \label{eq:PF-patches-general}
\mc{P}_{\mc{F}} = \{f^{\ast 1},\dotsc,f^{\ast n}\},
\end{equation}
and a corresponding distance
\begin{equation} \label{eq:df-patches}
d_{\mc{F}}(f^{i},f^{\ast j}),\qquad i \in [m],\quad j \in [n].
\end{equation}
There are many ways to choose this distance depending on the application at hand. We refer to the Examples \ref{ex:roof} and \ref{ex:fprint} below. Expression \eqref{eq:df-patches} is based on the tacit assumption that patch $f^{\ast j}$ is centered at $i$ and indexed by $\mc{N}_{p}(i)$ as well.

Given an optimal assignment matrix $W^{\ast}$, it remains to specify how prior information is assigned to every location $i \in \mc{V}$, resulting in a vector  $u^{i} = u^{i}(W^{\ast})$ that is the overall result of processing the input image $f$. Location $i$ is affected by patches that overlap with $i$. Let us denote the indices of these patches by 
\begin{equation}
\mc{N}_{p}^{i \leftarrow j} := \{ j \in \mc{V} \colon i \in \mc{N}_{p}(j) \}.
\end{equation}
Every such patch is centered at location $j$ to which prior patches are assigned by
\begin{equation} \label{eq:Wj-prior-patch-mean}
\EE_{W_{j}^{\ast}}[\mc{P}_{\mc{F}}] = \sum_{k \in [n]} W_{jk}^{\ast} f^{\ast k}.
\end{equation}
Let location $i$ be indexed by $i_{j}$ in patch $j$ (local coordinate inside patch $j$). Then, by summing over all patches indexed by $\mc{N}_{p}^{i \leftarrow j}$ whose supports include location $i$, and by weighting the contributions to location $i$ by the corresponding weights \eqref{eq:Gaussian-weight}, we obtain the vector
\begin{equation} \label{eq:u(W)-patches-a}
u^{i} = u^{i}(W^{\ast})
= \frac{1}{\sum_{j' \in \mc{N}_{p}^{i \leftarrow j}} w^{p}_{j' i_{j}}}
\sum_{j \in \mc{N}_{p}^{i \leftarrow j}} w^{p}_{j i_{j}} 
\sum_{k \in [n]} W_{jk}^{\ast} f^{\ast ki_{j}} 
\quad\in\quad \R^{d},
\end{equation}
that is assigned by $W^{\ast}$ to location $i$. This expression looks more clumsy than it actually is. In words, the vector $u^{i}$ assigned to location $i$ is the convex combination of vectors contributed from patches overlapping with $i$, that itself are formed as convex combinations of prior patches. In particular, if we consider the common case of \emph{equal} patch supports $\mc{N}_{p}(i)$ for every $i$, that additionally are \emph{symmetric} with respect to the center location $i$, then $\mc{N}_{p}^{i \leftarrow j} = \mc{N}_{p}(i)$. As a consequence, due to the symmetry of the weights \eqref{eq:Gaussian-weight}, the first sum of \eqref{eq:u(W)-patches-a} sums up all weights $w^{p}_{ij}$. Hence, the normalization factor on the right-hand side of \eqref{eq:u(W)-patches-a} equals $1$, because the low-pass filter $w^{p}$ preserves the zero-order moment (mean) of signals. Furthermore, it then makes sense to denote by $(-i)$ the location $i_{p}$ corresponding to $i$ in patch $j$. Thus \eqref{eq:u(W)-patches-a} becomes
\begin{equation}
u^{i} = u^{i}(W^{\ast}) = \sum_{j \in \mc{N}_{p}(i)} w^{p}_{j (-i)} 
\sum_{k \in [n]} W_{jk}^{\ast} f^{\ast k(-i)}.
\end{equation}
Introducing in view of \eqref{eq:Wj-prior-patch-mean} the shorthand
\begin{equation}
\EE^{i}_{W_{j}^{\ast}}[\mc{P}_{\mc{F}}] := \sum_{k \in [n]} W_{jk}^{\ast} f^{\ast k(-i)}
\end{equation}
for the vector assigned to $i$ by the convex combination of prior patches assigned to $j$, we finally rewrite \eqref{eq:u(W)-patches-a} due the symmetry $w^{p}_{j(-i)} = w^{p}_{ji} = w^{p}_{ij}$ in the more handy form\footnote{For locations $i$ close to the boundary of the image domain where patch supports $\mc{N}_{p}(i)$ shrink, the definition of the vector $w^{p}$ has to be adapted accordingly.}
\begin{equation} \label{eq:u(W)-patches-b}
u^{i} = u^{i}(W^{\ast}) = \EE_{w^{p}}\big[\EE^{i}_{W_{j}^{\ast}}[\mc{P}_{\mc{F}}]\big].
\end{equation}
The inner expression represents the assignment of prior vectors to location $i$ by fitting prior patches to all locations $j\in \mc{N}(i)$. The outer expression fuses the assigned vectors. If they were all the same, the outer operation would have no effect, of course.

We discuss further properties of this approach by concrete examples. 
\begin{example}[Patch Assignment]\label{ex:roof} 
Figure \ref{fig:roof} shows an image $f$ and the corresponding assignment $u(W^{\ast})$ based on a patch dictionary $\mc{P}_{\mc{F}}$ that was formed as explained in the caption.

We chose the distance $d_{\mc{F}}$ of Eq.~\eqref{eq:df-vector-l1}, 
\begin{equation} \label{eq:df-vector-l1-2adapted}
d_{\mc{F}}(f^{i},f^{\ast j}) = \frac{1}{|\mc{N}_{p}(i)|} \|f^{i}-f^{\ast j(i)}\|_{1},
\end{equation}
where here the arguments $f^{i}, f^{\ast j}$ stand for the vectorized scalar-valued patches centered at location $i$, after adapting each prior template $f^{\ast j}$ at each pixel location $i$ to the data $f$, denoted by $f^{\ast j}=f^{\ast j(i)}$ in \eqref{eq:df-vector-l1-2adapted}. Each such template takes two values that were adapted to the template $f^{i}$ to which it is compared, i.e. 
\begin{subequations} \label{eq:roof-fi-values}
\begin{align}
f^{\ast j(i)}_{k} &\in \{f^{i}_{\text{low}}, f^{i}_{\text{high}}\},
\quad \forall k, 
\intertext{where}
f^{i}_{\text{low}} &= \mrm{median}\big\{f^{i}_{j} \colon j \in \mc{N}_{p}(i),\; f^{i}_{j} < \mrm{median}\{f^{i}_{j}\}_{j \in \mc{N}_{p}(i)}\big\}, \\
f^{i}_{\text{high}} &= \mrm{median}\big\{f^{i}_{j} \colon j \in \mc{N}_{p}(i),\; f^{i}_{j} \geq \mrm{median}\{f^{i}_{j}\}_{j \in \mc{N}_{p}(i)}\big\}.
\end{align}
\end{subequations}

The result $u^{\ast}=u(W^{\ast})$ demonstrates
\begin{itemize}
\item
the ``best explanation'' of the given image $f$ in term of the (rudimentary) prior knowledge,
\item
a pronounced nonlinear filtering effect due to the consistent assignment of more than 200 patches at each pixel location and fusing the corresponding predicted values, and
\item 
a corresponding normalization of the irregular signal structure of $f$ regarding both the spatial geometry and the signal values.
\end{itemize}
It is also evident that the approach enables additive image decompositions
\begin{equation}
f = u(W^{\ast}) + v(W^{\ast}),
\end{equation}
that are more discriminative, with respect to image classes modelled by the prior data $\mc{P}_{\mc{F}}$ and a corresponding distance $d_{\mc{F}}$, than additive image decompositions achieved by convex variational approaches (see, e.g., \cite{Aujol-et-al-Decomp06}) that employ various regularizing norms, for this purpose.
\end{example}

%\begin{gather}
%\min_{s,h} \|f-(s h + c \eins)\|_{1} 
%= \min_{s,h} \sum_{i}|f_{i}-s h_{i}-c| \\
%= \min_{s,c,t}(\la \eins, t^{+} \ra + \la \eins, t^{-} \ra),\quad 
%t^{+} \geq f-s h-c \eins,\quad 
%t^{-} \geq c \eins + s h - f,\quad t^{+},t^{-},s \geq 0 \\
%= \min_{s,c,t}(\la \eins, t^{+} \ra + \la \eins, t^{-} \ra),\quad
%\bpm
%I & 0 & h & \eins \\
%0 & I & -h & -\eins
%\epm
%\bpm t^{+} \\ t^{-} \\ s \\ c \epm
%\geq 
%\bpm f \\ -f \epm
%\end{gather}

%%%
\begin{figure}
\centering
\subcaptionbox{Input image $f$.}[0.2\textwidth]{
\includegraphics[height=0.2\textwidth]{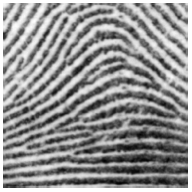}
} \hfill
\subcaptionbox{
Contourplot of a smooth image computed and subtracted from $f$ as a preprocessing step.
}[0.2\textwidth]{
\includegraphics[height=0.2\textwidth]{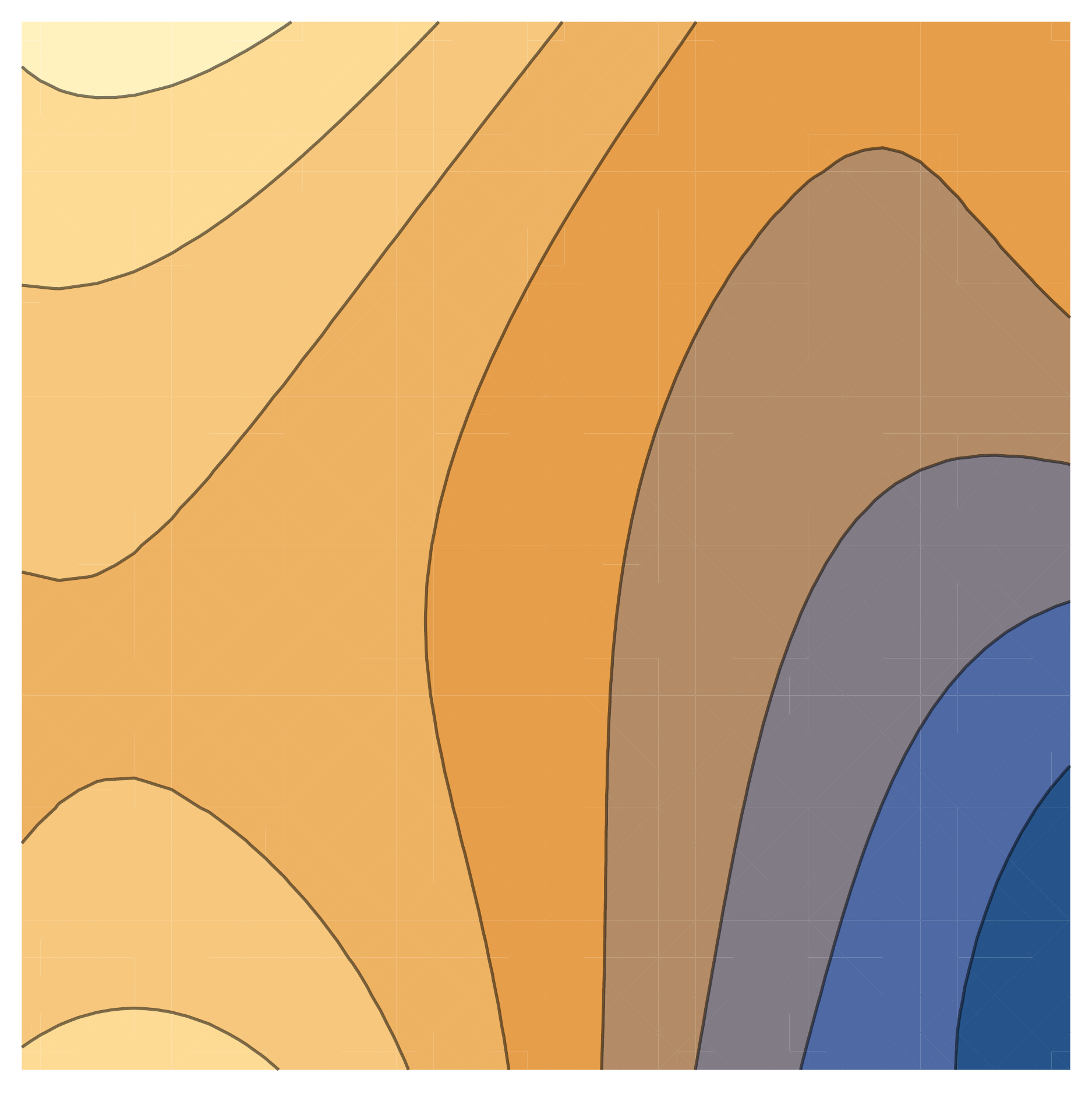}
} \hfill
\subcaptionbox{ \label{fig:fprint-dictionary}
Prior patches representing binary signal transitions at orientations $0^{\circ}, 30^{\circ},\dotsc$ (top row), and the corresponding translation invariant dictionary (bottom row). Each row of patches constitutes an \textit{equivalence class} of patches.
}[0.4\textwidth]{
\includegraphics[width=0.25\textwidth]{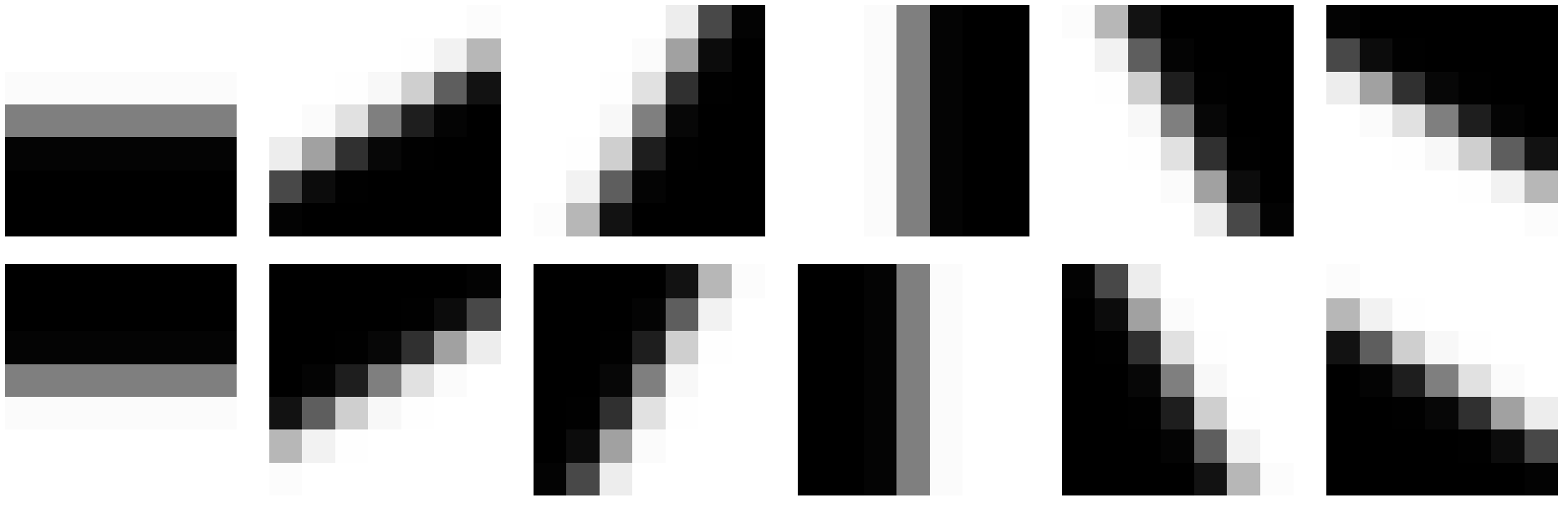} 
\includegraphics[width=0.25\textwidth]{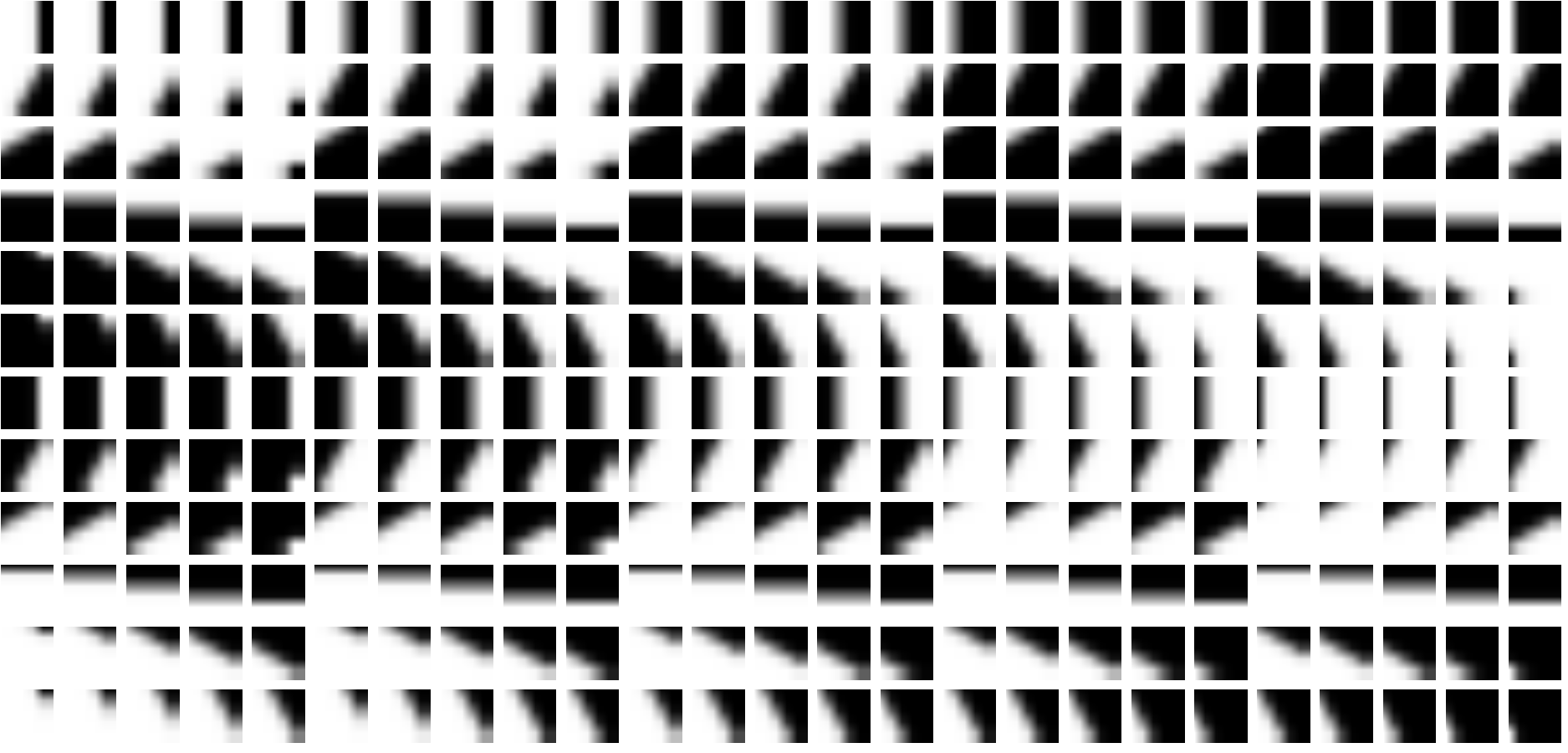}
} \hfill
\subcaptionbox{
Color code indicating oriented bright-to-dark signal transitions.
}[0.15\textwidth]{
\includegraphics[height=0.1\textwidth]{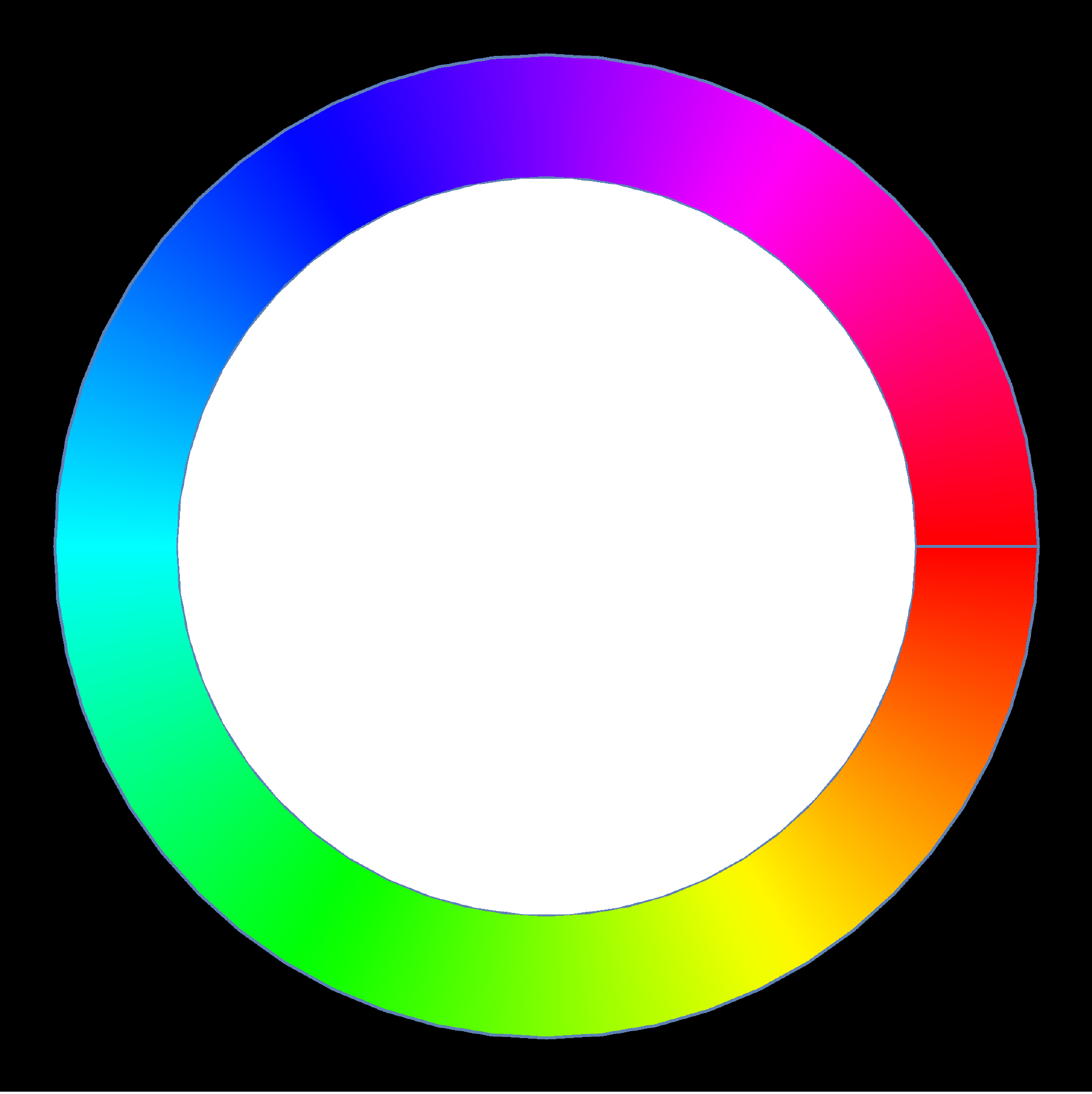}
}
%%%
\\[0.02\textwidth]
%%%
\subcaptionbox{
Assignment $u(W^{\ast})$ of $3 \times 3$ patches to image $f$ from (a). \\ ($\rho=0.02$)
}[0.3\textwidth]{
\includegraphics[height=0.2\textwidth]{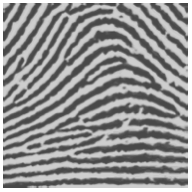}
} \hfill
\subcaptionbox{
Class label of assigned patches encoded due to (d). Black means assignment of the constant template that was added to the dictionary (c).
}[0.35\textwidth]{
\includegraphics[height=0.2\textwidth]{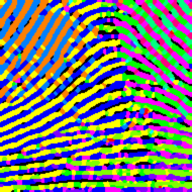}
} \hfill
\subcaptionbox{
Residual image $v(W^{\ast})=f-u(W^{\ast})$ (rescaled for visualization).
}[0.3\textwidth]{
\includegraphics[height=0.2\textwidth]{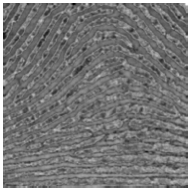}
}
%%%
\\[0.02\textwidth]
%%%
\subcaptionbox{
Assignment $u(W^{\ast})$ of $7 \times 7$ patches to image $f$ from (a). \\ ($\rho=0.02$)
}[0.3\textwidth]{
\includegraphics[height=0.2\textwidth]{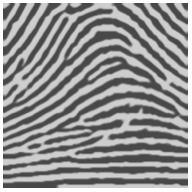}
} \hfill
\subcaptionbox{
Class label of assigned patches encoded due to (d).
}[0.3\textwidth]{
\includegraphics[height=0.2\textwidth]{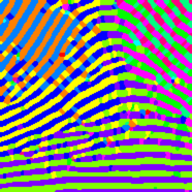}
} \hfill
\subcaptionbox{
Residual image $v(W^{\ast})=f-u(W^{\ast})$ (rescaled for visualization).
}[0.3\textwidth]{
\includegraphics[height=0.2\textwidth]{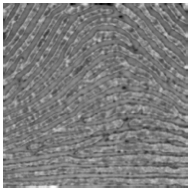}
}
\caption{
Analysis of the local signal structure of image (a) by patch assignment. This process is twofold non-local: (i) through the assignment of $3 \times 3$ patches (center row) and $7 \times 7$ patches, respectively, and (ii) due to the gradient flow \eqref{eq:W-gradient-flow} that promotes the spatially coherent assignment of patches corresponding to different orientations of signal transitions, in order to maximize the similarity objective \eqref{eq:objective-orig}. 
}
\label{fig:fprint}
\end{figure}

%%%
\begin{example}[Patch Assignment] \label{ex:fprint}
Figure \ref{fig:fprint} shows 
a fingerprint image characterized by two grey values $f^{\ast}_{\text{dark}}, f^{\ast}_{\text{bright}}$, that were extracted from the histogram of $f$ after removing a smooth function of the spatially varying mean value (panel (b)). The latter was computed by interpolating the median values for each patch of a coarse $16 \times 16$ partition of the entire image. 

Figure \ref{fig:fprint-dictionary} shows the dictionary of patches modelling the remaining binary signal transitions.
An essential difference to Example \ref{ex:roof} is the \textit{subdivision of the dictionary into classes of equivalent patches} corresponding to each orientation. The averaging process was set-up to distinguish only the assignment of patches of \textit{different} patch classes and to treat patches of the same class equally. This makes geometric averaging particularly effective if signal structures conform to a single class on larger spatial  connected supports. Moreover, it reduces the problem size to merely 13 class labels: 12 orientations at $k \cdot 30^{\circ},\, k \in [12]$ degrees, together with the single constant patch complementing the dictionary. 
 
The distance $d_{\mc{F}}(f^{i},f^{\ast j})$ between the image patch centered at $i$ and the $j$-th prior patch was chosen depending on both the prior patch and the data patch it was compared to: For the constant prior patch, the distance was
\begin{equation} \label{eq:distance-fingerprint}
d_{\mc{F}}(f^{i},f^{\ast j}) = \frac{1}{|\mc{N}_{p}(i)|}
\|f^{i}-f^{\ast}_{i} f^{\ast j}\|_{1}
\quad\text{with}\quad
f^{\ast}_{i} = \begin{cases}
f^{\ast}_{\text{dark}} & \text{if}\; \mrm{med}\{f^{i}_{j}\}_{j \in \mc{N}_{p}(i)} \leq \frac{1}{2}(f^{\ast}_{\text{dark}} + f^{\ast}_{\text{bright}}), \\
f^{\ast}_{\text{bright}} & \text{otherwise.}
\end{cases}
\end{equation}
For all other prior patches, the distance was
\begin{equation} \label{eq:df-ell1}
d_{\mc{F}}(f^{i},f^{\ast j}) = \frac{1}{|\mc{N}_{p}(i)|}
\|f^{i}-f^{\ast j}\|_{1}.
\end{equation}

The center and bottom row of Figure \ref{fig:fprint}, respectively, show the assignment $u(W^{\ast})$ of the dictionary of $3 \times 3$ patches (center row) and of $7 \times 7$ patches (bottom row). The center panels (f) and (i) depict the class labels of these assignments according to the color code of panel (d). These images display the interpretation of the image structure of $f$ from panel (a). While the assignment of patches of size $3 \times 3$ is slightly noisy, which becomes visible through the assignment of the constant template marked by black in panel (f), the assignment of $5 \times 5$ or $7 \times 7$ patches results in a robust and spatially coherent, accurate representation of the local image structure. The corresponding pronounced nonlinear filtering effect is due to the consistent assignment of a large number of patches at each pixel location and fusing the corresponding predicted values.

Panels (g) and (j) show the resulting additive image decompositions
\begin{equation}
f = u(W^{\ast}) + v(W^{\ast}),
\end{equation}
that seem difficult to achieve when using established convex variational approaches (see, e.g., \cite{Aujol-et-al-Decomp06}) that employ various regularizing norms and duality, for this purpose.

Finally, we point out that it would be straighforward to add to the dictionary further patches modelling minutiae and other features relevant to fingerprint analysis. We do not consider in this paper any application-specific aspects, however.
\end{example}

\begin{figure}
\centering
\subcaptionbox{
Uniform noise.
}[0.2\textwidth]{
\includegraphics[height=0.2\textwidth]{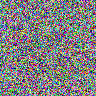}
} \hfill
\subcaptionbox{ \label{fig:u-noise-unsupervised}
Sparse assignment $u(W^{\ast})$ (displayed after rescaling) of $6^{3}$ color vectors corresponding to a uniform discretization of the rgb-cube $[0,1]^{3}$ to the image (a) yields a noise-induced random piecewise constant partition through geometric averaging (parameters: $|\mc{N}_{\veps}|=7 \times 7, \rho=0.01$). 
}[0.79\textwidth]{
\includegraphics[height=0.2\textwidth]{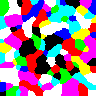}
}
%%%%%%%%%%%%%%%%%%
\\[0.05\textwidth]
%%%%%%%%%%%%%%%%%%
\subcaptionbox{ \label{fig:uPrior-statistics-noise-unsupervised}
Relative frequencies of assignment of the prior color vectors $f^{\ast j},\, j \in [6^3]$. The 8 non-zero frequencies correspond to vectors indicated in the color cube (d).
}[0.49\textwidth]{
\includegraphics[height=0.25\textwidth]{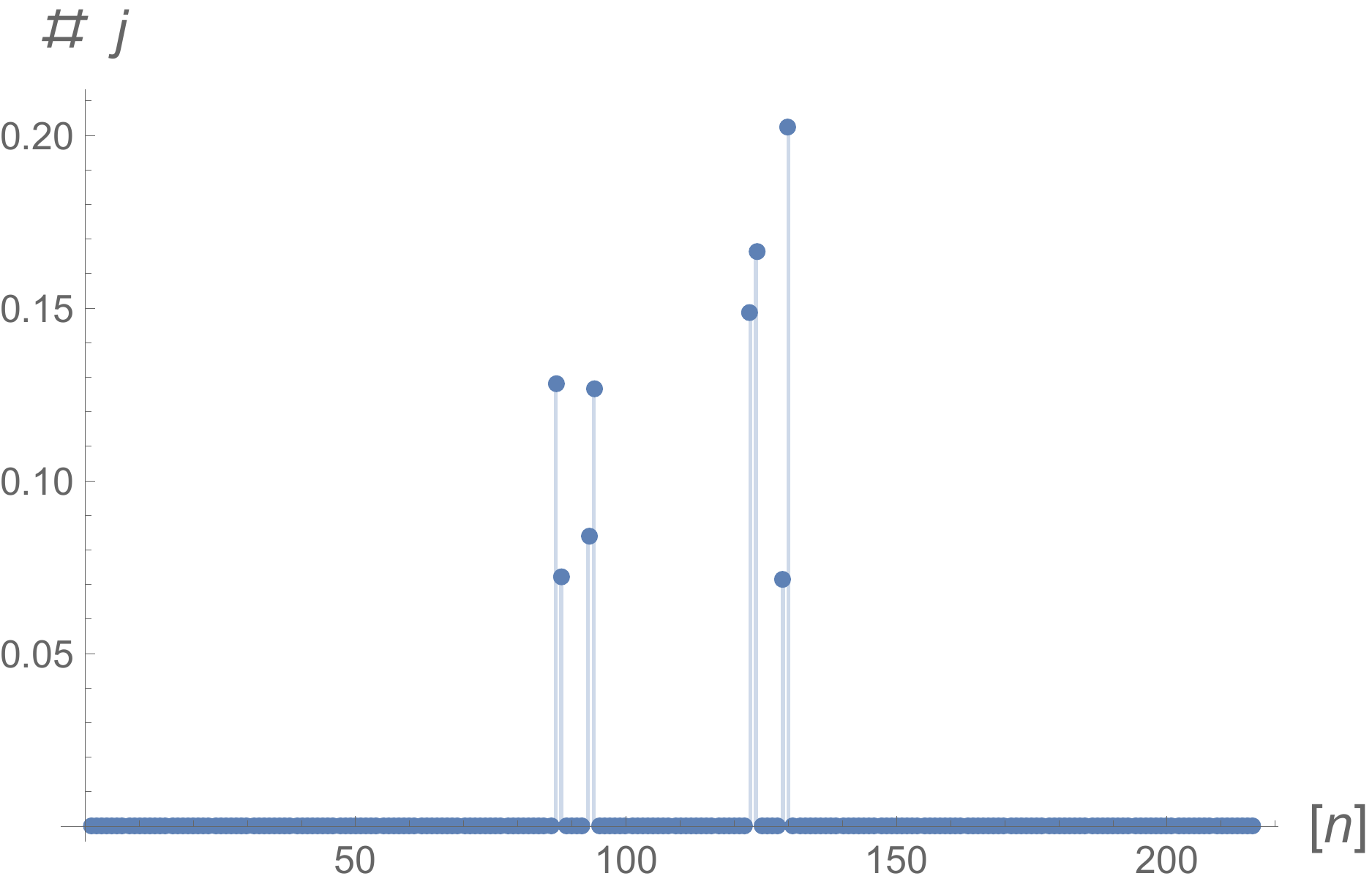}
} \hfill
\subcaptionbox{ \label{fig:uPrior-selection-noise-unsupervised}
8 color vectors (out of $6^3$) closest to grey (with equal distance) only were assigned to (a), resulting in (b). These colors look differently in (b) due to rescaling the image $u(W^{\ast})$ to $[0,1]^{3}$ for better visibility.
}[0.49\textwidth]{
\includegraphics[height=0.25\textwidth]{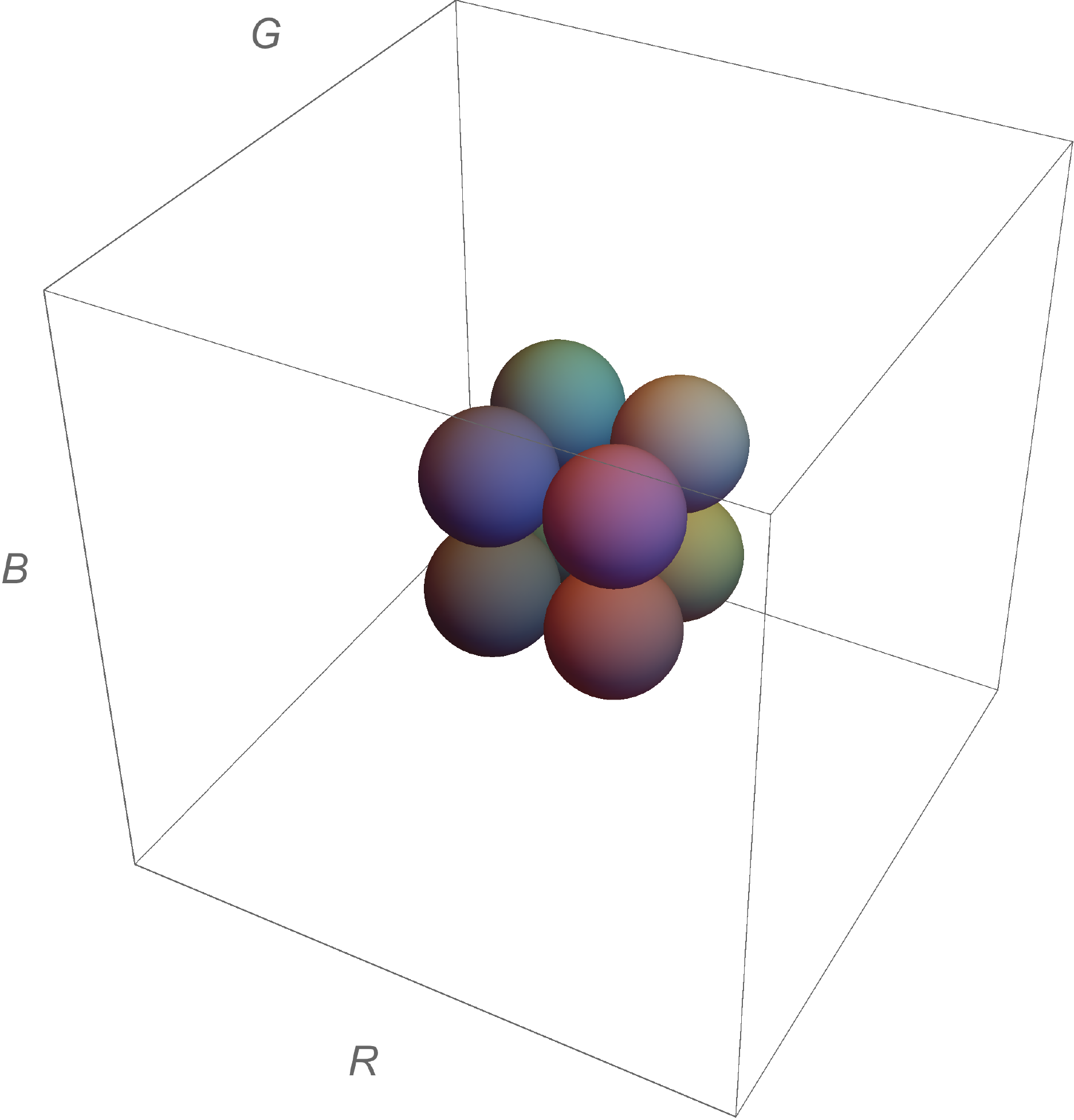}
}
\caption{
Unsupervised assignment of uniform noise (a) to itself in terms of a uniform discretization of the rgb-color cube $[0,1]^{3}$ that does not include the color grey $0.5 (1,1,1)^{\T}$. The assignment selects the 8 colors (d) closest to grey with random frequencies (c) and a spatially random partition (b) (rescaled to highlight the partition).
}
\label{fig:noise-unsupervised}
\end{figure}
%%%

%%%%%%%%%%
\subsection{Unsupervised Assignment}
\label{sec:Unsupervised-Assignment}

We consider the case that no prior information is available.

The simplest way to handle the absence of prior information is to use the given data themselves as prior information along with a suitable constraint, to enforce selection of the most important parts by \textit{self-assignment}.

In order to illlustrate this mechanism clearly, Figure \ref{fig:noise-unsupervised} shows as example the assignment of uniform noise to itself. As prior data $\mc{P}_{\mc{F}}$, we uniformly discretized the rgb-color cube $[0,1]^{3}$ at $0, 0.2, 0.4, \dotsc, 1$ along each axis, resulting in $|\mc{P}_{\mc{F}}| = 6^{3} = 216$ color vectors. Because there is no preference for any of these vectors, spatial diffusion of uniform noise at any spatial scale will inherently end up with the average color grey, which however is excluded from the prior set, by construction. Accordingly, the process terminated with a spatially random assignment of the 8 color vectors closest to grey (Figs.~\ref{fig:u-noise-unsupervised} rescaled and \ref{fig:uPrior-selection-noise-unsupervised}) solely induced by the input noise and geometric averaging at a certain scale. Figure \ref{fig:uPrior-statistics-noise-unsupervised} depicts the relative frequencies each prior vector is assigned to some location. Except for the 8 afore-mentioned vectors, all others are ignored.

A detailed elaboration of unsupervised scenarios based on our approach, for both vector- and patch-valued data, will be studied in our follow-up work (Section \ref{sec:Conclusion}).

\begin{figure}
\centering
\subcaptionbox{ \label{fig:rectangles-fgbg}
Collection of rectangular areas that result in (e) after uniform point sampling.
}[0.24\textwidth]{
\includegraphics[width=0.24\textwidth]{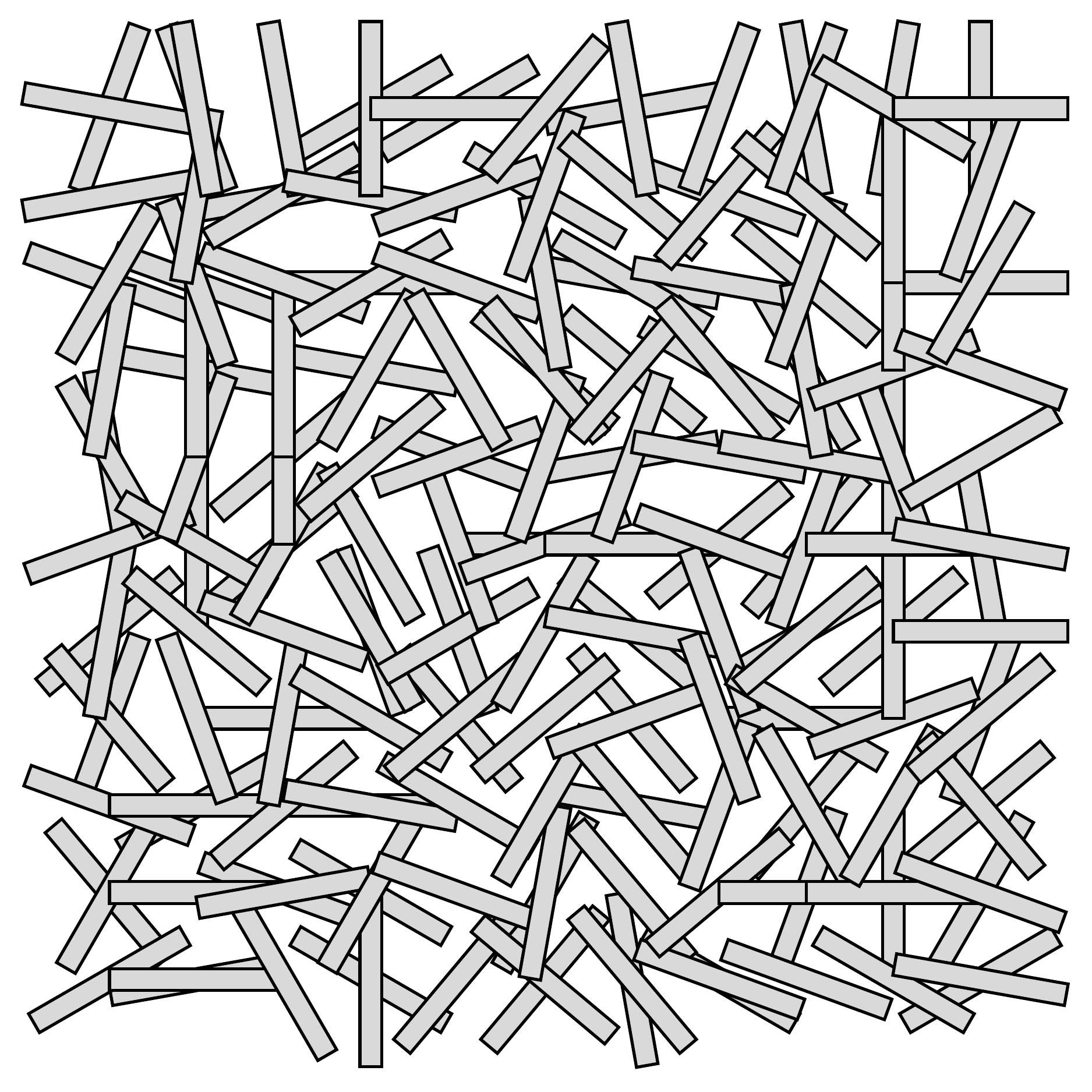}
} \hfill
\subcaptionbox{
Decomposition of the rectangles (a) into foreground (dark, cf.~(c)), and background (light, cf.~(d)).
}[0.24\textwidth]{
\includegraphics[width=0.24\textwidth]{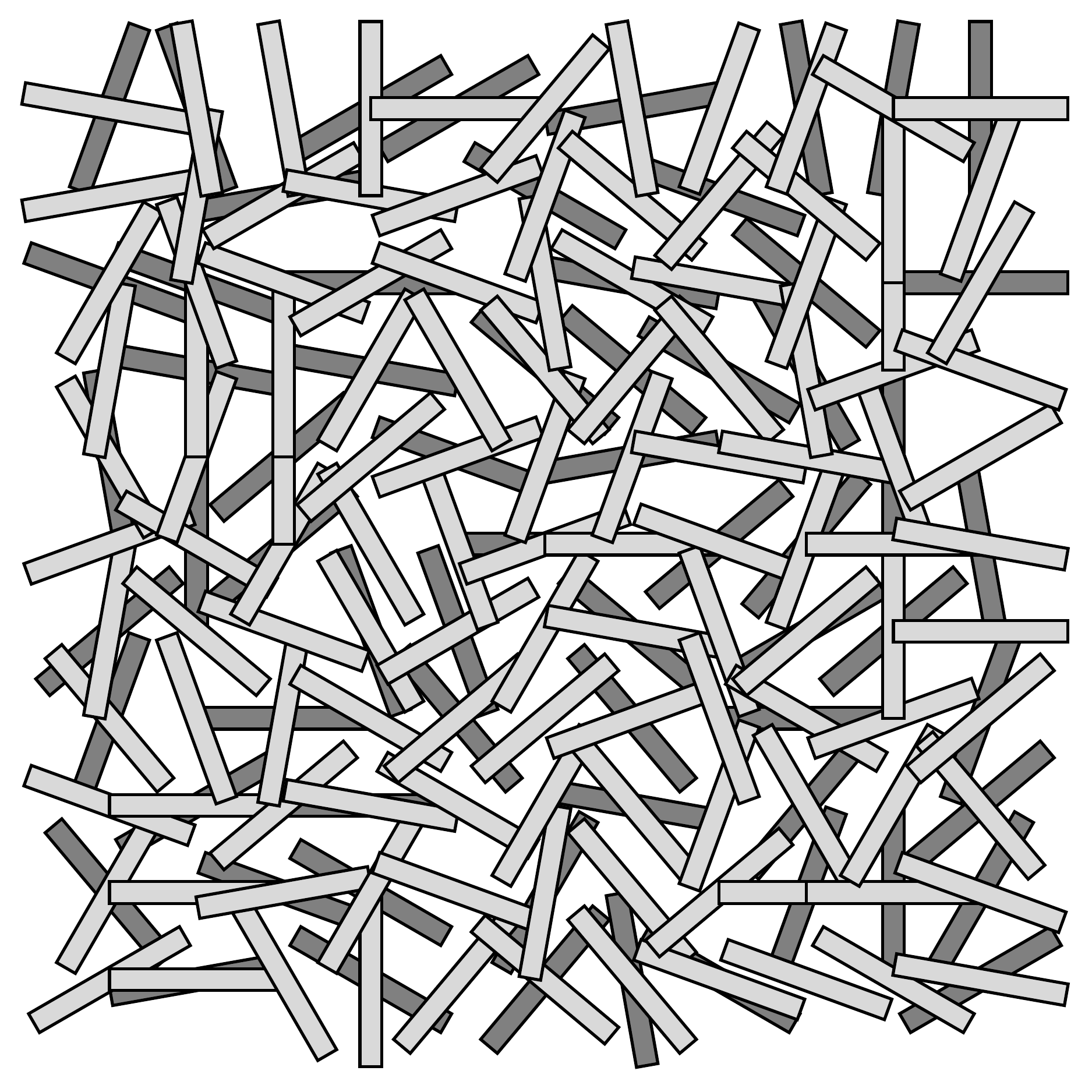}
} \hfill
\subcaptionbox{ \label{fig:rectangles-fg}
Randomly oriented foreground rectangles that do not intersect.
}[0.24\textwidth]{
\includegraphics[width=0.24\textwidth]{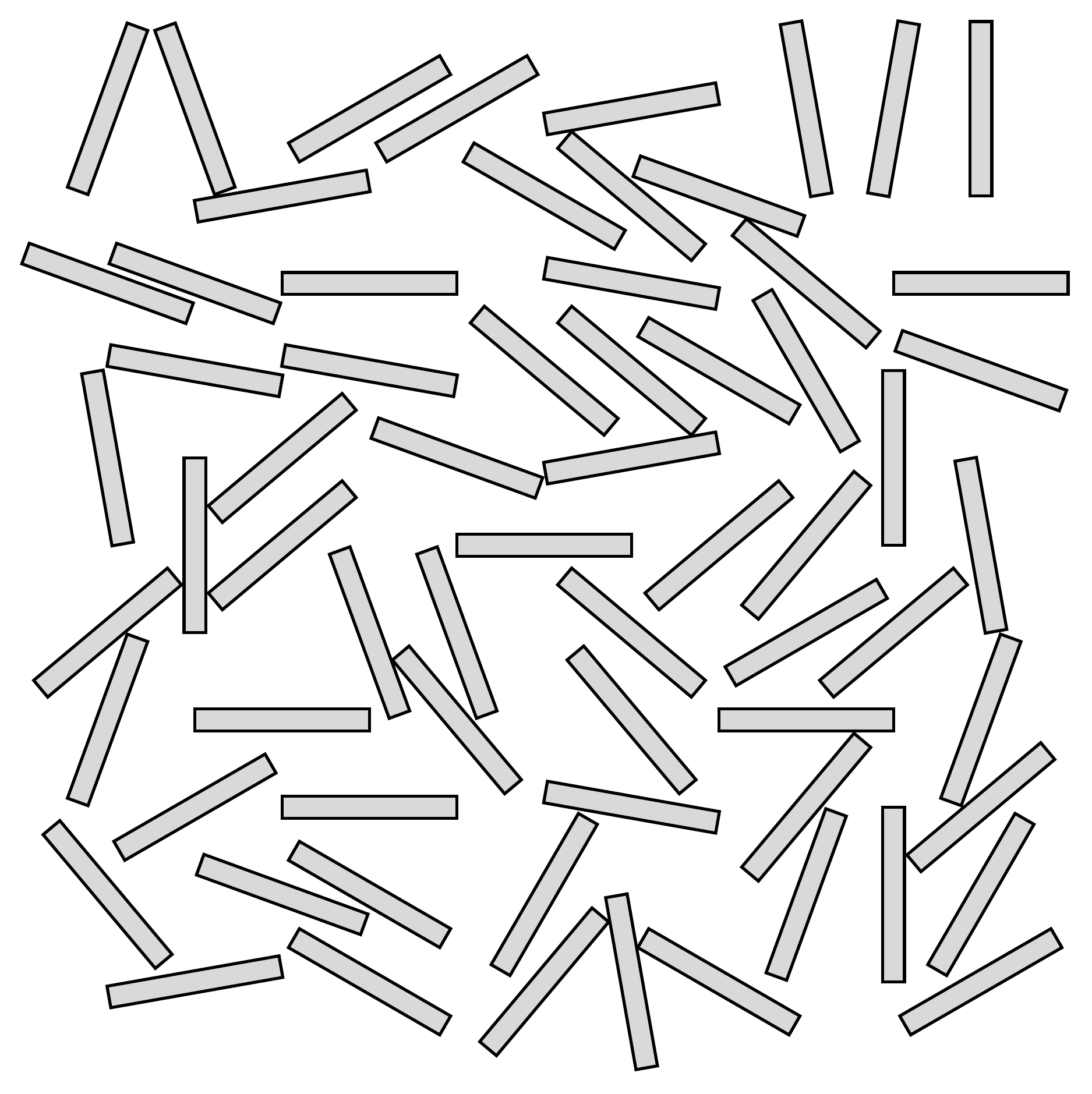}
} \hfill
\subcaptionbox{
Arbitrary sample of background rectangles from (f).
}[0.24\textwidth]{
\includegraphics[width=0.24\textwidth]{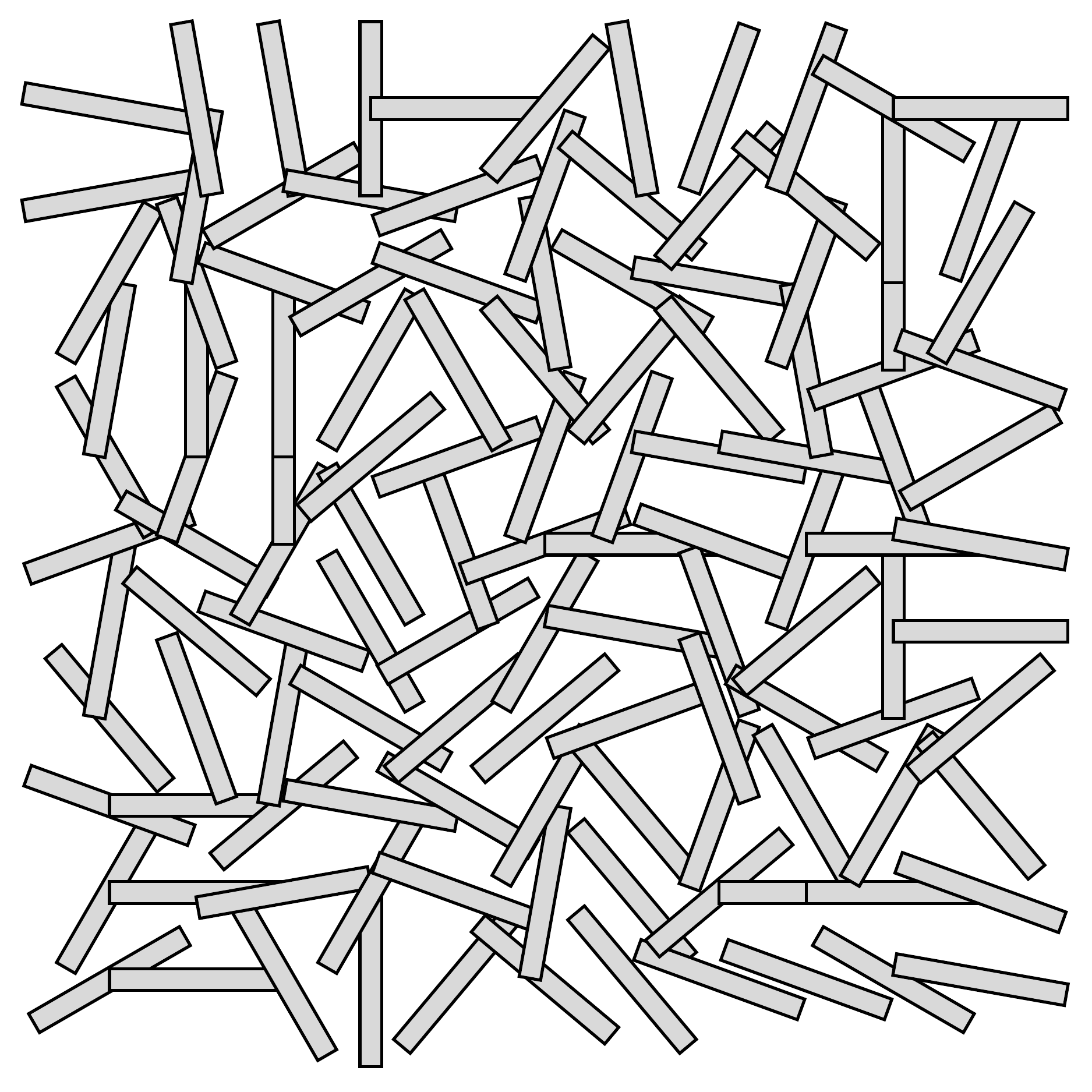}
}
%%%%%%%%%%%%%%%%%%
\\%[0.05\textwidth]
%%%%%%%%%%%%%%%%%%
\subcaptionbox{ \label{fig:rectangles-input}
Input data: point pattern resulting from uniformly sampling the rectangles (a).
}[0.24\textwidth]{
\includegraphics[width=0.24\textwidth]{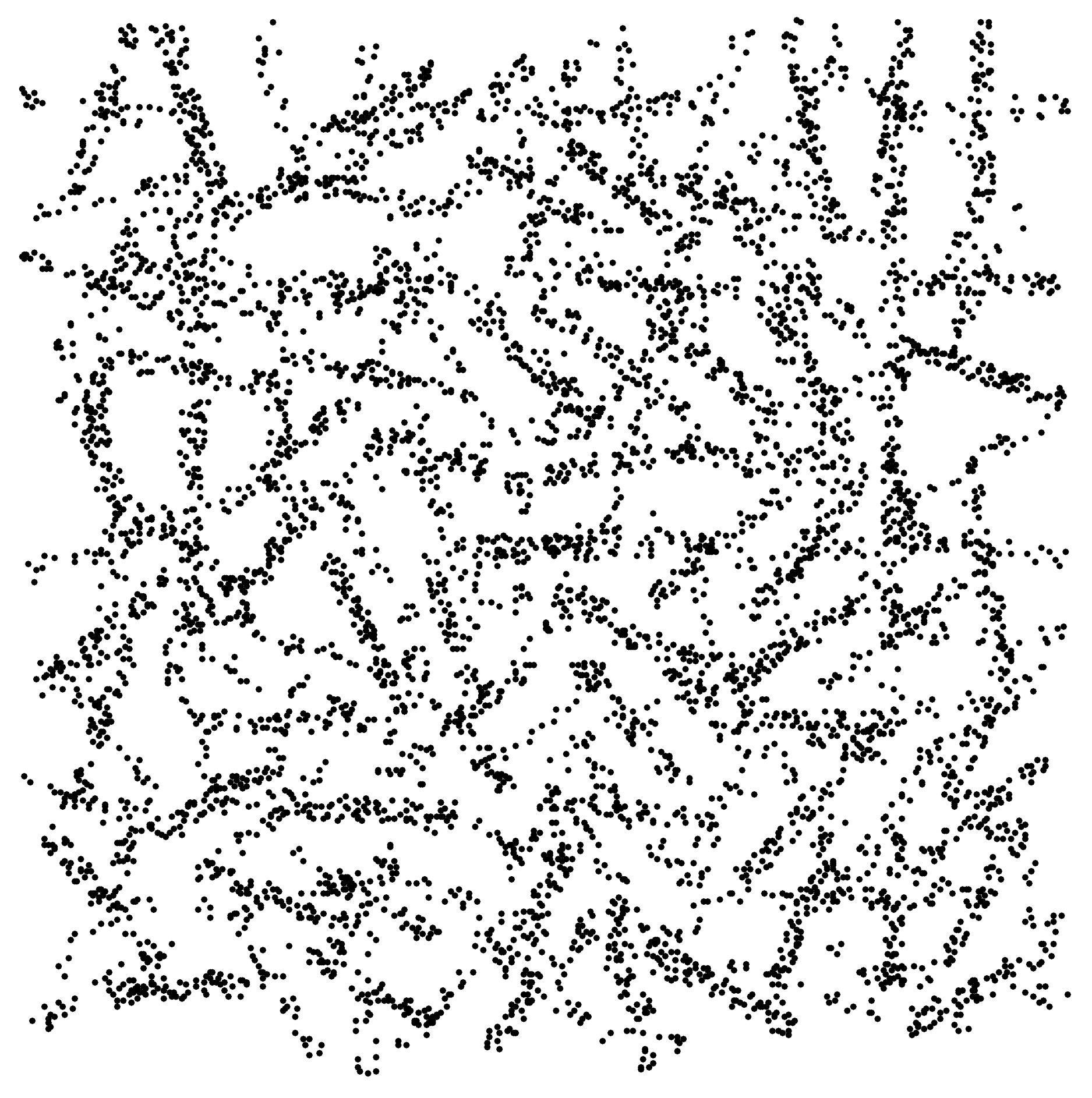}
} \hfill
\subcaptionbox{ \label{fig:rectangles-all}
All possible rectangles densely cover the domain as indicated in the center region (not completely shown for better visibility).
}[0.24\textwidth]{
\includegraphics[width=0.24\textwidth]{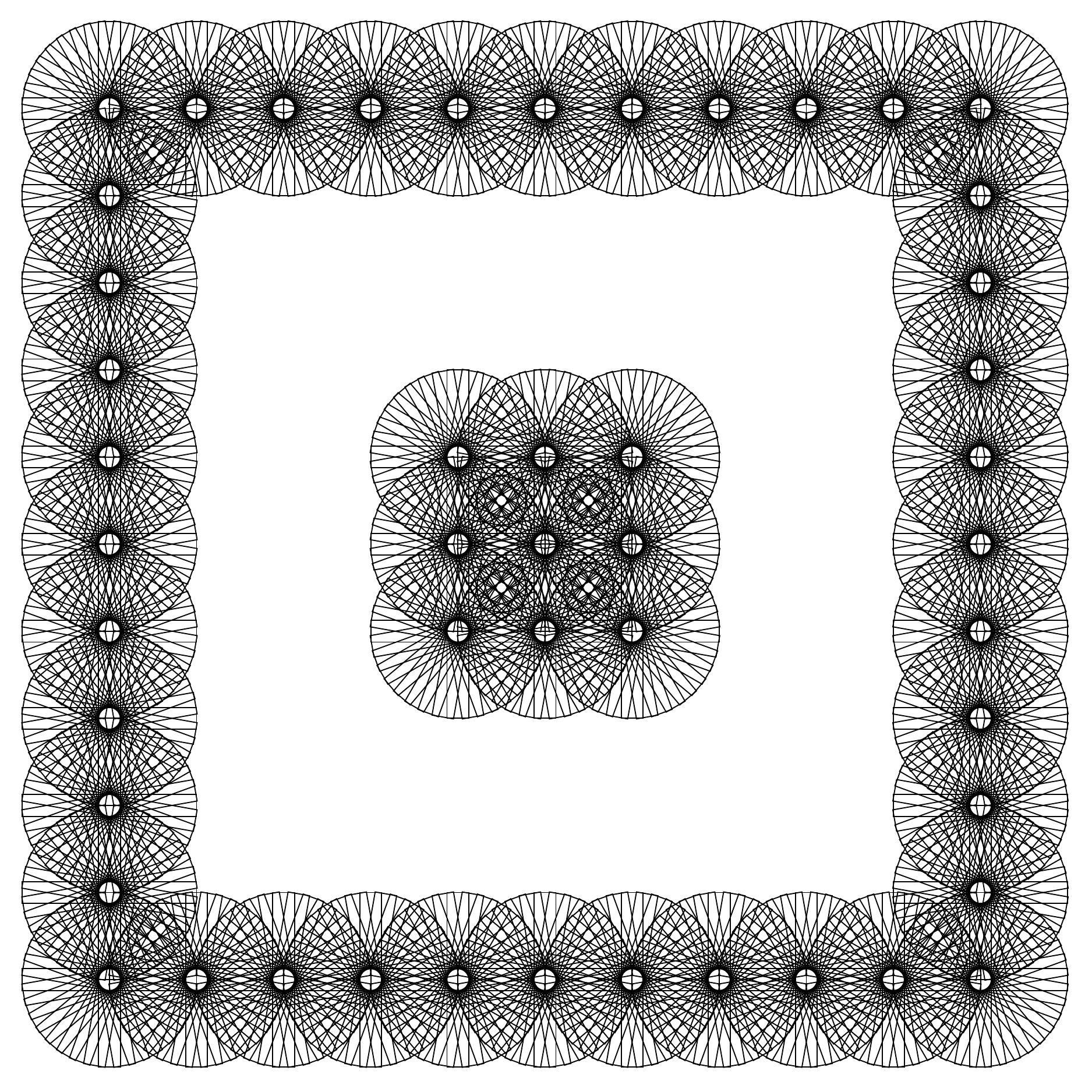}
} \hfill
\subcaptionbox{ \label{fig:rectangles-result}
Assignment (labeling) of the rectangles (f) based on the data (e): recognized foreground objects from (c) (black) and recognized background objects from (d) (dashed). Two foreground objects were erroneously labeled as background (gray). All remaining rectangles from (f) also belong to the background, four of which were erroneously labelled as foreground (white).
}[0.48\textwidth]{
\includegraphics[width=0.24\textwidth]{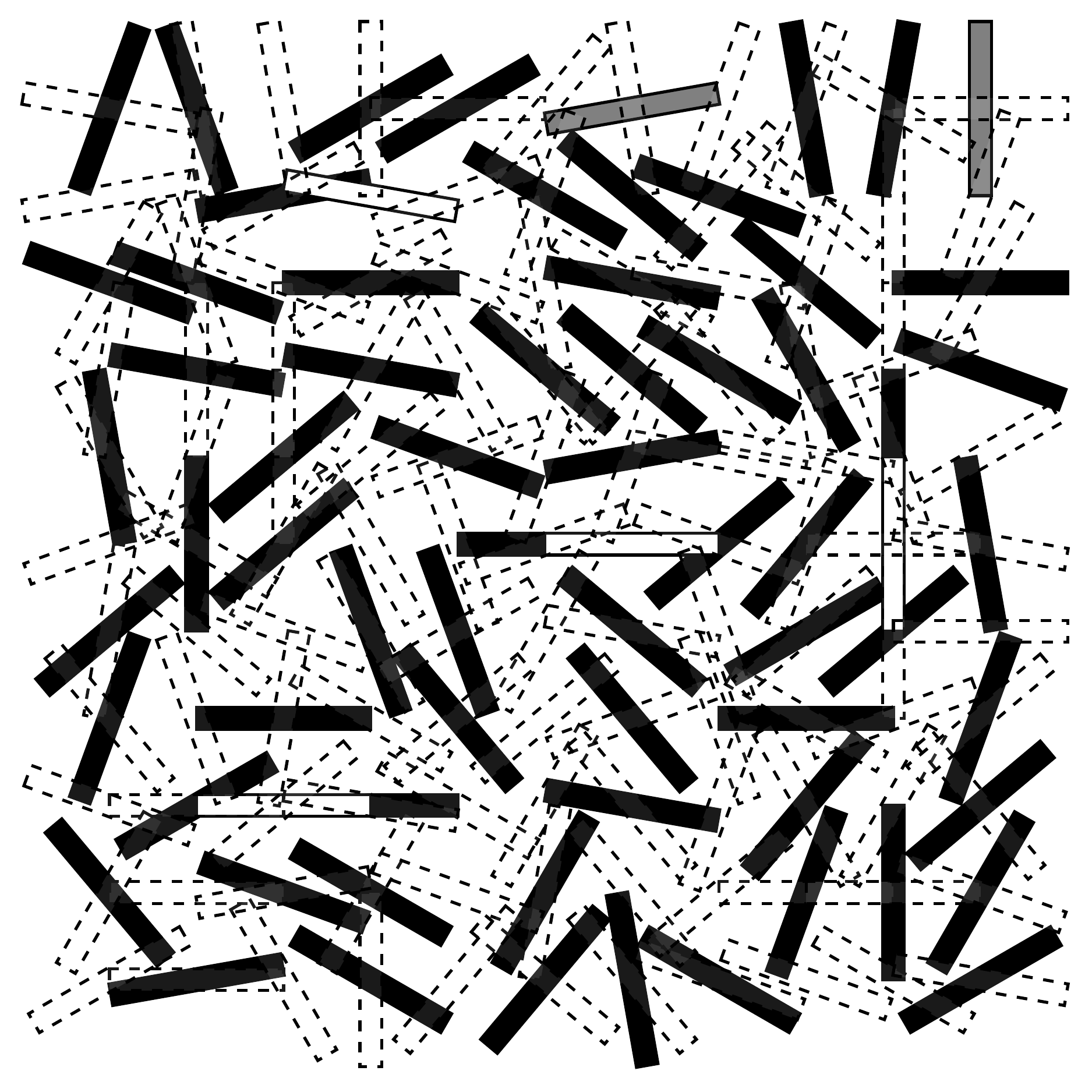}
} 
%%%
\caption{
Scenario for evaluating the approach of Section \ref{sec:rectangles}. (f) illustrates the set of all rectangles and corresponding subsets (c), (d). Unlike (d), the rectangles (c) do not intersect. Sampling the rectangles from both (c) and (d), shown together by (a) and (b), produced the input data (e). The task is to recognize among (f) all foreground objects (c) based on unary features (coverage of points) and disjunctive constraints (rectangles should not intersect). Panel (g) shows and discusses the result.
}
\label{fig:rectangles}
\end{figure}
\subsection{Labeling with Adaptive Distances}
\label{sec:rectangles}

In this section, we consider a simple instance of the more general class of scenarios where the distance matrix \eqref{eq:uDist} $D = D(W)$ depends on the assignment matrix $W$, in addition to the likelihood matrix $L(W)$ and the similarity matrix $S(W)$.

Figure \ref{fig:rectangles-input} displays a point pattern that was generated by sampling a foreground and background process of randomly oriented rectangles, as explained by the remaining panels of Figure \ref{fig:rectangles}. The task is to recover the foreground process among all possible rectangles (Fig.~\ref{fig:rectangles-all}) based on (i) unary features given by the fraction of points covered by each rectangle, and on (ii) the prior knowledge that unlike background rectangles, elements of the foreground process do \textit{not} intersect. Rectangles of the background process were slightly less densely sampled than foreground rectangles so as to make the unary features indicative. Due to the overlap of many rectangles (Fig.~\ref{fig:rectangles-fgbg}), however, these unary features are noisy (``weak''). 

As a consequence, exploiting the prior knowledge that foreground rectangles do not intersect becomes decisive. This is done by determining the intersection pattern of all rectangles (Fig.~\ref{fig:rectangles-all}) in terms of boolean values that are arranged into matrices $R_{ij}$, for each edge $ij$ of the grid graph whose vertices correspond to the centroids of the rectangles of Fig.~\ref{fig:rectangles-all}: $(R_{ij})_{k,l}=1$ if rectangle $k$ at position $i$ intersects with rectangle $l$ at position $j$, and $(R_{ij})_{k,l}=0$ otherwise. Due to the geometry of the rectangles, a rectangle at position $i$ may only intersect with $8 \times 18 = 144$ rectangles located within a 8-neighborhood $j \in \mc{N}_{\veps}(i)$. Generalizations to other geometries are straighforward. 

The inference task to recover the foreground rectangles (Fig.~\ref{fig:rectangles-fg}) from the point pattern (Fig.~\ref{fig:rectangles-input}) may be seen as a multi-labeling problem based on an asymmetric Potts-like model: labels correspond to equally oriented rectangles and have to be determined so as to maximize the coverage of points, subject to the pairwise constraints that selected rectangles do not intersect. Alternatively, we may think of binary ``off-on'' variables that are assigned to each rectangle of Fig.~\ref{fig:rectangles-all}, which have to be determined subject to \textit{disjunctive} constraints: at each location, at most a single variable may become active, and pairwise active variables have to satisfy the intersection constraints. Note that in order to suppress intersecting rectangles, penalizing costs are only encountered if (a subset of) pairs of variables receive the \textit{same} value 1 (= active and intersecting). This violates the submodularity constraint \cite[Eq.~(7)]{EnergyGraphCuts-PAMI04} and hence rules out global optimization using graph cuts.

Taking all ingredients into account, we define the distance vector field
\begin{equation} \label{eq:distance-rectangles}
D_{i} = D_{i}(W) = \frac{1}{\rho} \bpm \tilde D_{i}(W) \\ \sigma \epm,\quad
\tilde D_{i}(W) = -p^{i} + \frac{\lambda}{|\mc{N}_{\veps}(i)|} \sum_{j \in \mc{N}_{\veps}(i)} R_{ij} W_{j}, \quad
\lambda,\sigma > 0,
\end{equation}
where $\rho > 0$ is the selectivity parameter from \eqref{eq:uDist}, $\sigma > 0$ represents the cost of the additional label: ``none rectangle'', vector $p^{i}$ collects the fractions of points covered by the rectangles at position $i$, and $\lambda > 0$ weights the influence of the intersection prior. This latter term is defined by the matrices $R_{ij}$ discussed above and given by the gradient with respect to $W$ of the penalty $(\lambda/|\mc{N}_{\veps}(i)|) \sum_{ij \in \mc{E}} \la W_{i}, R_{ij} W_{j} \ra$.

In \cite{Kappes-et-al-08}, a continuous optimization approach using DC (difference of convex functions) programming was proposed to compute local minimizers of non-convex functionals similar to $\la D(W), W \ra$, with $D$ given by \eqref{eq:distance-rectangles}. This ``Euclidean approach'' -- in contrast to the geometric approach proposed here -- entails to provide a DC-decomposition of the intersection penalty just discussed and to \textit{explicitly} take into account the affine constraints $W_{i} \in \Delta_{n-1}$. As a result, the DC-approach computes a local minimizer by solving a \textit{sequence} of convex quadratic programs. 

In order to apply our present approach instead, we bypass the averaging step \eqref{eq:def-S} because labels will most likely be different at adjacent vertices $i$ in our random scenario, and we thus set $S(W) = L(W)$ with $L(W)$ given by \eqref{eq:def-L} based on \eqref{eq:distance-rectangles}. Applying then algorithm \eqref{eq:def-algorithm} \textit{implicitly} handles all constraints through the geometric flow and computes a local minimizer by multiplicative updates, within a small fraction of the runtime that the DC approach would need, and without compromising the quality of the solution (Fig.~\ref{fig:rectangles-result}).

\begin{figure}
\centering
\begin{subfigure}[b]{0.35\textwidth}
\centering
\includegraphics[width=0.45\textwidth]{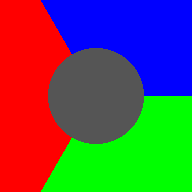}
\hspace{0.02\textwidth}%%%
\includegraphics[width=0.45\textwidth]{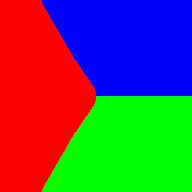}
\caption{
Inpainting of the regions marked by grey through assignment leads to the result on the right.
}
\label{fig:triplePoint}%
\end{subfigure}
\hspace{0.05\textwidth}%%%%%%%%%%
\begin{subfigure}[b]{0.35\textwidth}
\centering
\includegraphics[width=0.45\textwidth]{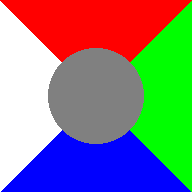}
\hspace{0.02\textwidth}%%%
\includegraphics[width=0.45\textwidth]{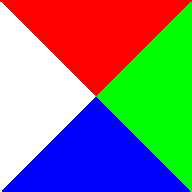}
\caption{
Inpainting of the regions marked by grey through assignment leads to the result on the right.
}
\label{fig:quadPoint}%
\end{subfigure}
%\hspace{0.02\textwidth}
%
\caption{
Two instances shown on the left in (a) and (b), adopted from \cite{Lellmann-Schnoerr-10a}, \cite{Chambolle2012} to study the tightness of convex \textit{outer} relaxations of the image labeling problem. The task is both to inpaint and to label the grey regions. Our smooth non-convex approach constitutes an \textit{inner} approximation that yields the labeling results shown on the right in (a) and (b), without the need of a separate rounding post-processing step that projects the solution of convex relaxations onto the feasible set of label assignments (parameters: $\rho=1$, $|\mc{N}_{\mc{E}}(i)|=3\times 3$).
}
\label{fig:rgb-inpainting}
\end{figure}

\subsection{Image Inpainting}
\textit{Inpainting} denotes the task to fill in a known region where no image data were observed or are known to be corrupted, based on the surrounding region and prior information. 

Once the feature metric $d_{\mc{F}}$ is fixed, we assign to each pixel in the region to be inpainted \textit{as datum} the \textit{uninformativ feature vector} $f$ which has the \textit{same} distance $d_{\mc{F}}(f,f^{\ast}_{j})$ to \textit{every} prior feature vector $f^{\ast}_{j} \in \mc{P}_{\mc{F}}$. Note that there is not need to explicitly compute this data vector $f$. It merely represents the rule for evaluating the distance $d_{\mc{F}}$ if one of its arguments belongs to a region to be inpainted.

Figure \ref{fig:rgb-inpainting} shows two basic examples that were used by the authors of \cite{Lellmann-Schnoerr-10a} and \cite{Chambolle2012}, respectively, to examine numerically the tightness of convex relaxations of the image labeling problem. Unlike convex relaxations that constitute \textit{outer} approximations of the combinatorically complex feasible set of assignments, our smooth non-convex approach may be considered as an \textit{inner} approximation that yields results without the need of further rounding, i.e.~the need of a post-processing step for projecting the solution of a convex relaxed problem onto the feasible set.

%\clearpage
%%%
\section{Conclusion and Further Work}
\label{sec:Conclusion}
We presented a novel approach to image labeling, formulated in a smooth geometric setting. The approach contrasts with etablished convex and non-convex relaxations of the image labeling problem through smoothness and geometric averaging. The numerics boil down to parallel sparse updates, that maximize the objective along an interior path in the feasible set of assignments and finally return a labeling. Although an elementary first-order approximation of the gradient flow was only used, the convergence rate seems competitive. In particular, a large number of labels, like in Section \ref{sec:Unsupervised-Assignment}, does not slow down convergence as is the case of convex relaxations. All aspects specific to an application domain are represented by a single distance matrix $D$ and a single user parameter $\rho$. This flexibility and the absence ad-hoc tuning parameters should promote applications of the approach to various image labeling problems.

\vspace{0.25cm}
\noindent
Aspects and open points to be addressed in future work include the following.
\begin{description}
\item[Numerics]
Many alternatives exist to the simple algorithm detailed in Section \ref{sec:optimization-algorithm}. An alternative first-order example are exponential multiplicative updates \cite{Cabrales1992}, that result from an explicit Euler discretization of the flow \eqref{eq:W-gradient-flow} rewritten in the form
\begin{equation}
\dv{t}\log\big(W_{i}(t)\big) = \nabla_{i} J(W) - \la W_{i}, \nabla_{i} J(W) \ra \eins,\qquad i \in [m].
\end{equation}
Of course, higher-order schemes respecting the geometry are conceivable as well. We point out that the inherent \textit{smoothness} of our problem formulation paves the way for \textit{systematic} progress.
\item[Nonuniform geometric averaging] So far, we did not exploit the degrees of freedom offered by the weights $w_{i},\, i \in [N]$, that define the Riemannian means by the objective \eqref{eq:objective-Rmean}. Possible enhancements of the solution-driven adaptivity of the assignment process in this connection need further investigation.
\item[Connection to nonlinear diffusion] Referring to the discussion of neighborhood filters and nonlinear diffusion in Section \ref{sec:Further-Related-Work}, research making these connections explicit is attractive  because, apparantly, our approach  is not covered by existing work.
\item[Unsupervised scenarios] The nonexistence of a prior data set $\mc{P}_{\mc{F}}$ in applications was only briefly addressed in Section \ref{sec:Unsupervised-Assignment}. In particular, the emergence of labels along with assignments and a corresponding generalization of our approach, deserves attention.
\item[Learning and updating prior information] This fundamental problem ties in with the preceding point: How can we learn and evolve prior information from many assignments over time?
\end{description}
We hope for a better mathematical understanding of corresponding models and that our work will stimulate corresponding research.

%%%%%%%%%%%%%%%%%%%%
\clearpage

\appendix
\section{Basic Notation}
\label{sec:Basic-Notation}

For $n \in \N$, we set $[n] = \{1,2,\dotsc,n\}$. 
$\eins = (1,1,\dotsc,1)^{\T}$ denotes the vector with all components equal to $1$, whose dimension can either be inferred from the context or is indicated by a subscript, e.g.~$\eins_{n}$. Vectors $v^{1}, v^{2},\dotsc$ are indexed by lower-case letters and superscripts, whereas subscripts $v_{i},\, i \in [n]$, index vector components. $e^{1},\dotsc,e^{n}$ denotes the canonical orthonormal basis of $\R^{n}$. 

We assume data to be indexed by a graph $\mc{G}=(\mc{V},\mc{E})$ with nodes $i \in \mc{V}=[m]$ and associated locations $x^{i} \in \R^{d}$, and with edges $\mc{E}$. A regular grid graph and $d=2$ is the canonical example. But $\mc{G}$ may also be irregular due to some preprocessing like forming super-pixels, for instance, or correspond to 3D images or videos ($d=3$). For simplicity, we call $i$ \textit{location} although this actually is $x^{i}$.

If $A \in \R^{m \times n}$, then the row and column vectors are denoted by $A_{i} \in \R^{n},\, i \in [m]$ and $A^{j} \in \R^{m},\, j \in [n]$, respectively, and the entries by $A_{ij}$. This notation of row vectors $A_{i}$ is the only exception from our rule of indexing vectors stated above.

The \textit{component-wise application of functions $f \colon \R \to \R$ to a vector} is simply denoted by $f(v)$, e.g.~
\begin{equation} \label{eq:def-function-componentwise}
\forall v \in \R^{n},\qquad
\sqrt{v} := (\sqrt{v_{1}},\dotsc,\sqrt{v_{n}})^{\T},\qquad
\exp(v) := \big(e^{v_{1}},\dotsc,e^{v_{n}}\big)^{\T} \qquad \text{etc.}
\end{equation}
Likewise, binary relations between vectors apply component-wise, e.g.~$u \geq v \;\gdw\; u_{i} \geq v_{i},\; i \in [n]$, and binary component-wise operations are simply written in terms of the vectors. For example,
\begin{equation}
p q := (\dotsc, p_{i} q_{i},\dotsc)^{\T},\qquad
\frac{p}{q} := \Big(\dotsc,\frac{p_{i}}{q_{i}},\dotsc\Big)^{\T},
\end{equation}
where the latter operation is only applied to strictly positive vectors $q > 0$. The \textit{support} $\supp(p) = \{p_{i} \neq 0 \colon i \in \supp(p)\} \subset [n]$ of a vector $p \in \R^{n}$ is the index set of all non-nonvanishing components of $p$.

$\la x, y \ra$ denotes the standard Euclidean inner product and $\|x\| = \la x, x \ra^{1/2}$ the corresponding norm. Other $\ell_{p}$-norms, $1 \leq p \neq 2 \leq \infty$, are indicated by a corresponding subscript,
%\begin{equation}
$
\|x\|_{p} = \big(\sum_{i \in [d]} |x_{i}|^{p}\big)^{1/p},
$
%\end{equation}
except for the case $\|x\| = \|x\|_{2}$. For matrices $A, B \in \R^{m \times n}$, the canonical inner product is 
%\begin{equation}
$
\la A, B \ra = \tr(A^{\T} B)
$
%\end{equation}
with the corresponding Frobenius norm $\|A\| = \la A, A \ra^{1/2}$. $\Diag(v) \in \R^{n \times n},\, v \in \R^{n}$, is the diagonal matrix with the vector $v$ on its diagonal.

\vspace{0.25cm} 
Other basic sets and their notation are
\begin{subequations}
\begin{align}
&\text{the positive orthant} &
\R_{+}^{n} &= \{ p \in \R^{n} \colon p \geq 0 \}, \\
&\text{the set of strictly positive vectors}, &
\R_{++}^{n} &= \{p \in \R^{n} \colon p > 0\}, \\
&\text{the ball of radius $r$ centered at $p$}, &
\B_{r}(p) &= \{p \in \R^{n} \colon \|p\| \leq r\}, \\
&\text{the unit sphere} &
\Sp^{n-1} &= \{p \in \R^{n} \colon \|p\|=1\}, \\
&\text{the probability simplex} &
\Delta_{n-1} &= \{p \in \R_{+}^{n} \colon \la \eins, p \ra = 1 \}, \\
\label{eq:rint-delta}
&\text{and its relative interior} &
\mc{S} &= \otop\Delta_{n-1} = \Delta_{n-1} \cap \R_{++}^{n}, \\
& &
\mc{S}_{n} &= \mc{S}\;\text{with concrete value of $n$ 
(e.g.~$\mc{S}_{3}$)}, \\
& \text{closure (not regarded as manifold)} &
\ol{\mc{S}} &= \Delta_{n-1}, \\
&\text{the sphere with radius $2$} &
\mc{N} &= 2 \Sp^{n-1}, \\
&\text{and the assignment manifold} &
\mc{W} &= \mc{S} \times \dotsb \times \mc{S},
\quad\text{($m$ times)}, \\
&\text{closure (not regarded as manifold)} &
\ol{\mc{W}} &= \ol{\mc{S}} \times \dotsb \times \ol{\mc{S}},
\quad\text{($m$ times)}.
\end{align}
\end{subequations}
For a discrete distribution $p \in \Delta_{n-1}$ and a finite set $S=\{s^{1},\dotsc,s^{n}\}$ vectors, we denote by
\begin{equation}
\EE_{p}[S] := \sum_{i \in [n]} p_{i} s^{i}
\end{equation}
the mean of $S$ with respect to $p$.

\vspace{0.25cm}
Let $\mc{M}$ be a any differentiable manifold. Then $T_{p}\mc{M}$
denotes the tangent space at base point $p \in \mc{M}$ and $T\mc{M}$ the total space of the tangent bundle of $\mc{M}$. If $F \colon \mc{M} \to \mc{N}$ is a smooth mapping between differentiable manifold $\mc{M}$ and $\mc{N}$, then the differential of $F$ at $p \in \mc{M}$ is denoted by
\begin{equation} \label{eq:differential-notation}
DF(p) \colon T_{p}\mc{M} \to T_{F(p)}\mc{N},\qquad
DF(p) \colon v \mapsto DF(p)[v].
\end{equation}
If $F \colon \R^{m} \to \R^{n}$, then $DF(p) \in \R^{n \times m}$ is the Jacobian matrix at $p$, and the application $DF(p)[v]$ to a vector $v \in \R^{m}$ means matrix-vector multiplication. We then also write $DF(p)v$. If $F = F(p,q)$, then $D_{p}F(p,q)$ and $D_{q}F(p,q)$ are the Jacobians of the functions $F(\cdot,q)$ and $F(p,\cdot)$, respectively.

The gradient of a differentiable function $f \colon \R^{n} \to \R$ is denoted by $\nabla f(x) = \big(\partial_{1} f(x),\dotsc,\partial_{n} f(x)\big)^{\T}$, whereas the Riemannian gradient of a function $f \colon \mc{M} \to \R$ defined on Riemannian manifold $\mc{M}$ is denoted by $\nabla_{\mc{M}} f$. Eq.~\eqref{eq:def-nabla_S} recalls the formal definition.

The \textit{exponential mapping} \cite[Def.~1.4.3]{Jost2005}
\begin{equation} \label{eq:def-Exp}
\Exp_{p} \colon T_{p}\mc{M} \to \mc{M},\quad
v \mapsto \Exp_{p}(v) = \gamma_{v}(1),\qquad \gamma_{v}(0)=p,\; \dot \gamma_{v}(0)=\dv{t}\gamma_{v}(t)\big|_{t=0} = v
\end{equation}
maps the tangent vector $v$ to the point $\gamma_{v}(1) \in \mc{M}$, uniquely defined by the \textit{geodesic curve} $\gamma_{v}(t)$ emanating at $p$ in direction $v$. $\gamma_{v}(t)$ is the shortest path on $\mc{M}$ between the points $p, q \in \mc{M}$ that $\gamma_{v}$ connects. This miminal length equals the \textit{Riemannian distance} $d_{\mc{M}}(p,q)$ induced by the 
\textit{Riemannian metric}, denoted by 
\begin{equation} \label{eq:def-Rmetric}
\la u, v \ra_{p},
\end{equation}
i.e.~the inner product on the tangent spaces $T_{p}\mc{M},\,p \in \mc{M}$, that smoothly varies with $p$. Existence and uniqueness of geodesics will not be an issue for the manifolds $\mc{M}$ considered in this paper.
\begin{remark}
The exponential mapping $\Exp_{p}$ should not be confused with 
\begin{itemize}
\item
the exponential function $e^{v}$ used e.g.~in \eqref{eq:def-function-componentwise};
\item
the mapping $\exp_{p} \colon T_{p}\mc{S}$ defined by Eq.~\eqref{eq:def-exp-vector}.
\end{itemize}
\end{remark}
\noindent
The abbreviations `l.h.s.'~and `r.h.s.'~mean \textit{left-hand side} and \textit{right-hand side} of some equation, respectively. We abbreviate \textit{with respect to} by `wrt.'.

%%%%%%%%%
\newpage

\section{Proofs and Further Details} \label{sec:App-Proofs}
\subsection{Proofs of Section \ref{sec:Assignment-Manifold}}
\label{sec:S-proofs}

\begin{proof}[Proof of Lemma \ref{lem:sphere-map}]
Let $p \in \mc{S}$ and $v \in T_{p}\mc{S}$. We have 
\begin{equation} \label{eq:Dpsi-simplex}
D\psi(p) = \Diag(p)^{-1/2}
\end{equation}
and $\big\la \psi(p), D\psi(p)[v] \big\ra = \la 2 \sqrt{p}, \frac{v}{\sqrt{p}} \ra = 2 \la \eins, v \ra = 0$,
that is $D\psi(p)[v] \in T_{\psi(p)}\mc{N}$. Furthermore,
\begin{equation}
\big\la D\psi(p)[u], D\psi(p)[v] \big\ra
= \big\la u/\sqrt{p}, v/\sqrt{p} \ra \overset{\eqref{eq:metric-simplex}}{=} \la u, v \ra_{p},
\end{equation}
i.e.~the Riemannian metric is preserved and hence also the length $L(s)$ of curves $s(t) \in \mc{N},\, t \in [a,b]$: Put $\gamma(t) = \psi^{-1}\big(s(t)\big) = \frac{1}{4} s^{2}(t) \in \mc{S},\, t \in [a,b]$. Then $\dot\gamma(t)=\frac{1}{2} s(t) \dot s(t) = \frac{1}{2} \psi\big(\gamma(t)\big) \dot s(t) = \sqrt{\gamma(t)} \dot s(t)$ and
\begin{equation}
L(s) = \int_{a}^{b} \|\dot s(t)\| dt
= \int_{a}^{b} \bigg\la \frac{\dot\gamma(t)}{\sqrt{\gamma(t)}}, \frac{\dot\gamma(t)}{\sqrt{\gamma(t)}} \bigg\ra^{1/2} dt
\overset{\eqref{eq:metric-simplex}}{=} \int_{a}^{b} \|\dot\gamma(t)\|_{\gamma(t)} dt = L(\gamma).
\end{equation}
\end{proof}
\begin{proof}[Proof of Prop.~\ref{prop:simplex-gradient}]
Setting $g\colon\mc{N} \to \R$, $q \mapsto g(s) := f\big(\psi^{-1}(s)\big)$ with $s = \psi(p) = 2 \sqrt{p}$ from \eqref{eq:def-psi-simplex-sphere}, we have 
\begin{equation} \label{eq:g-N-gradient}
\nabla_{\mc{N}} g(s) = \bigg(I - \frac{s}{\|s\|} \frac{s^{\T}}{\|s\|}\bigg) \nabla g(s),
\end{equation}
because the 2-sphere $\mc{N}=2\Sp^{n-1}$ is an embedded submanifold, and hence the Riemannian gradient equals the orthogonal projection of the Euclidean gradient onto the tangent space. Pulling back the vector field $\nabla_{\mc{N}} g$ by $\psi$ using 
\begin{equation}
\nabla g(s) = \nabla f\big(\psi^{-1}(s)\big)
= \nabla f\Big(\frac{1}{4} s^{2}\Big)
= \frac{1}{2} s \big(\nabla f(p)\big),
\end{equation}
we get with \eqref{eq:Dpsi-simplex}, \eqref{eq:g-N-gradient} and $\|s\|=2$ and hence $s/\|s\| = \frac{1}{2} \psi(p)=\sqrt{p}$
\begin{subequations}
\begin{align}
\nabla f_{\mc{S}}(p) &= \big(D\psi(p)\big)^{-1}\big(\nabla_{\mc{N}} g(\psi(p))\big) \\
&= \Diag(\sqrt{p}) \Big(\big(I - \sqrt{p} \sqrt{p}^{\T}\big)
\sqrt{p} \big(\nabla f(p)\big)\Big) \\
&= p \big(\nabla f(p)\big) - \la p, \nabla f(p) \ra p,
\end{align}
\end{subequations}
which equals \eqref{eq:simplex-gradient}. We finally check that $\nabla f_{\mc{S}}(p)$ satisfies \eqref{eq:def-nabla_S} (with $\mc{S}$ in place of $\mc{M}$). Using \eqref{eq:metric-simplex}, we have
\begin{subequations}
\begin{align}
\la \nabla f_{\mc{S}}(p), v \ra_{p}
&= \Big\la \sqrt{p} \big(\nabla f(p)\big) - \la p, \nabla f(p) \ra \sqrt{p}, \frac{v}{\sqrt{p}} \Big\ra \\
&= \la \nabla f(p), v \ra - \la p, \nabla f(p) \ra \la \eins, v \ra
\overset{\eqref{eq:def-TSimplex}}{=} \la \nabla f(p), v \ra,\qquad \forall v \in T_{p}\mc{S}.
\end{align}
\end{subequations}
\end{proof}
\begin{proof}[Proof of Prop.~\ref{prop:Exp}]
The geodesic on the 2-sphere emanating at $s(0) \in \mc{N}$ in direction $w=\dot s(0) \in T_{s(0)}\mc{N}$ is given by
\begin{equation}
s(t) = s(0) \cos\Big(\frac{\|w\|}{2} t\Big) 
+ 2 \frac{w}{\|w\|} \sin\Big(\frac{\|w\|}{2} t\Big).
\end{equation}
Setting $s(0)=\psi(p)$ and $w = D\psi(p)[v]=v/\sqrt{p}$, the geodesic emanating at $p=\gamma_{v}(0)$ in direction $v$ is given by $\psi^{-1}\big(s(t)\big)$ due to Lemma \ref{lem:sphere-map}, which results in \eqref{eq:gamma-S} after elementary computations.
\end{proof}

\subsection{Proofs of Section \ref{sec:Approach} and Further Details}
\label{sec:approach-proofs}

\begin{proof}[Proof of Prop.~\ref{prop:Exp-by-exp}]
We have $p = \exp_{p}(0)$ and 
\begin{equation}
\dv{t} \exp_{p}(u t) = \frac{\la p, e^{u t} \ra p e^{u t} u - p e^{u t} \la p, e^{u t} u \ra}{\la p, e^{u t} \ra^{2}} 
= p(t) u - \la p(t), u \ra p(t), 
\end{equation}
which confirms \eqref{eq:exp-pt}, is equal to \eqref{eq:v-from-u} at $t=0$ and hence yields the first expression of \eqref{eq:exp-approximates-gamma}. The second expression of \eqref{eq:exp-approximates-gamma} follows from a Taylor expansion of \eqref{eq:gamma-S}
\begin{equation}
\gamma_{v}(t) \approx p + v t + \frac{1}{4}\big(v_{p}^{2}-\|v_{p}\|^{2} p\big) t^{2},\qquad v_{p} = \frac{v}{\sqrt{p}}.
\end{equation}
\end{proof}

\begin{proof}[Proof of Lemma \ref{lem:global-maxima}]
By construction, $S(W) \in \mc{W}$, that is $S_{i}(W) \in \mc{S},\; i \in [m]$. Consequently, $0 \leq J(W)=\sum_{i \in [m]} \la S_{i}(W), W_{i} \ra \leq \sum_{i \in [m]} \|S_{i}(W)\| \|W_{i}\| < m$. The upper bound corresponds to matrices $\ol{W}^{\ast} \in \ol{\mc{W}}$ and $S(\ol{W}^{\ast})$ where for each $i \in [m]$, \textit{both} $\ol{W}^{\ast}_{i}$ and $S_{i}(\ol{W}^{\ast})$ equal the \textit{same} unit vector $e^{k_{i}}$ for some $k_{i} \in [m]$.
\end{proof}

\begin{proof}[Explicit form of \eqref{eq:Tij-matrices}]
The matrices $T^{ij}(W) = \pdv{W_{ij}} S(W)$ are implicitly given through the optimality condition \eqref{eq:grad-Rmean} that each vector $S_{k}(W),\, k \in [m]$, defined by \eqref{eq:def-S} has to satisfy,
\begin{equation} \label{eq:SkW-optimality-condition}
S_{k}(W) = \mrm{mean}_{\mc{S}}\{L_{r}(W_{r})\}_{r \in \tilde{\mc{N}}_{\mc{E}}(k)}
\qquad\gdw\qquad
0 = \sum_{r \in \tilde{\mc{N}}_{\mc{E}}(k)} \Exp_{S_{k}(W)}^{-1}\big(L_{r}(W_{r})\big).
\end{equation}
Writing
\begin{equation} \label{eq:def-phi}
\phi\big(S_{k}(W),L_{r}(W_{r})\big) 
:= \Exp_{S_{k}(W)}^{-1}\big(L_{r}(W_{r})\big),
\end{equation}
and temporarily dropping below $W$ as argument to simplify the notation, and using the indicator function $\delta_{\mrm{P}} = 1$ if the predicate $\mrm{P}=\mrm{true}$ and $\delta_{\mrm{P}} = 1$ otherwise, we differentiate the optimality condition on the r.h.s.~of \eqref{eq:SkW-optimality-condition},
\begin{subequations}
\begin{align}
0 &= \pdv{W_{ij}} \sum_{r \in \tilde{\mc{N}}_{\mc{E}}(k)}
\phi\big(S_{k}(W),L_{r}(W_{r})\big) \\
&= \sum_{r \in \tilde{\mc{N}}_{\mc{E}}(k)} \Big( D_{S_{k}}\phi(S_{k},L_{r}) \Big[\pdv{W_{ij}} S_{k}(W)\Big] + \delta_{i=r} D_{L_{r}}\phi(S_{k},L_{r}) \Big[\pdv{W_{rj}} L_{r}(W_{r})\Big] \Big) \\
&= \Big(\sum_{r \in \tilde{\mc{N}}_{\mc{E}}(k)} D_{S_{k}}\phi(S_{k},L_{r}) \Big) \Big(\pdv{W_{ij}} S_{k}(W)\Big)
+ \delta_{i \in \tilde{\mc{N}}_{\mc{E}}(k)} D_{L_{i}}\phi(S_{k},L_{i})\Big(\pdv{W_{ij}} L_{i}(W_{i}) \Big) \\ \label{eq:def-H-h}
&=: H^{k}(W) \Big(\pdv{W_{ij}} S_{k}(W)\Big) + h^{k,ij}(W).
\end{align}
\end{subequations}
Since the vectors $\phi(S_{k},L_{r})$ given by \eqref{eq:def-phi} are the negative Riemannian gradients of the (locally) strictly convex objectives \eqref{eq:objective-Rmean} defining the means $S_{k}$ \cite[Thm.~4.6.1]{Jost2005}, the regularity of the matrices $H^{k}(W)$ follows. Thus, using \eqref{eq:def-H-h} and defining the matrices
\begin{equation}
T^{ij}(W) \in \R^{m \times n},\qquad 
T^{ij}_{kl}(W) := \pdv{S_{kl}(W)}{W_{ij}},\qquad
i,k \in [m],\quad j,l \in [n],
\end{equation}
results in \eqref{eq:Tij-matrices}. The explicit form of this expression results from computing and inserting into \eqref{eq:def-H-h} the corresponding Jacobians $D_{p}\phi(p,q)$ and $D_{q}\phi(p,q)$ of
\begin{subequations} \label{eq:InvExp-pdfLi}
\begin{align}
\phi(p,q)&=\Exp_{p}^{-1}(q)
= \frac{d_{\mc{S}}(p,q)}{\sqrt{1-\la \sqrt{p},\sqrt{q}\ra^{2}}}\big(\sqrt{p q}-\la\sqrt{p},\sqrt{q}\ra p\big), 
\label{eq:InvExp-pdfLi-1} \intertext{and} \label{eq:InvExp-pdfLi-2}
\pdv{W_{ij}} L_{i}(W_{i})
&= \frac{e^{-U_{ij}}}{\la W_{i}, e^{-U_{i}} \ra} \big(e^{j} - L_{i}(W_{i})\big).
\end{align}
\end{subequations}
The term \eqref{eq:InvExp-pdfLi-1} results from mapping back the corresponding vector from the 2-sphere $\mc{N}$,
\begin{equation}
\Exp_{p}^{-1}(q) = -\big(D\psi(p)\big)^{-1}\Big(\frac{1}{2}\nabla_{\mc{N}}d_{\mc{N}}^{2}\big(\psi(p),\psi(q)\big)\Big),
\end{equation}
where $\psi$ is the sphere map \eqref{eq:def-psi-simplex-sphere} and $d_{\mc{N}}$ is the geodesic distance on $\mc{N}$.
The term \eqref{eq:InvExp-pdfLi-2} results from directly evaluating \eqref{eq:def-L}.
\end{proof}

\begin{proof}[Proof of Lemma \ref{eq:Rmean-approximation}]
We first compute $\exp_{p}^{-1}$. Suppose
\begin{equation}
q = \exp_{p}(u) = \frac{p e^{u}}{\la p, e^{u} \ra}, \qquad
p,q \in \mc{S},\quad u \in \R^{n}.
\end{equation}
Then
\begin{equation}
\log(q) = \log(p)+u-\log(\la p, e^{u} \ra) \eins,\qquad
\log(\la p, e^{u} \ra) = \frac{1}{n} \la \eins, \log(p)-\log(q) \ra,
\end{equation}
and 
\begin{equation}
u = \exp_{p}^{-1}(q) 
= (I-\frac{1}{n} \eins \eins^{\T})\big(\log(q)-\log(p)\big).
\end{equation}
Thus, in view of \eqref{eq:v-from-u}, we approximate
\begin{subequations}
\begin{align}
\Exp_{p}^{-1}(q) \approx v &= \big(\Diag(p)-p p^{\T}\big) u 
= \big(\Diag(p)-\frac{1}{n} p \eins^{\T} - p p^{\T} + \frac{1}{n} p \eins^{\T}) \log\Big(\frac{q}{p}\Big) \\
&= \big(\Diag(p)-p p^{\T}\big) \log\Big(\frac{q}{p}\Big). 
\end{align}
\end{subequations}
Applying this to the point set $\mc{P}$, i.e.~setting
\begin{equation}
v^{i} = \big(\Diag(p)-p p^{\T}\big) \log\frac{p^{i}}{p}, \qquad 
i \in [N],
\end{equation}
step (3) of \eqref{eq:Rmean-iteration} yields
\begin{subequations}
\begin{align}
v &:= \frac{1}{N} \sum_{i \in [N]} v^{i}
= \frac{1}{N} \big(\Diag(p)-p p^{\T}\big)\Big(\sum_{i \in [N]} \log(p^{i}) - N \log(p)\Big) \\
&= \big(\Diag(p)-p p^{\T}\big) \log\bigg( \frac{1}{p}
\Big(\prod_{i \in [N]} p^{i}\Big)^{\frac{1}{N}} \bigg) \\
&= \big(\Diag(p)-p p^{\T}\big) \log\Big(\frac{\mrm{mean}_{g}(\mc{P})}{p}\Big) =: \big(\Diag(p)-p p^{\T}\big) u.
\end{align}
\end{subequations}
Finally, approximating step (4) of \eqref{eq:Rmean-iteration} results in view of Prop.~\ref{prop:Exp-by-exp} in the update of $p$
\begin{equation}
\exp_{p}(u) = \frac{p e^{u}}{\la p, e^{u} \ra}
= \frac{\mrm{mean}_{g}(\mc{P})}{\la \eins, \mrm{mean}_{g}(\mc{P}) \ra}.
\end{equation}
\end{proof}

\newpage
%%%%%%%%%%%%%%%%%%%%%%%%%%%%%%%
\bibliographystyle{amsalpha}
\bibliography{assignmentFilter}

\end{document}